\documentclass{article}

\usepackage[preprint,nonatbib]{neurips_2021}

\usepackage[utf8]{inputenc} 
\usepackage[T1]{fontenc}    
\usepackage{hyperref}       
\usepackage{url}            
\usepackage{booktabs}       
\usepackage{amsfonts}       
\usepackage{nicefrac}       
\usepackage{microtype}      
\usepackage{xcolor}         

\usepackage{amsmath, amssymb, bm, bbm}

\DeclareMathOperator*{\argmin}{arg\,min}
\DeclareMathOperator*{\argmax}{arg\,max}

\usepackage{amsthm}
\newtheorem{thm}{Theorem}[section]
\newtheorem{lemma}[thm]{Lemma}
\newtheorem{prop}[thm]{Proposition}
\newtheorem{cor}[thm]{Corollary}
\newtheorem{defn}[thm]{Definition}

\newtheorem{assumption}{Assumption}
\usepackage{lipsum}
\usepackage{mathtools}
\usepackage{algorithm,algorithmic}
\usepackage{url}
\usepackage{graphicx}
\usepackage{grffile}
\usepackage{tikz}
\usetikzlibrary{bayesnet}
\usepackage{bbm}
\usepackage{booktabs}
\usepackage{multirow}

\usepackage{stackengine}

\newcommand{\cellbest}{\bf}
\newcommand*\samethanks[1][\value{footnote}]{\footnotemark[#1]}

\makeatletter
\newcommand{\figcaption}[1]{\def\@captype{figure}\caption{#1}}
\newcommand{\tblcaption}[1]{\def\@captype{table}\caption{#1}}
\usepackage{wrapfig}

\title{When can we formulate the Out-of-Distribution Generalization problem as an invariance problem ?}

\author{%
  Masanori Koyama \thanks{Equal contribution} \\
  Preferred Networks \\
  Tokyo, Japan \\
  \texttt{masomatics@preferred.jp} \\
  \And
  Shoichiro Yamaguchi \samethanks \\
  Preferred Networks \\
  Tokyo, Japan \\
  \texttt{guguchi@preferred.jp} \\
}

\begin{document}

\maketitle

\begin{abstract}
\label{abst}
The goal of Out-of-Distribution (OOD) generalization problem is to train a predictor that generalizes on all possible environments.
Popular approaches in this field use the hypothesis that such a predictor is an \textit{invariant predictor} that captures the mechanism that remains constant across environments.
While these approaches have been experimentally successful in various case studies, there is still much room for the theoretical validation of this hypothesis.
This paper presents a set of theoretical conditions necessary for an invariant predictor to achieve the OOD optimality.
In other words, we formalize when we can formulate the OOD generalization problem as a problem of finding a good invariant predictor. 
Our theory not only applies to non-linear cases, but also generalizes the necessary condition used in \cite{rojas2018invariant}.
We also derive Inter-environment Gradient Alignment algorithm from our theory and demonstrate its competitiveness on MNIST-derived benchmarks as well as on two of the three \textit{Invariance Unit Tests} proposed by \cite{aubinlinear}.
\end{abstract}

\section{Introduction}
\label{sec:intro}
In machine learning, a predictor model is usually evaluated based on its performance on unseen data.
However, it is also common to assume that the test data is collected in i.i.d fashion from the same distribution from which the training data was sampled, when in fact such an assumption does not always hold in applications.
A naively trained supervised learner may therefore perform poorly on a dataset drawn from outside the training distribution \cite{damour2020underspecification, geirhos2020shortcut,  shen2018causally, storkey2009training}. 
This problem has been known as Out-of-Distribution generalization problem (OOD problem).

Studies of OOD problem like \cite{arjovsky2019invariant, buhlmann2018invariance, rojas2018invariant} treat the difference between the test distribution and the training distribution as an effect of a possibly unknown environmental variable $\mathcal{E}$. 
When the task is to predict the target variable $y$ associated to the input variable $x$, this framework assumes that every dataset is sampled from $p(x, y | \epsilon)$ with a different realization of the environmental variable $\epsilon$ \footnote{Throughout, we follow the notation in \cite{Durrett} and use an uppercase letter to represent a random variable, and a lowercase letter to represent its realization. That is, $x$ is a realization of $X$, $y$ is a realization of $Y$ and $\epsilon$ is a realization of $\mathcal{E}$.}.
In other words, if there are two datasets $D_1$ and $D_2$ collected from two different environments, the framework assumes that they are sampled respectively from $p(x, y | \epsilon_1)$ and $p(x, y | \epsilon_2)$ with different values of $\epsilon_1$ and $\epsilon_2$. 
If $l(f(x), y)$ represents the error between $y$ and the prediction $f(x)$,  
we may formulate the OOD problem as the problem of finding $f^*$ that satisfies
\begin{align}
f^* = \argmin_{f} \max_{\epsilon \in \textrm{supp}(\mathcal{E})}  \mathbb{E}_{X, Y}[l(f(X), Y) | \epsilon] 
\label{eq:objective_informal}
\end{align}
where $\textrm{supp}(\mathcal{E})$ is the support of the environmental variable $\mathcal{E}$ or the space of all possible test environments \cite{arjovsky2019invariant, buhlmann2018invariance, rojas2018invariant}.
We say that such $f^*$ is \textit{OOD-optimal}. 
To solve the OOD problem \eqref{eq:objective_informal} directly, we need to evaluate a model on all environments in $\textrm{supp}(\mathcal{E})$, which also includes the test environments themselves.

A strategy that has gained much traction in recent years is to seek the solution of the OOD problem \eqref{eq:objective_informal} from the set of invariant predictors that does not depend on the choice of the environment $\epsilon$. 
Experimental results on various works on invariant predictors \cite{arjovsky2019invariant, ilyas2019adversarial, li2018advdomain, rojas2018invariant} show this strategy to be promising.
This strategy is indeed based on the belief that solving \eqref{eq:objective_informal} is essentially equivalent to the problem of finding a good invariant predictor (the \textit{invariance problem}).  
But is this belief 'provably' correct? 
In particular, is there some guarantee that \textit{a good invariance predictor} can solve the OOD problem for a ubiquitous problem like image classification, in which the form of invariance is complex and the true underlying model is non-linear? 
If so, what is theoretically necessary for an invariant predictor to be able to solve the problem \eqref{eq:objective_informal}?

The idea of \textit{invariance problem} has its roots in the study of causality, which uses Directed Acyclic Graph (DAG) to formalize the invariance \cite{buhlmann2018invariance,peters2017elements}.
To prove a set of sufficient conditions for their invariant predictor to be OOD optimal, \cite{prevent19a} used ``\textit{causal DAG to encode prior information about how the distribution of data might change}".
\cite{rojas2018invariant} also proved the OOD optimality of an invariant predictor that is inspired from the ideas of causality.
Essentially, most causality inspired methods define an invariant predictor with a discrete subset of \textit{observable} variables (e.g., a node-set of Causal DAG) that is causal to the target variable \cite{Peters2016jrssb, rojas2018invariant, prevent19a}. 
However, with such a framework, it can be difficult to theoretically investigate the conditions of OOD optimality on datasets in which there is no universal subset of \textit{observable variable} that is causal.
Such a case may arise on datasets involving images, in which each observable variable is a vector of RGB values at a single pixel.  

Recently, Invariant Risk Minimization (IRM) \cite{arjovsky2019invariant} changed the landscape of the OOD research by introducing an invariant predictor that is defined without specifying a fixed discrete subset of observable variables.
\cite{arjovsky2019invariant}'s  definition of the invariant predictor does not use a priori knowledge of the underlying true model, such as DAG. 
For this reason, IRM as a method can be applied to practically any type of dataset. 
\cite{arjovsky2019invariant} has also shown that their invariant predictor agrees with the solution of the OOD problem \eqref{eq:objective_informal} when the true underlying model is linear.
However, whether their invariant predictor can achieve the OOD optimality on a nonlinear model is unknown\begin{footnote}{This claim was correct at the time of the publication of the early version of this paper on the preprint server~[Anonymous, 2020]. Recently, this problem was partially addressed by a follow-up preprint study that cites our paper. To keep our anonymity, we omit the explicit citation of the study.}\end{footnote}.
The question still remains as to whether it is possible to formulate \eqref{eq:objective_informal} as a \textit{invariance problem} when the system is nonlinear and the observable variables are not sorted in a way that respects some causal system.

The purpose of this study is to extend the range of problems on which the concept of invariance can be used to formulate the OOD generalization problem.
By presenting a novel theoretical condition for the OOD optimal invariant predictor, we show that we can in fact formulate the OOD problem as a \textit{invariance problem} in a more general setting than those considered before.   
In particular, we show that this is possible not only on the cases in which there isn't a causal variable that is representable as a subset of observable variables, but also on the cases in which the true underlying model is nonlinear.
Our theoretical condition, which we call \textit{controllability condition}, also generalizes the condition proposed in the causality-inspired work of \cite{rojas2018invariant}.
Moreover, when an additional set of conditions is met,  we can also formulate the OOD problem \eqref{eq:objective_informal}
as \textit{Maximal Invariant Predictor} (MIP) problem,  an information theoretic objective function with invariant constraint.
We discuss these theoretical results in Section \ref{sec:theory}.
In Section \ref{sec:method}, we also present Inter-environment Gradient Alignment (IGA) algorithm inspired from MIP, and discuss its ability to extrapolate in practice.
In Section \ref{sec:exp}, we show that our IGA performs competitively on benchmark experiments, including two of the unit-tests presented in \cite{aubinlinear}.

\section{Theory}
\label{sec:theory}
In this section, we present our main theoretical results about the question of when we can formulate the OOD problem \eqref{eq:objective_informal} as an \textit{invariance problem}, or a problem of finding a good invariant predictor.
We first provide our answer in a purely theoretical form in Section \ref{subsec:when_inv}, and then provide an answer with an objective function in Section \ref{subsec:MIP}. 
We begin this section with a set of notations. We use $X$ and $Y$ respectively to represent the input and the target random variable, and use $\mathcal{E}$ to represent the environmental random variable. 
The environmental variable $\mathcal{E}$ may represent any continuous variable that affects both the input variable $X$ and the output variable $Y$.
If $X$ is a picture of an animal and $Y$ is the animal label, the environmental variable $\mathcal{E}$ may be an aggregate variable whose features contain environmental factors such as weather and lighting conditions, for example.
In the context of causal studies \cite{rojas2018invariant, buhlmann2018invariance}, the environmental variable $\mathcal{E}$ may represent the set of all variables that are not causal to the output variable $Y$.
For a predictor $f$ that maps the range of $X$ to the range of $Y$, we measure its performance by KL or a symmetric Bregman divergence $\ell(f(X), Y) $ between $f(X)$ and $Y$.
Our goal in this section is to find a condition required for an invariant predictor to become the solution of the OOD problem \eqref{eq:objective_informal} based on this performance measure. 
In the next subsection, we present the concept of \textit{invariant predictor} more formally.

\subsection{Preliminary remarks on the invariant predictors used in causality-inspired researches}
\label{subsec:prelimi}
We define an invariant predictor to be a predictor that can be written as a function of a feature $h(X)$ that satisfies the \textit{some} invariance property.
The choice of invariance property that we focus in this study borrows much from those used in causality-inspired works \cite{buhlmann2018invariance, rojas2018invariant}.

If $\mathcal{X} = \{X_i; i=1,...d\}$ is the set of \textit{observable} variables,  many causality-based OOD studies  \cite{buhlmann2018invariance, rojas2018invariant} assume that there is a subset $S \subseteq \{ 1, ... d\}$ such that the conditional distribution of $Y$ given $\{X_j ; j \in S\}$ remains constant across all datasets.
\cite{buhlmann2018invariance, rojas2018invariant} define such $X_S = \{X_j ; j \in S\}$ as invariant/stable feature, and predict $Y$ using a function of $X_S$. 
If we use $X$ to denote the tensor whose $i$th coordinate is $X_i$, we may write $X_S$ as $M_S \odot X$,  or the coordinate-wise product between $X$ and the binary vector $M_S \in \{0, 1\}^d$ whose $j$th coordinate is $1$ whenever $j \in S$ and is $0$ otherwise. 
Thus, we may say that many causality-inspired works construct their invariant predictor as a function of a feature of the form $M_S \odot X$ that satisfies the invariance property $P(Y |  M_S \odot X, \epsilon) = P(Y |  M_S \odot X)$.

In our study, we also seek a variable that satisfies a similar invariance property, but we would like to extend our search-space to include the nonlinear features of $X$ that cannot be expressed as $M_S \odot X$ for a fixed $M_S$.
That is, we look for a generic nonlinear feature $h(X)$ that satisfies the invariance property $P(Y | h(X) ,\epsilon) = P(Y| h(X))$.
Thus, in this work, we say that $h(X)$ is an \textit{invariant feature} if it satisfies the invariance property $P(Y | h(X) ,\epsilon) = P(Y| h(X))$.
We then use $\mathbb{E}[Y| h(X)]$ with such $h(X)$ as our choice of the invariant predictor (which is, by definition, a function of $h(X)$.)
We will first show that, under the condition we call \textit{controllablity condition}, 
the invariant predictor $\mathbb{E}[Y| h(X)]$ with nonlinear invariant $h$ can solve the OOD problem \eqref{eq:objective_informal} even when the underlying true model is also nonlinear.
For more discussion of the types of invariance, please see Section \ref{sec:discussion}.

\subsection{Controllability Condition}
\label{subsec:when_inv}
Our following theoretical result provides a set of necessary conditions for the invariant predictor $\mathbb{E}[Y| h(X)]$ to be able to solve the OOD problem \eqref{eq:objective_informal}.
This result not only serves as our initial answer to the question of when the OOD problem can be written as \textit{invariance problem}, but it also generalizes the previously discovered conditions like those discussed in \cite{rojas2018invariant}.

\begin{thm}[\textbf{Controllability condition (Informal)}]
\mbox{} \\ 
We say that an invariant feature $h(X)$ satisfies the \textbf{controllability condition} if, for all $\epsilon \in \textit{supp}(\mathcal{E})$ there exists $\tilde \epsilon$ in $\textit{supp}(\mathcal{E})$,  a modified version of $\epsilon$ such that $ p(Y| X, \tilde \epsilon) =  p(Y| h(X), \tilde \epsilon)$.
If an invariant feature $h(X)$ satisfies the controllability condition, the invariant predictor $\mathbb{E}[Y| h(X)]$ solves the OOD problem \eqref{eq:objective_informal}.
\label{thm:controllability_informal}
\end{thm}
Thus, when there exists an invariant feature that satisfies the controllability condition, we can solve the OOD problem by finding the very invariant feature.
Just like \cite{arjovsky2019invariant}, our result is \textit{not} stated in terms of some known model representation like Causal DAG, and
it can therefore be applied to a wide range of situations. 
The proof of this statement uses measure theoretic probability and a variant of functional representation lemma \cite{achille2017emergence, Darmois, gamalNetworkinfo, peters2012identifiability}. 
For the formal statement of this theorem and its proof, please see Appendix \ref{appsec:proof}\begin{footnote}{
We emphasize that a feature $h$ satisfying the controllability condition is \textbf{\textit{not}} a feature that satisfies  $p(Y | X, \epsilon) \! = \! p(Y | h(X) , \epsilon) \! =  \! p(Y| h(X))$ for \textit{all} $\epsilon$. 
However, for a $h$ that satisfies the controllability condition, there is always a way to modify any $\epsilon \in \textit{supp}(\mathcal{E})$ into another environment $\epsilon' \in \textit{supp}(\mathcal{E})$ such that $p(Y | X, \tilde \epsilon) = \! p(Y | h(X) , \tilde \epsilon)$. 
The gist of the controllability condition is that we can define an optimality condition for $h$ without requiring $X \! \perp \! Y\!|(h(X), \epsilon)$ for all $\epsilon$.}\end{footnote}.

In Section \ref{subsec:MIP}, we will provide more concrete variant of \eqref{eq:objective_informal} defined with information theoretic objective (Maximal Invariant Predictor objective / MIP). 
However, although the result \eqref{eq:objective_informal} is abstract in its raw form, it is insightful in its own light.
Before we introduce MIP in Section \ref{subsec:MIP},  we therefore describe the intuition behind the controllability condition as well as its relation to \cite{rojas2018invariant}.

\subsubsection{Intuitive explanation of Controllability condition (Animal Classification example) }
Consider the example task of identifying the label $Y$ of the animal captured in a picture $X$ (e.g., Figure \ref{fig:dog_original}).
Each observation of $X$ is influenced by the environment variable $\mathcal{E}$, which may include factors such as the weather, lighting condition and background. 
A possible OOD goal of this example task is to find a predictor $f$ that can predict the label $Y$ from any picture $X$ in which most of the animal body is clearly visible.
Thus, $\textrm{supp}(\mathcal{E})$ in our consideration contains a wide range of environments in which a photographer can take a clear shot of an animal.
Our intuition tells us that the OOD-optimal invariant predictor in this example shall be a function of the feature $h^*(X)$ that consists exclusively of the appearance features of the animal in $X$; a function which does not depend on the background.
We will describe why such $h^*(X)$ can satisfy the controllability condition.

\textbf{How does $h^*$ satisfy the controllability condition?:}
For the sake of simplicity, let us suppose that \textit{wind condition} and \textit{lighting condition} are the only environmental features of $\mathcal{E}$ that affect the animal appearance $h^*(X)$ in the dataset. 
Consider then the picture of Figure \ref{fig:dog_original}, and let us use $\epsilon$ to denote the realization of the environmental variable $\mathcal{E}$ used to create Figure \ref{fig:dog_original}.
Note that, if we reproduce the wind-condition and the lighting condition of Figure \ref{fig:dog_original} in a photographic studio and take a picture of the same dog in front of a green-screen, we can take a picture like Figure \ref{fig:dog_green}.
Let us use $\tilde \epsilon$ to designate the environment used to create Figure \ref{fig:dog_green}.
Then the environment $\tilde \epsilon$ satisfies the following two properties.
(i) $\tilde \epsilon$ is a modified version of $\epsilon$ that agrees with $\epsilon$ on the \textit{wind} feature and the \textit{lighting} feature.
\begin{wrapfigure}{r}[0mm]{35mm}
  \begin{minipage}[t]{1.0\hsize}
      \includegraphics[width=\linewidth]{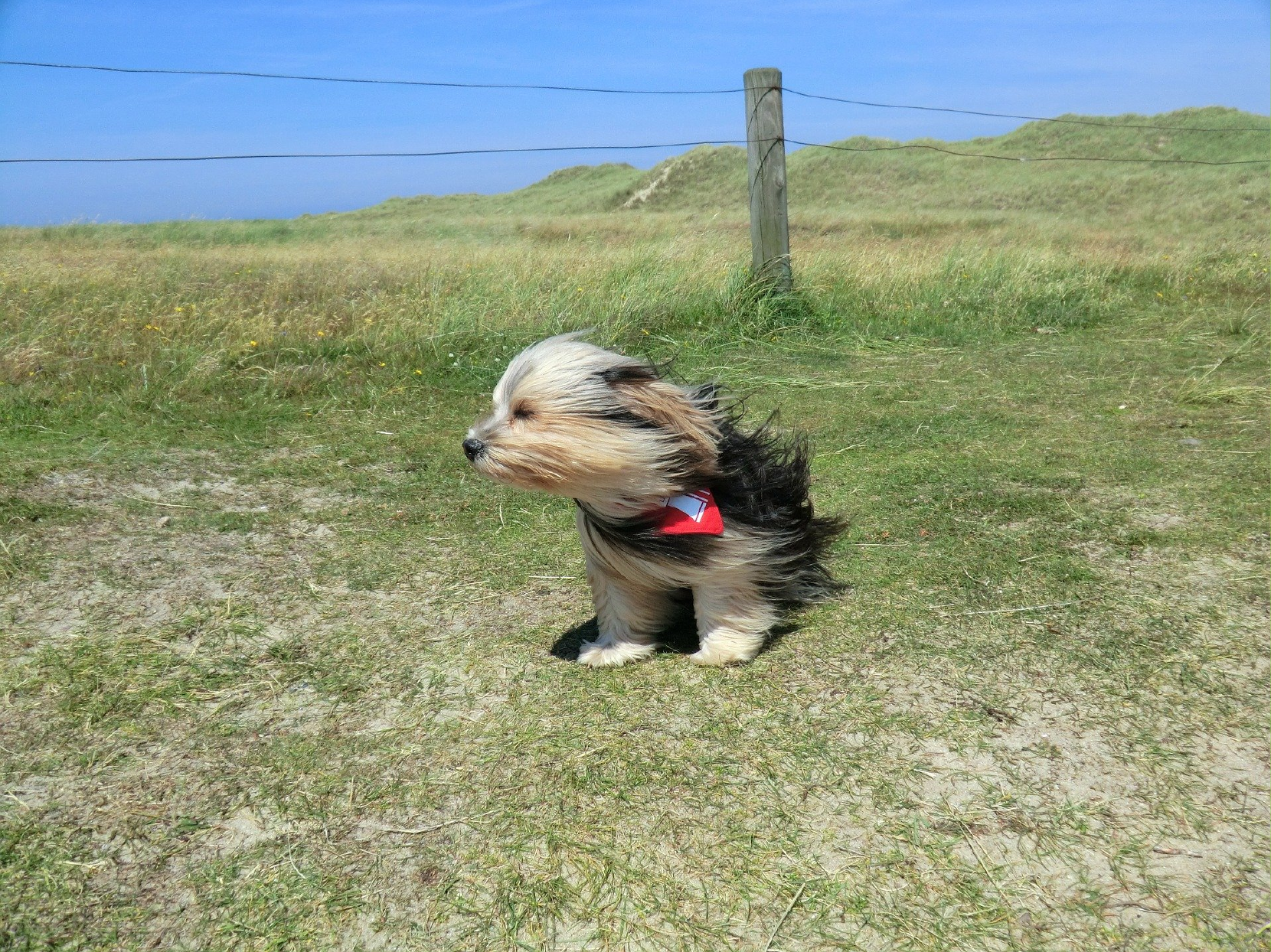}
      \vspace{-0.4cm}
      \caption{A picture of a dog on grass.}
    \label{fig:dog_original}
    \end{minipage}
    \begin{minipage}[t]{1.0\hsize}
      \centering
      \includegraphics[width=\linewidth]{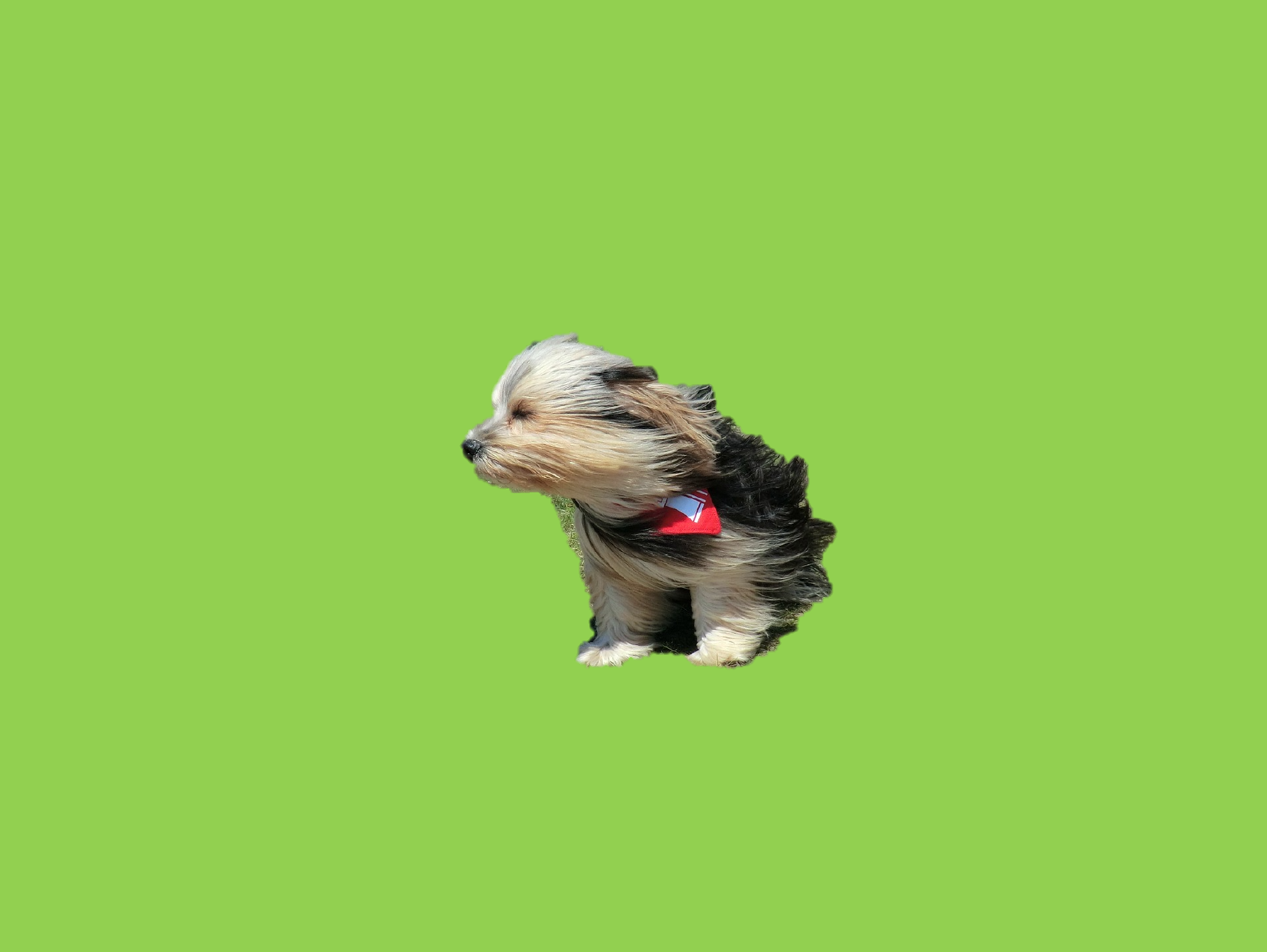}
      \vspace{-0.4cm}
      \caption{A picture of a dog taken in a photographic studio. }
    \label{fig:dog_green}
    \end{minipage}
  \label{fig:dog_pics}
\end{wrapfigure}
Because \textit{wind} is assumed to be the only environmental feature that is affecting the animal appearance, the appearance of the dog does not differ between the environment $\epsilon$ and the environment $\tilde \epsilon$ (See Figure \ref{fig:dog_original} and Figure \ref{fig:dog_green}). 
(ii) On the environment $\tilde \epsilon$,  $P(Y| h^*(X), \tilde \epsilon) = P(Y| X, \tilde \epsilon)$ because the background is empty and $X$ has no more information than $h^*(X)$ to infer the label $Y$. 
Moreover,  $\tilde \epsilon \in \text{supp}(\mathcal{E})$ most likely in our setup as well, because the animal body is visible in Figure \ref{fig:dog_green}. 
In a modern photographic studio, we can reproduce a wide variety of wind-conditions and lighting conditions.
Thus, even for another environment $\epsilon' \in \textrm{supp}(\mathcal{E})$ with a different combination of lighting and wind condition, we can construct the corresponding $\tilde \epsilon' \in \textrm{supp}(\mathcal{E})$ in a similar way.
This suggests that $h^*(X)$ satisfies the controllability condition. 

\textbf{Effect of the choice of $\textrm{supp}(\mathcal{E})$: }
Indeed, the argument so far depends on the size of the $\textrm{supp}(\mathcal{E})$. 
If $\textrm{supp}(\mathcal{E})$ in the OOD problem \eqref{eq:objective_informal} is so small that all animals in the datasets are photographed at their biological habitat, the background information in the picture $X$ would always contain more information about the animal label $Y$ than the animal appearance $h^*(X)$, and our choice of $h^*(X)$ in this example will not be optimal.
On the another extreme, if $\textrm{supp}(\mathcal{E})$ is so large that it contains the environment $\epsilon_{\mathit{darkness}}$ of complete darkness (so that all pictures taken in $\epsilon_{\mathit{darkness}}$ is completely black) , then $P(Y | h(X), \epsilon_{\mathit{darkness}}) \neq P(Y | h(X))$ and $h^*(X)$ does not satisfy the invariance property. In such a case, there will not be any reasonable solution to the OOD problem as well.

\subsubsection{Controllability condition vs the condition used by Rojas et al}
\label{subsec:relation_roj}
\cite{rojas2018invariant} claims to prove a sufficient condition for the OOD optimality of an invariant predictor $\mathbb{E}[Y| M_S \odot X]$.
However, they also require in their proof that, for every $\epsilon \in \text{supp}(\mathcal{E})$, there exists an environment $\tilde \epsilon \in \text{supp}(\mathcal{E})$ such that  $p(X, Y|\tilde \epsilon) = p(Y, M_S \odot X| \epsilon) p(Y, M_S^c \odot X | \epsilon)$ where $M_S^c$ is the complementary mask of $M_S$.
Fortunately, the formal version of theorem \ref{thm:controllability_informal} (theorem \ref{thm:controllability} in Appendix \ref{appsec:proof}) generalizes the necessary condition of \cite{rojas2018invariant} as a special case. 
We present this result as a corollary to \ref{thm:controllability_informal}. 
Please see corollary \ref{thm:Rojas_formal} in Appendix \ref{appsubsec:vs_Rojas} for the formal version of this claim.
\begin{cor}
The necessary condition for the OOD optimality in \cite{rojas2018invariant} is a special case of the controllability condition in theorem.  \ref{thm:controllability_informal}.
\label{thm:vs_Rojas} 
\end{cor} 

\subsection{Maximal Invariant Prefictor (MIP)} 
\label{subsec:MIP}
Although theorem \ref{thm:controllability_informal} provides an answer to the question of "when the OOD problem \eqref{eq:objective_informal} can be formulated as an \textit{invariant problem}", it defines the \textit{invariant problem} too abstractly, because theorem \ref{thm:controllability_informal} itself does not immediately suggest an objective function that may be used to find an OOD optimal invariant predictor. 
Therefore, we provide a variant of theorem \ref{thm:controllability_informal} that comes together with a trainable objective function whose solution agrees with the solution the OOD problem \eqref{eq:objective_informal}.
The following result holds even when the underlying model is nonlinear.

\begin{thm}[\textbf{MIP (Informal)}]
\mbox{} \\ 
Suppose that there exists at least one invariant feature that satisfies the controllability condition (theorem \ref{thm:controllability_informal}/definition \ref{def:controllability}), and that all other invariant features $\tilde h(X)$ can be written as a function of a common invariant feature $h_0(X)$. 
Then the invariant feature $h_0(X)$ agrees with the solution of
\begin{align}
\max_{h; P(Y | h(X), \epsilon) = P(Y |h(X))} I(Y ; h(X))  
 \label{eq:MIP_objective_informal}
\end{align}
and the invariant predictor $\mathbb{E}[Y| h_0(X)]$ solves the OOD problem \eqref{eq:objective_informal}. 
\label{thm:MIP_objective_informal}
\end{thm} 
We refer to \eqref{eq:MIP_objective_informal} as Maximal Invariant Predictor problem (MIP).  
This result essentially claims that, under appropriate set of conditions, we can reformulate (\ref{eq:objective_informal}) as an InfoMax objective \cite{linsker1988self} with invariance constraint.
The formulation of MIP also agrees with the hypothesis presented in \cite{ossan2019}. 
MIP itself is also similar to \cite{chang2020invariant} that aims to solve another problem that is different from the OOD problem \eqref{eq:objective_informal}.
For a more detailed comparison, please see Section \ref{sec:discussion}.

\section{Method} 
\label{sec:method}
In this section, we present an algorithm inspired by MIP. 
In Section \ref{subsec:parametrization}, we first explain our parametrization of $P(Y | h(X))$ that allows us to train the model while imposing the invariance constraint. 
In Section \ref{subsec:IGA}, we present Inter-environment Gradient Alignment (IGA) algorithm.
Finally, in Section \ref{subsec:IGA_linear}, we discuss the OOD loss of an IGA trained model and a case study of IGA on a linear example. 

\subsection{Model parametrization} 
\label{subsec:parametrization}
Given a model for $P(Y | h(X))$, the mutual information $I(Y; h(X))$ itself can be optimized at relative ease because the mutual information is related to loss functions like KL divergence and L2 loss. 
However, to encourage the constraint of $P(Y | h(X) ,\epsilon) = P(Y| h(X))$ for a nonlinear system without using a knowledge of underlying model like Causal DAG, we need a good black-box parameter representation to describe the relationship between $P(Y| h(X), \epsilon) = P(Y | h(X))$ for a generic $h$.

To do this, we adopt the idea of MAML \cite{finn17a} that parametrizes the conditional distribution on each task $\mathcal{T}_i$ as $P(Y| X ; \theta - \alpha \nabla_\theta \mathcal{L}_{\mathcal{T}_i}(\theta))$, where $\theta$ is a task-agnostic parameter of the base model $P(Y | X; \theta)$
and $\mathcal{L}_{\mathcal{T}_i}(\theta)$ is the loss value of $P(Y| X ;\theta)$ on task $\mathcal{T}_i$.
We first use this philosophy to model $P(Y | h(X), \epsilon)$, and then obtain the model for $P(Y| h(X))$.

We begin by making an analogy between \textit{environment} and \textit{task}, using 
$ \mathcal{L}_{\mathcal{E}'}(\theta) = \mathbb{E}_{X',Y'} [l (Q(Y' |X'; \theta), Y') | \mathcal{E}'] $ to represent the loss of a predictor distribution $Q(Y | X ; \theta)$ for the environment $\mathcal{E}'$. 
The distribution $Q(Y | X ; \theta)$ can be a black box distribution parametrized by Neural Networks.
We then model $P(y | h(X), \mathcal{E}')$ for each $y$ as 
\begin{align}
    P(y | h(X), \mathcal{E}') &:= Q(y| X; \theta - \alpha \nabla_\theta \mathcal{L}_{\mathcal{E}'}(\theta))  \label{eq:MAML_like_param_e} 
\end{align}
where $\theta$ is a $h$ specific parameter that is agnostic to the environment $\epsilon$ (Thus, $h$ in the LHS of \eqref{eq:MAML_like_param_e} is implicitly represented in the RHS as $\theta$).
Note that we herein used the notation $(Y', X', \mathcal{E}')$ to distinguish the variables $(X,Y)$ to be used at the time of inference from the variables from the \textit{training} variables $(Y', X', \mathcal{E}')$ that are used to to determine the environment specific model parameter $\theta - \alpha \nabla_\theta \mathcal{L}_{\mathcal{E}'}(\theta)$. 
In other words, $\mathcal{E}'$ might be correlated to $(X',Y')$ that is integrated away in $\mathcal{L}_{\mathcal{E}'}(\theta)$, but is independent from $(X,Y)$.
If we assume uniform continuity of $Q$ with respect to $\theta$, we can say that there exists some $\alpha > 0$ for which the following approximation holds with small error $O(\alpha^2)$ for any measurable $A$ in the range of $Y$(See Appendix \ref{apppsec:param_app}). ;
\begin{align}
    P(y \in A | h(X) ) &:= \mathbb{E}_Y[1_A(Y) | h(X)]  \\
    &=\mathbb{E}_\mathcal{E}'[\mathbb{E}_Y[1_A(Y) | h(X), \mathcal{E}']] \\
    &\cong \mathbb{E}_{\mathcal{E}'} [P(y \in A | h(X), \mathcal{E}')] 
\end{align}
Since this holds for all $A$,  $ P(y | h(X) ) \cong \mathbb{E}_{\mathcal{E}'} [P(y | h(X), \mathcal{E}')]$. Now by substituting  \eqref{eq:MAML_like_param_e} and using the fact that $\mathcal{E}'$ used in the model parameter is independent from $(X,Y)$,  
\begin{align}
    \mathbb{E}_{\mathcal{E}'} [P(y | h(X), \mathcal{E}')] &\cong \mathbb{E}_{\mathcal{E}'}[Q(y| X; \theta - \alpha \nabla_\theta \mathcal{L}_{\mathcal{E}'}(\theta))| X ] \\ 
    &\cong Q(y| X; \theta) - \alpha \nabla_\theta Q(y|X; \theta)^T  \nabla_\theta \mathbb{E}_{\mathcal{E'}}[\mathcal{L}_{\mathcal{E}'}(\theta) |X]  + O(\alpha^2) \\ 
    &= Q(y| X; \theta) - \alpha \nabla_\theta Q(y|X; \theta)^T  \nabla_\theta \mathbb{E}_{\mathcal{E'}}[\mathcal{L}_{\mathcal{E}'}(\theta)]  + O(\alpha^2) \\ 
    &\cong  Q(y| X; \theta - \alpha \nabla_\theta \mathbb{E}_{\mathcal{E'}}[\mathcal{L}_{\mathcal{E}'}(\theta)])   + O(\alpha^2)
    \label{eq:MAML_like_param} 
\end{align}

Because the explicit form of $h$ is absent in both \eqref{eq:MAML_like_param_e} and \eqref{eq:MAML_like_param}, the training based on our parametrizations do not treat the invariant feature $h(X)$ and the invariant predictor function $P(Y|h(X))$ separately. 
This approach is different from many other studies based on invariance, and we would discuss its pros and cons in Section \ref{sec:discussion} and Appendix \ref{appsubsec:method_lim}.
\begin{figure*}[t!]
  \centering
  \begin{tabular}{cc}
    \begin{minipage}[t]{0.55\hsize}
      \centering
      \includegraphics[width=\linewidth]{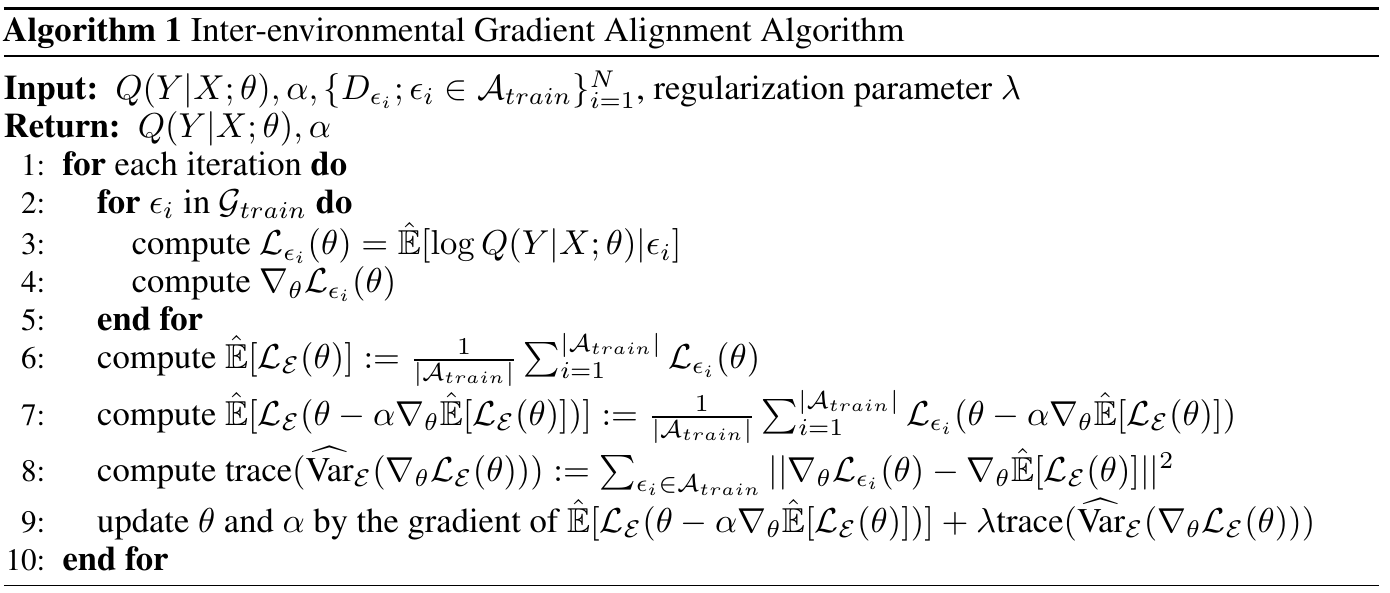}
      \vspace{-0.5cm}
      \tblcaption{IGA algorithm. All integration with respect to $\mathcal{E}$ are evaluated empirically, and each $\epsilon$ appears only as the index of dataset. }
    \label{fig:IGA_algorithm}
    \end{minipage}
    &
    \begin{minipage}[t]{0.4\hsize}
      \centering
      \includegraphics[width=\linewidth]{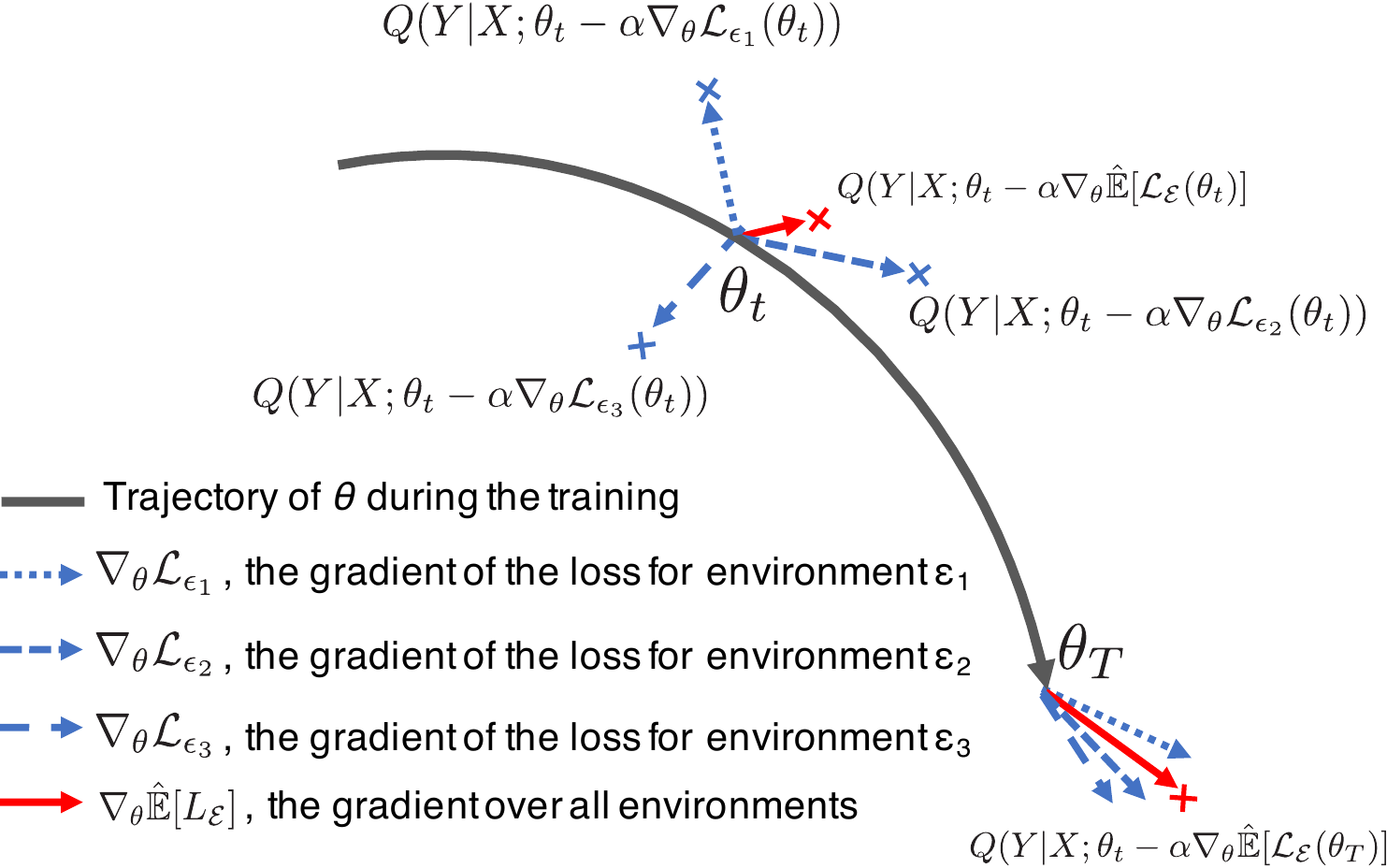}
      \vspace{-0.5cm}
      \caption{Relations among $\nabla_\theta \mathcal{L}_{\epsilon_i}$ during the training with IGA. $\theta$ in the figure is the parameter of the base black-box distribution $Q(Y| X; \theta)$.}
    \label{fig:IGA_overview}
    \end{minipage}
  \end{tabular}
  \label{fig:exp_para}
\end{figure*}
\subsection{Inter-environment Gradient Alignment Algorithm} 
\label{subsec:IGA}

MIP requires that  $d_{KL}(P(Y| h(X), \epsilon) \| P(Y | h(X)))]$ is small for every $\epsilon$. 
This can be achieved by making
$\mathbb{E}_{\mathcal{E}'} [d_{KL}(P(Y| h(X), \mathcal{E}') \| P(Y | h(X)))]$
small. 
Substituting our parametrization (\eqref{eq:MAML_like_param_e}, \eqref{eq:MAML_like_param}) into $\mathbb{E}_{\mathcal{E}'} [d_{KL}(P(Y| h(X), \mathcal{E}') \| P(Y | h(X)))]$ and working out the algebra, 
we obtain
\begin{align}
    \mathbb{E}[d_{KL}(P(Y| h(X), & \mathcal{E}') \| P(Y | h(X)))]  \nonumber \\ &\cong \mathbb{E}[d_{KL}(Q(Y| X; \theta - \alpha \nabla_\theta \mathcal{L}_{\mathcal{E}'}(\theta)) \| Q(Y | X ; \theta - \alpha \nabla_\theta
    \mathbb{E}_{\mathcal{E}'}[\mathcal{L}_{\mathcal{E}'}(\theta)]))] \\
    &\cong \alpha \  \textrm{trace} (\text{Var}_{\mathcal{E}'}(\nabla_\theta  \mathcal{L}_{\mathcal{E}'}(\theta)))
    \label{eq:IGA_reg}
\end{align}
For more detailed derivation of \eqref{eq:IGA_reg}, please see Appendix \ref{appsec:derivation}. 
Because $I(Y ; h(X))$ can be maximized by minimizing the loss $\mathcal{L}_{\mathcal{E}'}(\theta - \alpha \nabla_\theta \mathbb{E}_{\mathcal{E}'}[\mathcal{L}_{\mathcal{E}'}(\theta)])$, we shall minimize the following with respect to $\theta$:
 \begin{align}
  \mathbb{E}_{\mathcal{E}'}[\mathcal{L}_{\mathcal{E}'}(\theta - \alpha \nabla_\theta \mathbb{E}_{\mathcal{E}'}[\mathcal{L}_{\mathcal{E}'}(\theta)])] + \lambda ~  \text{trace}(\text{Var}_{\mathcal{E}'}(\nabla_{\theta} \mathcal{L}_{\mathcal{E}'}(\theta)))  
  \label{eq:master_final}
\end{align}
Because the objective \eqref{eq:master_final} encourages the gradient $\nabla_{\theta} \mathcal{L}_{\mathcal{\epsilon}}$ evaluated on each $\epsilon$ to align with the gradients evaluated on other $\epsilon$s, we call the algorithm for the objective function \eqref{eq:master_final} as \textbf{Inter-environment Gradient Alignment} (IGA) algorithm. 
This formulation agrees with the \textit{general} formulation of IRM \cite{arjovsky2019invariant} (as opposed to \textit{IRMv1}) when $Y$ is a categorical variable. 
MIP thus provides justification to IRM for a nonlinear case when the task of interest is a classification problem.

\subsection{IGA in practice} 
\label{subsec:iga_practice}
Now that we have proposed the objective function \eqref{eq:master_final} for our MIP-inspired algorithm, how is it carried out in practice and what is the range of environments on which the empirical solution of IGA can generalize?

In the algorithm, we evaluate \eqref{eq:master_final} on a set of datasets $\{D_{\epsilon_i}; \epsilon_i \in \mathcal{A}_{\mathit{train}} \}_{i=1}^N $, with each $D_{\epsilon_i}$ consisting of a set of input-output pairs $\{(x_{ij},y_{ij})\}_{j=1}^{N_i}$ drawn from $P(X, Y | \epsilon_i)$.
Also, the variance and the expectation with respect to $\mathcal{E}$ is taken empirically. 
For example, $\mathbb{E}_\mathcal{\mathcal{E}} [ \mathcal{L}_\mathcal{E}(\theta)]$ is approximated by the empirical average $\widehat{\mathbb{E}_\mathcal{\mathcal{E}}} [ \mathcal{L}_\mathcal{E}(\theta)] =\frac{1}{N} \sum_{i=1}^N    \mathcal{L}_{\epsilon_i}(\theta) =  \frac{1}{N} \sum_{i=1}^N   \frac{1}{N_i} \sum_{j=1}^{N_i} \ell(Q(Y| (x_{ij}); \theta), y_{ij})$.
We emphasize that IGA does not require the user to specify the identity of the environment $\epsilon$;  each environment $\epsilon$ appears only as an index in the equation.
Table \ref{fig:IGA_algorithm} is an outline of our algorithm.
Figure \ref{fig:IGA_overview} is a visualization of the relations amongst $\{\nabla_\theta \mathcal{L}_{\epsilon_i}(\theta)\}_i$ in IGA. 

\textbf{IGA loss as a bound of OOD loss on linear combinations of $\mathcal{A}_{train}$: } It turns out that, when evaluated over finite $\mathcal{A}_{train}$, (7) can actually bound the OOD loss on a set of linear combinations of training distributions. Namely, let $\Delta_\eta = \{ \{\alpha_\epsilon\} ; \alpha_\epsilon > - \eta, ~  \sum_\epsilon \alpha_\epsilon =1 \}$, and consider the set of distributions defined by
$\mathcal{E}_\eta  = \{ \sum_{\epsilon \in \mathcal{A}_{train}} \alpha_\epsilon P_\epsilon ; \alpha \in \Delta \eta\}$.
Then it turns out that
\begin{align}
\sup_{\epsilon \in \mathcal{E}_\eta } L_\epsilon(\theta - \alpha \nabla_\theta \mathbb{E}_\mathcal{E}[\mathcal{L}_\mathcal{E}(\theta)])  \leq \mathbb{E}_\mathcal{E}[\mathcal{L}_\mathcal{E}(\theta - \alpha \nabla_\theta \mathbb{E}_\mathcal{E}[\mathcal{L}_\mathcal{E}(\theta)])] + \lambda_\eta ~  \text{trace}(\text{Var}_{\mathcal{E}}(\nabla_{\theta} \mathcal{L}_{\mathcal{E}}(\theta))
\end{align}
approximately for $\lambda_\eta$ that is monotonic in $\eta$. 
Thus, the OOD loss of an IGA trained model is bounded by its loss on $\Delta_\eta$.
For the details of this claim, please see \ref{appsec:app_bound} in the Appendix.

\textbf{An analysis of a simple linear example : }
\label{subsec:IGA_linear}
As a case study, we present our analysis of IGA on the following linear example \footnote{This example was raised by the author of \cite{arjovsky2019invariant} in personal communication.} 
\begin{align}
    X_1 &= N_0, \ Y = X_1 + N_1, \ X_2 = \mathcal{E}Y + N_2,  
\end{align}
where $N_k$ are independent standard normal noises and $\mathcal{E}$ is an arbitrary, scaler valued environmental random variable with non-zero variance.
For this problem,  $X_1$ satisfies $P(Y | X_1, \epsilon) = P(Y | \epsilon)$ for all $\epsilon$, and $X_1$ also stands as the solution of OOD problem over all real values of $\mathcal{E}$.
We claim that $X_1$ is also the solution of IGA \eqref{eq:master_final}. 

Let the base distribution $Q(Y| X; \theta)$ be the distribution of $\hat Y_\theta = \theta_1 X_1 + \theta_2 X_2$, and let 
$\mathcal{L}_\epsilon(\theta) = \mathbb{E}[\|Y - \hat Y\|^2 |\epsilon]$ be the loss function. 
As we discussed above, we parametrize   
$P(Y| h(X), \epsilon)$ as $Q(Y| X; \theta - \alpha \nabla_\theta \mathcal{L}_\epsilon(\theta))$ with $\theta$ implicitly determining $h$.
We assume that the empirical variance  $\widehat{\text{Var}}(\mathcal{E})$ is non zero. 
We can then compute  $\text{trace}(\widehat{\text{Var}}(\nabla_{\theta} \mathcal{L}_{\mathcal{E}}(\theta)))$
in the constraint as the sum of the followings:
\begin{align}
    \widehat{\text{Var}}(\nabla_{\theta_1} \mathcal{L}_{\mathcal{E}}(\theta)) = 4\widehat{\text{Var}}(\epsilon)\theta_2^2 , \ \widehat{\text{Var}}(\nabla_{\theta_2} \mathcal{L}_{\mathcal{E}}(\theta)) = 4\hat{\mathbb{E}}_\epsilon \Big[ \big(
    (\epsilon - \hat{\mathbb{E}}_\epsilon[\epsilon])(\theta_1 -2) + 2(\epsilon^2 - \hat{\mathbb{E}}_\epsilon[\epsilon^2])\theta_2
    \big)^2 \Big] \nonumber
\end{align}
Under the aforementioned assumption of $\widehat{\text{Var}}(\mathcal{E})> 0$, equating the sum of these two variances to $0$ would force $\theta = [2, 0]$. 
Because $\nabla_\theta \mathcal{L}_\epsilon(\theta) \big \vert_{\theta = [2,0]} = [2, 0]$, the solution of IGA would thus be
$Q(Y| X; [2,0] - \alpha [2,0])$ or $(1-\alpha)2 X_1$ in the variable form. 
If we optimize the loss of this function about $\alpha$, we obtain $(1- 0.5)2 X_1 = X_1$ as the optimal solution, which is indeed our intended answer.

\section{Experiment} 
\label{sec:exp} 
\subsection{Invariance Unit Tests}
\label{subsec:exp_unit_test}

Recently, \cite{aubinlinear} proposed Invariance Unit Tests, a set of linear problems to serve as a benchmark for OOD generalization problems. 
To construct the set of datasets for each one of these problems, we first sample a set of environments $\{\epsilon\}$.
Then, for each sampled $\epsilon$, we collect a dataset $D_\epsilon := \{ (x_i^\epsilon, y_i^\epsilon)\}$ from $p(x, y | \epsilon)$.
In the setup of \cite{aubinlinear}, each instance of $x^\epsilon$ is a pair of (a) $x_{\mathit{inv}}^\epsilon \in \mathbb{R}^{d_{\mathit{inv}}}$ that elicits invariant correlations and (b) $x_{\mathit{spu}}^\epsilon \in \mathbb{R}^{d_{\mathit{spu}}}$ that elicits spurious correlations. 
The goal of the experiments in \cite{aubinlinear} is to train a predictor for the target variable $y^\epsilon$ that depends exclusively on $x_{\mathit{inv}}^\epsilon$.
To test the performance of our IGA, we conducted experiments on the unit-tests designated as (2, 2s, 3, 3s) in \cite{aubinlinear}. 
In the unit-tests designated with labels containing $s$ (2s, 3s), the observations are scrambled by a full rank matrix (each observation is presented as $M X$ for some matrix $M$), so that it is difficult to construct an invariant predictor as a function of a small discrete subset of observations.
In this set of experiments, we compare our IGA against IRM, ERM and ANDMask \cite{parascandolo2020learning}.
ANDMask has a somewhat similar philosophy as our IGA, because it aims to minimize the error by updating the model only on the parameters on which the sign of the gradient of the loss is the same for most environments.
We did not conduct tests on (1, 1s), because $P(y | x_{\mathit{inv}} , \epsilon) \neq P(y | x_{\mathit{inv}})$ in these settings (please also see Section \ref{sec:discussion}).
We used the published code of the original paper \cite{aubinlinear} to conduct all experiments. 
For more details of the experimental settings, please see Appendix \ref{appsubsec:exp_unit_test}.
Table \ref{fig:IUT_table} summarizes the result of Invariance unit tests.
As expected, the performance of all models improve with the number of environments used in the training ($n_{\mathit{env}}$). 
Also, as reported in \cite{aubinlinear}, ANDMask performs well on unscrambled environments.
We can see in Table \ref{fig:IUT_table} that IGA performs competitively on all examples.
 

\subsection{Colored MNIST \& Extended Colored MNIST}
\label{subsec:mnist}
To test the performance of our IGA on an image dataset, we conducted a set of experiments on \textit{Colored-MNIST} (\textit{C-MNIST}) in \cite{arjovsky2019invariant} as well as its extension,   \textit{Extended Colored-MNIST} (\textit{EC-MNIST}) 
For C-MNIST, the OOD-optimal invariant predictor is a function of the feature $h(X)$ such that $P(h(X) | \epsilon) = P(h(X))$, which is also a solution of the Adversarial Domain Adaptation (ADA) \cite{li2018advdomain}.
EC-MNIST is an extension of C-MNIST in which the OOD-optimal invariant predictor is not necessarly a function of such a feature.  We therefore describe EC-MNIST first.

\textbf{Extended Colored-MNIST (EC-MNIST)} \\
In EC-MNIST, each $(x,y)$ in the environment $\epsilon = (\epsilon_{ch1}, \epsilon_{ch2})$ is constructed as follows:
{\setlength{\leftmargini}{15pt}
\begin{enumerate}
\setlength{\parskip}{0.025cm} 
\setlength{\itemsep}{0.025cm} 
\item Set $x_{ch2}$ to $1$ with probability $\epsilon_{ch2}$. Set it to $0$ with probability $1- \epsilon_{ch2}$. 
\item Generate a binary label $\hat y_{obs}$ from $y$ with the following rule: $\hat y_{obs}=0$ if $y \in \{0 \sim 4\}$ and  $\hat y_{obs}=1$ otherwise.  
If $\epsilon_{ch2}=k$, construct $y_{obs}$ by flipping $\hat y_{obs}$ with probability $p_k$ ($k \in \{0, 1\}$).
\item Put $y_{obs}= \hat x_{ch0}$, and construct $x_{ch0}$ from  $\hat x_{ch0}$ by flipping $\hat x_{ch0}$ with probability $\epsilon_{ch0}$.
\item Construct $x_{obs}$ as $x_{fig} \times  [x_{ch0}, (1-x_{ch0}), x_{ch2}]$.  As an RGB image, this will come out as an image in which the red scale is \textit{turned on} and the green scale is \textit{turned off} if $x_{ch0}=1$, and other-way around if $x_{ch0} =0$. 
Blue scale is turned on only if $x_{ch2} = 1$. 
\end{enumerate}}
Appendix Figure \ref{appfig:mnist} is the graphical model for the generation of EC-MNIST. 
We emphasize that only $(Y_{obs}, X_{obs})$ are assumed observable in the experiment, and that the node decomposition in Appendix Figure \ref{appfig:mnist} is \textit{not} assumed unknown. 
At the training time, the machine learner will be given a set of datasets $\mathcal{D}_{train} = \{D_{\epsilon_i} ; i=1,...,N_{train}\}$ in which $D_{\epsilon_i}$ is a set of observations gathered from $P(X,Y| \epsilon_i)$ and $R_{train} = \{\epsilon_i; i=1,...,N_{train}\}$ is the set of training environments.
At the test time, the learner will be challenged to make an inference of $Y_{obs}$ from $X_{obs}$ on the dataset drawn from $P(X,Y| \epsilon_{test})$ with $\epsilon_{test} \not \in R_{train}$.
The model is evaluated based on OOD performance \eqref{eq:objective_informal}, or the performance in the worst environment among the union of $R_{train}$ and the test environments.

For our EC-MNIST, the theoretical upper bound for the probability of correctly predicting  $Y_{obs}$ is $\epsilon_{ch2} \max \{p_0, 1-p_0\} + (1- e_{ch2})\max \{p_1, 1-p_1\}$.
In this problem, $X_{fig}$ is the only variable that is independent from $\mathcal{E}$, and it is therefore the solution of \cite{li2018advdomain}.  
However, $X_{ch2}$ together with $X_{fig}$ can create a better predictor than $X_{fig}$ alone. 
In fact, the oracle prediction by $X_{fig}$ alone can attain an average value as high as $\max\{ \epsilon_{ch2}p_0+ (1-\epsilon_{ch2})p_1 , \epsilon_{ch2}(1-p_0)+ (1-\epsilon_{ch2})(1-p_1)\}$, which is lower than the that of the $[X_{fig}, X_{ch2}]$ oracle. 
This fact follows from Fatou's lemma \cite{folland2013real}. 
Thus, ADA cannot find the optimal solution in this case. 
IRM \cite{arjovsky2019invariant} also discusses such a case in their work.

We compared our algorithm against Invariant Risk Minimization (IRM)\cite{arjovsky2019invariant}, Empirical Risk Minimization (ERM), and the oracle(s).
The right column in Table \ref{tab:result_table} compares the results of the algorithms in terms of the OOD performance \eqref{eq:objective_informal}.
We perform better than both ERM and IRM.
 We also perform better than the $X_{fig}$ oracle, which is equivalent to the upper bound of ADA. 
 Because $X_{ch2}$ is necessary in order to outperform the $X_{fig}$ oracle (see Appendix \ref{appsubsec:cmnist}), our result suggests that the IGA-trained models are actually using the feature $X_{ch2}$ in making the prediction of $Y$.
 Figure \ref{fig:ecmnist_result_app}(a)(b) plots the OOD accuracy of IGA-trained models against the regularization parameter $\lambda$ in Table \ref{fig:IGA_algorithm}.
In general, training with larger regularization parameters promotes the OOD performance.
The OOD accuracy plateaus around $\lambda \sim 10^4$. 

\textbf{Colored MNIST (C-MNIST)}\\
The original \textit{C-MNIST} in \cite{arjovsky2019invariant} is a special case of our EC-MNIST in which the distribution of $x_{ch2}$ does not vary with $\epsilon$. 
Figure \ref{appfig:mnist} in Appendix is a schematic of the data generation process of C-MNIST.
In C-MNIST, the OOD-optimal predictor can be constructed with $X_{fig}$ alone, and it achieves the optimal OOD accuracy of $\max(1- p, p)$.
The left column in Table \ref{tab:result_table} compares the results of the algorithms in terms of the OOD accuracy \eqref{eq:objective_informal}. 
Again, our method outperforms both ERM and IRM. 

\begin{figure}[h]
\begin{tabular}{cc}
\begin{minipage}[t]{0.625\hsize}
\includegraphics[width=\linewidth]{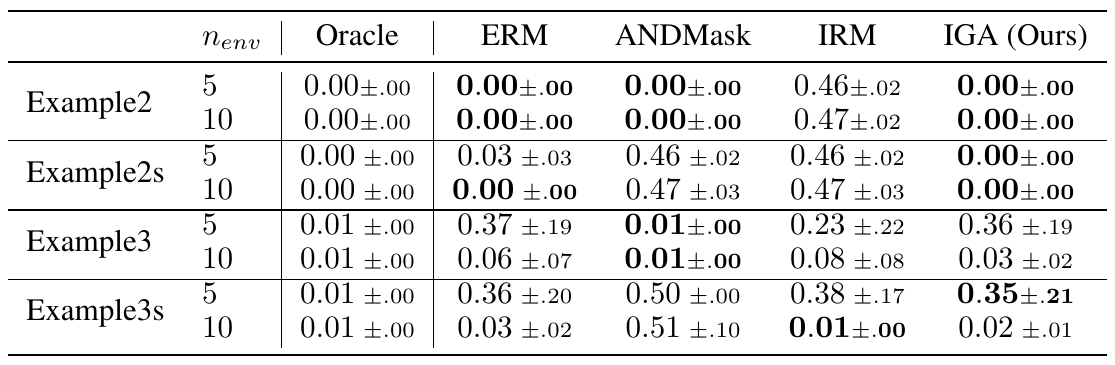}
\vspace{-0.5cm}
\tblcaption{Test errors for all algorithms and datasets for the invariant unit tests in \cite{aubinlinear}  $(d_{inv}, d_{spu}, n_{env}) = (5, 5, n_{env})$ ($n_{env}$ is the number of training environments). We provide the full version of this table in Appendix table \ref{appfig:IUT_table}} 
\label{fig:IUT_table}
\end{minipage}
\hspace{0.2cm}
\begin{minipage}[t]{0.325\hsize}
\includegraphics[width=\linewidth]{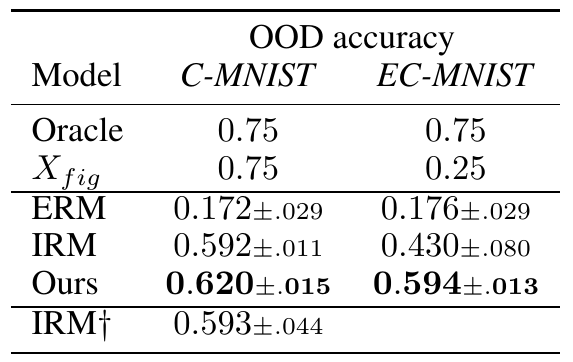}
\vspace{-0.5cm}
\tblcaption{Numerical performance of OOD algorithms. $X_{fig}$ designates the \textit{figure-only} oracle. $\dagger$ is the OOD accuracy of IRM reported in \cite{ahuja2020invariant}.}
\label{tab:result_table}
\end{minipage}
\end{tabular}
\end{figure}

\begin{figure*}[h]
\begin{tabular}{cccc}

\begin{minipage}[t]{0.24\hsize}
\stackunder[5pt]{\includegraphics[width=\linewidth]{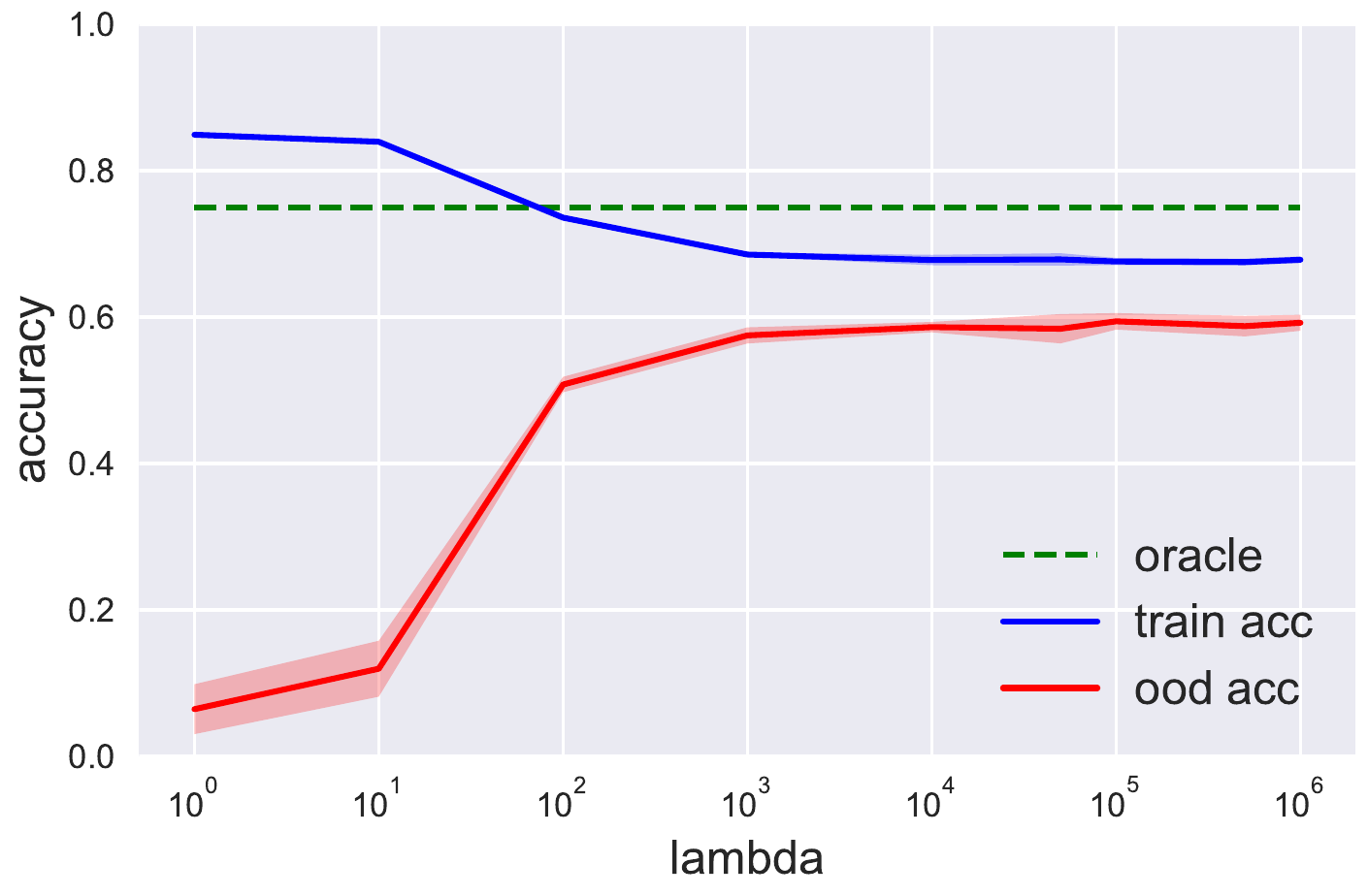}}{(a) IRM on \textit{C-MNIST}}
\end{minipage}
\centering
\begin{minipage}[t]{0.24\hsize}
\stackunder[5pt]{\includegraphics[width=\linewidth]{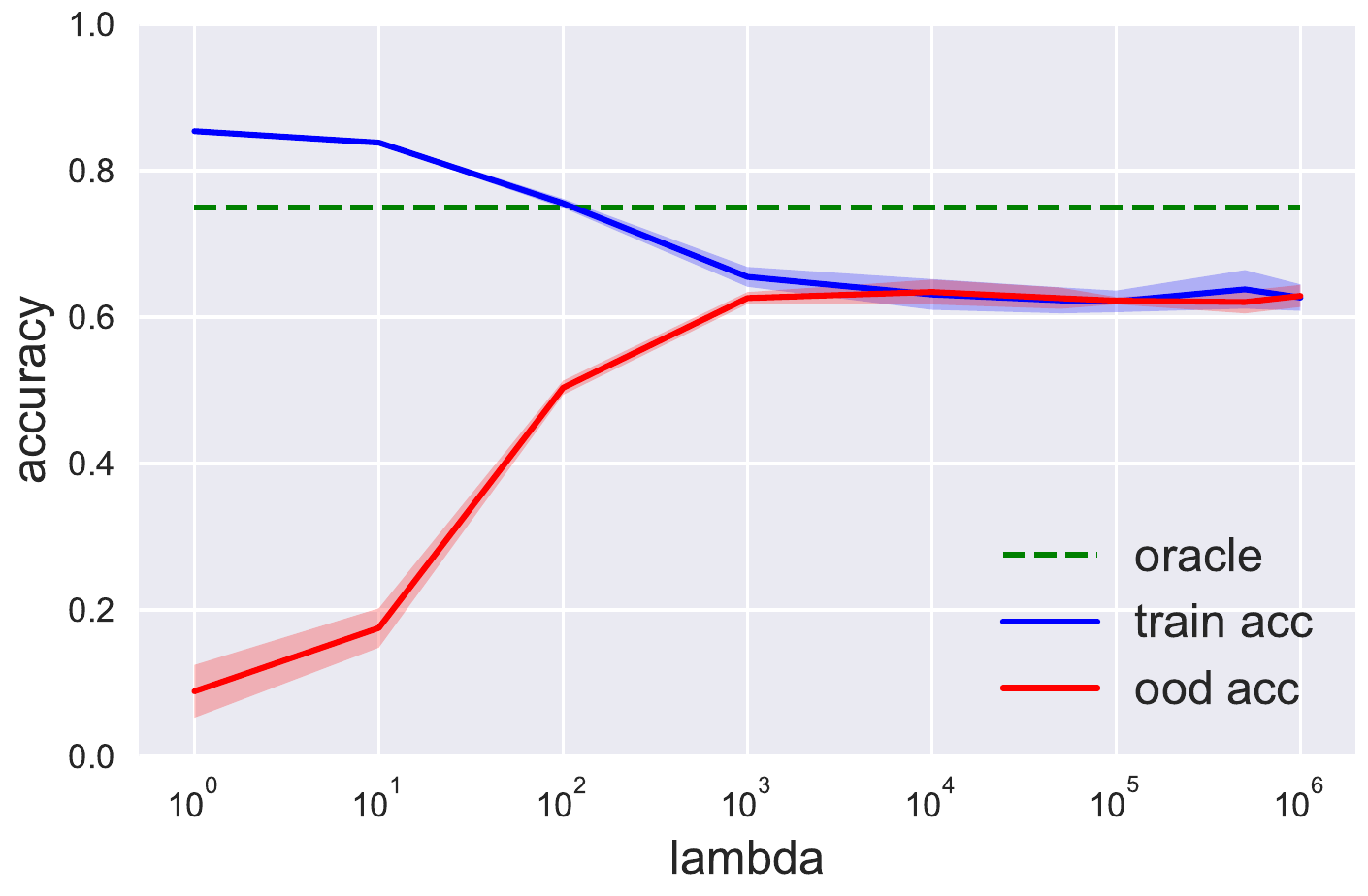}}{(b) Ours on \textit{C-MNIST}}
\end{minipage}
\label{fig:cmnist_result}

\begin{minipage}[t]{0.24\hsize}
\begin{center}
\stackunder[5pt]{\includegraphics[width=\linewidth]{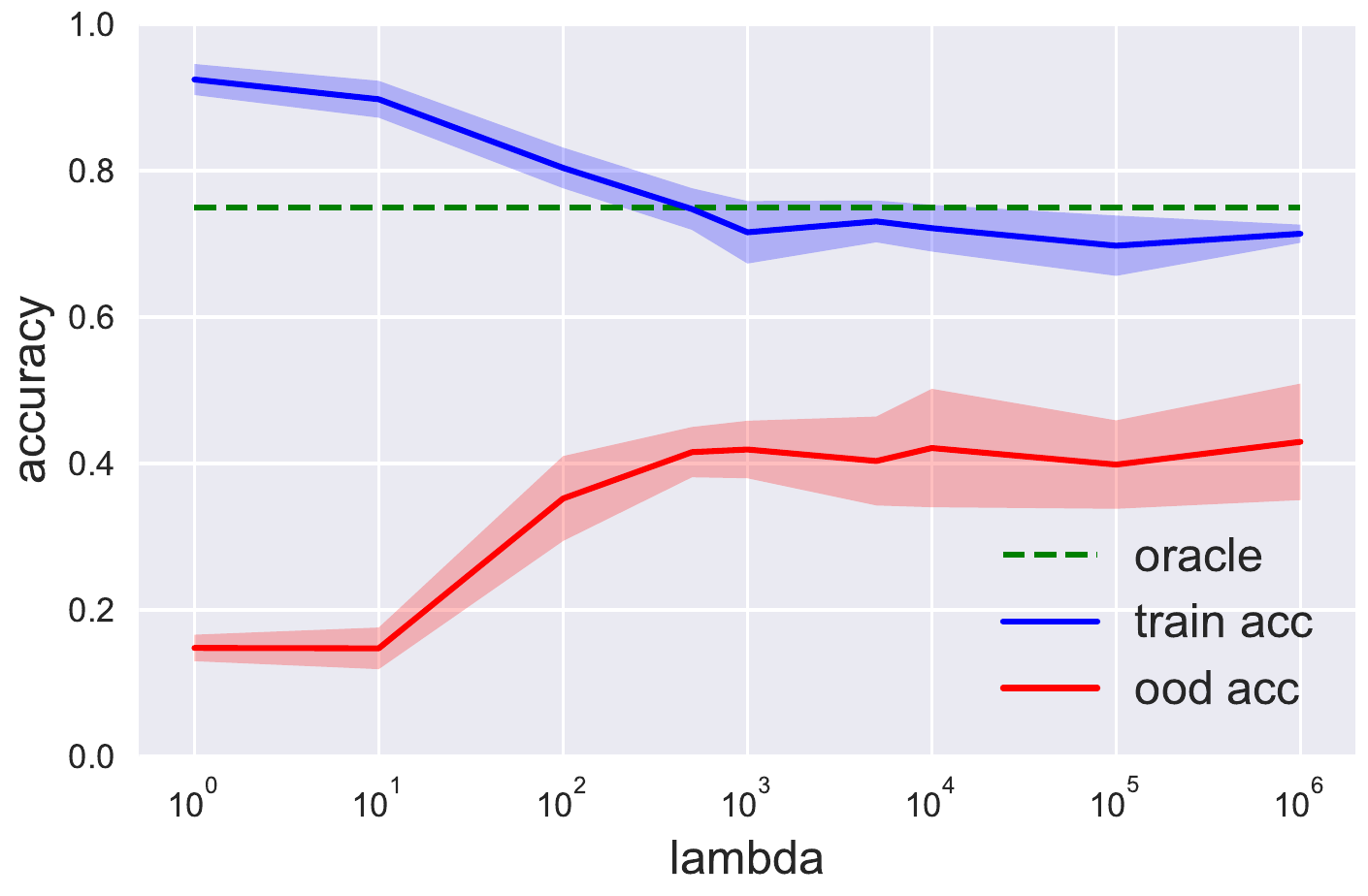}}{(c) IRM on \textit{EC-MNIST}}
\end{center}
\end{minipage}
\begin{minipage}[t]{0.24\hsize}
\begin{center}
\stackunder[5pt]{\includegraphics[width=\linewidth]{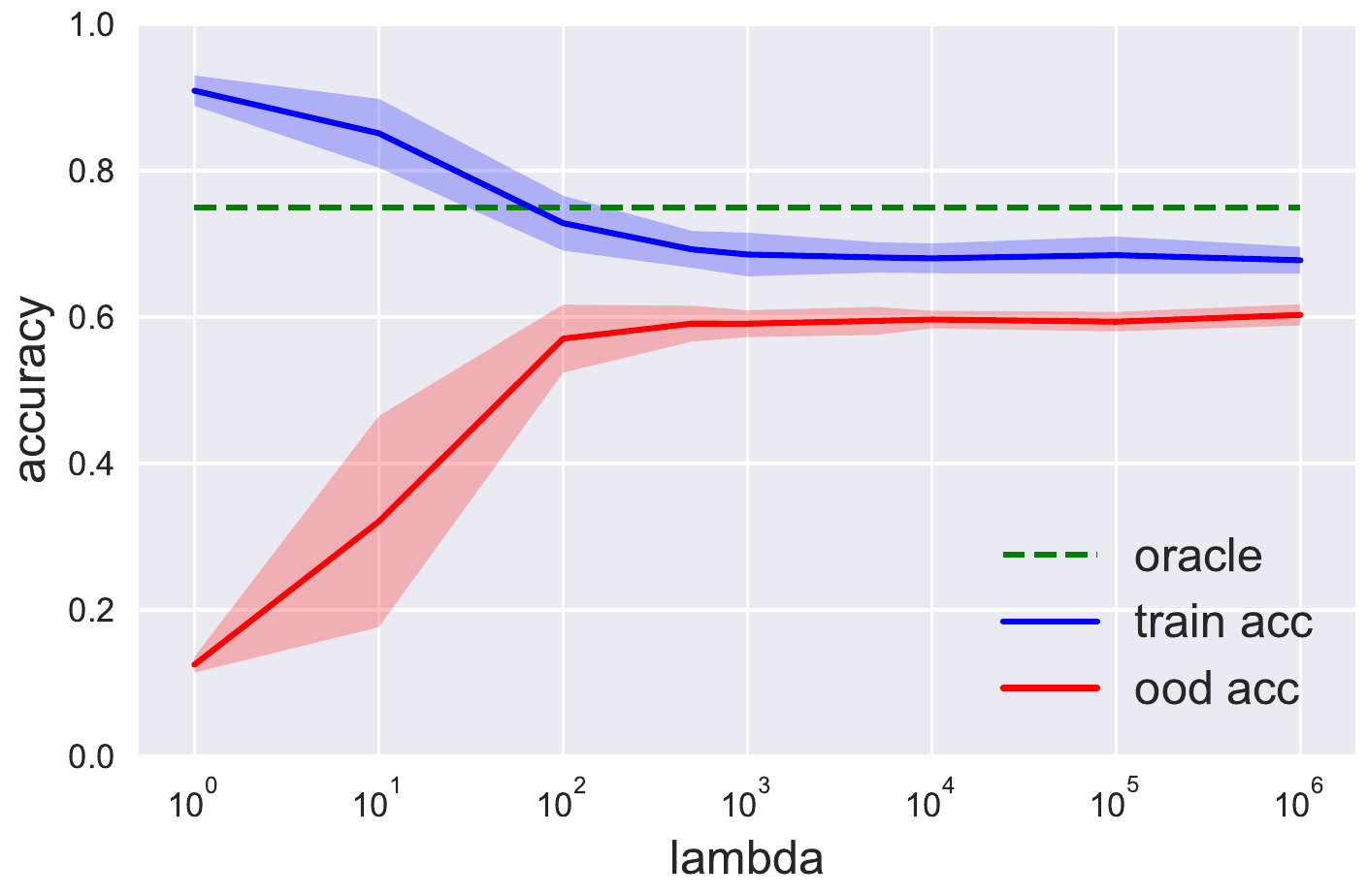}}{(d) Ours on \textit{EC-MNIST}}
\end{center}
\end{minipage}
\end{tabular}
\caption{The plot the accuracies on \textit{C-MNIST} ($p=0.25$) and \textit{EC-MNIST} ($p_0=0.25,p_1=0.75$) against the regularization parameter $\lambda$}
\label{fig:ecmnist_result_app}
\end{figure*}

\section{Related Works \& Discussion}
\label{sec:discussion}
\vspace{-0.3cm}
\textbf{Types of invariance \ }Recent studies of invariant predictor differ by the \textit{type of invariance} that is assumed to hold across environments.
As we mentioned in the introduction, many causality inspired methods \cite{rojas2018invariant, prevent19a} assume that there exists a causal mechanism that remains constant across different datasets.
In equation, they often assume that there exists a set of indices $S \subset \{1, ....,d\}$ such that $P(Y |  M_S \odot X , \epsilon) = P(Y | M_S \odot X)$, where $M_S$ is the binary mask function corresponding to $S$.  
\cite{chang2020invariant} also allows $S$ to depend on $X$.
The type of invariance that we investigate in this study is akin to this type of invariance, except that we only assume the existence of a possibly nonlinear function $h$ with $P(Y | h(X) , \epsilon) = P(Y | h(X))$. 
Meanwhile, \cite{arjovsky2019invariant} studies the set of features $h(X)$ that satisfies $\mathbb{E}[Y|h(X), \epsilon] = \mathbb{E}[Y |h(X)]$.
This is a less restrictive form of invariance because although $P(Y | h(X) , \epsilon) = P(Y | h(X))$ implies $\mathbb{E}[Y |h(X), \epsilon] = \mathbb{E}[Y |h(X)]$, its inverse does not hold in general. 
However, an exception occurs when $Y$ is a categorical variable representable as a one-hot vector, because $P(Y = e_k | h(X)) = \mathbb{E}[Y_k | h(X)]$. 
In a related note, we did not conduct experiments on unit-test 1 and 1s in Section \ref{subsec:exp_unit_test} because they  are both regression problems in which $x_{inv}$ satisfies $\mathbb{E}[Y |x_{inv}, \epsilon] = \mathbb{E}[Y |x_{inv}]$ only.
 \cite{kamath2021does, rosenfeld2021the} also stuides 
 the type of invariance discussed in \cite{arjovsky2019invariant}.

\textbf{MIP as an objective function \ } Our MIP is very much related to \cite{chang2020invariant}. 
However, instead of the commonly used OOD  objective \eqref{eq:objective_informal},  they define their objective as the minimization of $L^*_{test}  = \max_{\epsilon_a} H ( P(Y| M_S \odot X, \epsilon_a); P(Y| M_S \odot X, \epsilon \in \mathcal{A}_{train}))$ where $H$ is cross entropy and  $\mathcal{A}_{train}$ is the set of training environment.
The extent of their claims is also limited to a specific DAG model of their interest.
%
MIP as an objective function is not too special on its own, as it is just a variant of constrained maximization of mutual information (constrained InfoMax). 
One traditional form of constrained InfoMax comes with a constraint on entropy \cite{Constrained_MI}. 
However, because the invariance constraint we impose in MIP has a flavor of model selection, we may say that MIP in spirit is more closely related to Information Bottleneck (IB) \cite{DVIB, IB_ganso} and feature selection \cite{Learning2Explain, Vselection}. 
Developing further theoretical connections between the OOD problem and these methods is therefore a fascinating direction of research.

\textbf{Methodological limitation of IGA \ }One limitation of IGA is the difficulty of evaluating the regularization term $\text{Var}_\mathcal{E}(\nabla_\theta \mathcal{L}_\mathcal{E}(\theta))$.
To be able to apply IGA, we need to be able to empirically evaluate $\text{Var}_\mathcal{E}(\nabla_\theta \mathcal{L}_\mathcal{E}(\theta))$ correctly. 
Recall that $\text{Var}_\mathcal{E}(\nabla_\theta \mathcal{L}_\mathcal{E}(\theta))$ is a variance value taken with respect to the enviromnental variable $\mathcal{E}$, and that $\nabla_\theta \mathcal{L}_\mathcal{E}(\theta)  = \nabla_\theta \mathbb{E}[\ell(f_\theta(X), Y)]$ where $\ell$ is a divergence function and $f_\theta$ is the invariant predictor parametrized by $\theta$.
If the batch size during the training is too small,  the empirical evaluation of the expectation in $\nabla_\theta \mathcal{L}_\mathcal{E}(\theta)$ would have a non-trivial variance on its own and the evaluation of  $\text{Var}_\mathcal{E}(\nabla_\theta \mathcal{L}_\mathcal{E}(\theta))$ would be flawed.
To resolve this problem, we trained our invariant predictor with full batch size. 
Unfortunately, we could not train a good invariant predictor with smaller batch sizes.
This fact is barring us from applying our method to larger datasets such as CIFAR10 or Imagenet.
One important future work is to find a high-precision differentiable estimation for $\text{Var}_\mathcal{E}(\nabla_\theta \mathcal{L}_\mathcal{E}(\theta))$. 
Still yet, we shall also report that there has been a successful application of IGA on medical dataset \begin{footnote}{For anonymity, we refrain from directly citing this work. We however attach a pdf of this publication in the supplemental material with their citation of our preprint blacked out (Medical.pdf). }\end{footnote}. 
We discuss other technical limitation of IGA in Appendix \ref{appsubsec:method_lim}.

\textbf{Social Impact \ }Our study is an effort toward learning a model that can perform well in a newly encountered environment.   Further study in this field might allow safer/ more economical training of the model.
For example, our study might allow the user to train a good model without collecting a dataset from a dangerous/risky environment.
Further study of the OOD problem might also be helpful in promoting the fairness of the prediction \cite{mehrabi2019survey}.
However, one must be wary of the treatment of the environmental factor.
In this study, we consider the situations in which the identity of the environmental factor is unknown.
This is actually the case in many applications;  the effect of the hidden “environmental factor” might be inferrable only from a set of datasets sampled from a collection of environments.
In such cases, a user with an ulterior motive might be able to fake the true identity of the environment by using a particular collection of datasets.
For example, when the mission is to train a model that can perform well on people of all ages, a user with ulterior motive might collect a dataset of one age group from a particular socio-political group and a dataset of another age group from a yet-another socio-political group.
Such a user might advertise his/her predictor as an \textit{age} agnostic predictor, when in truth the “advertised” environmental factor (age) does not agree with the true factor that distinguishes one dataset from another in the specific collection used in the training process.
This problem applies to many statistical methods, and one must pay close attention to the data collection process in order to ensure fair analysis.

\bibliography{main.bib}

\begin{thebibliography}{10}

\bibitem{achille2017emergence}
A.~Achille and S.~Soatto.
\newblock Emergence of invariance and disentanglement in deep representations.
\newblock {\em arXiv preprint arXiv:1706.01350}, 2017.

\bibitem{ahuja2020invariant}
K.~Ahuja, K.~Shanmugam, K.~Varshney, and A.~Dhurandhar.
\newblock Invariant risk minimization games.
\newblock {\em International Conference on Machine Learning}, 2020.

\bibitem{DVIB}
A.~A. Alemi, I.~Fischer, J.~V. Dillon, and K.~Murphy.
\newblock Deep variational information bottleneck.
\newblock In {\em International Conference on Learning Representations}, 2017.

\bibitem{arjovsky2019invariant}
M.~Arjovsky, L.~Bottou, I.~Gulrajani, and D.~Lopez-Paz.
\newblock Invariant risk minimization.
\newblock {\em arXiv preprint arXiv:1907.02893}, 2019.

\bibitem{aubinlinear}
B.~Aubin, M.~Arjovsky, L.~Bottou, and D.~Lopez-Paz.
\newblock Linear unit tests for invariance discovery.
\newblock 2020.

\bibitem{Banerjee03onthe}
A.~Banerjee, X.~Guo, and H.~Wang.
\newblock On the optimality of conditional expectation as a bregman predictor.

\bibitem{alexis_confounding}
A.~Bellot and M.~van~der Schaar.
\newblock Accounting for unobserved confounding in domain generalization.
\newblock {\em arXiv preprint arXiv:2007.10653}, 2021.

\bibitem{buhlmann2018invariance}
P.~B{\"u}hlmann.
\newblock Invariance, causality and robustness.
\newblock {\em arXiv preprint arXiv:1812.08233}, 2018.

\bibitem{chang2020invariant}
S.~Chang, Y.~Zhang, M.~Yu, and T.~S. Jaakkola.
\newblock Invariant rationalization.
\newblock {\em arXiv preprint arXiv:2003.09772}, 2020.

\bibitem{Learning2Explain}
J.~Chen, L.~Song, M.~Wainwright, and M.~Jordan.
\newblock Learning to explain: An information-theoretic perspective on model
  interpretation.
\newblock In {\em Proceedings of Machine Learning Research}, volume~80, pages
  883--892, 2018.

\bibitem{clevert2015fast}
D.-A. Clevert, T.~Unterthiner, and S.~Hochreiter.
\newblock Fast and accurate deep network learning by exponential linear units
  (elus).
\newblock {\em arXiv preprint arXiv:1511.07289}, 2015.

\bibitem{damour2020underspecification}
A.~D'Amour, K.~Heller, D.~Moldovan, B.~Adlam, B.~Alipanahi, A.~Beutel, C.~Chen,
  J.~Deaton, J.~Eisenstein, M.~D. Hoffman, F.~Hormozdiari, N.~Houlsby, S.~Hou,
  G.~Jerfel, A.~Karthikesalingam, M.~Lucic, Y.~Ma, C.~McLean, D.~Mincu,
  A.~Mitani, A.~Montanari, Z.~Nado, V.~Natarajan, C.~Nielson, T.~F. Osborne,
  R.~Raman, K.~Ramasamy, R.~Sayres, J.~Schrouff, M.~Seneviratne, S.~Sequeira,
  H.~Suresh, V.~Veitch, M.~Vladymyrov, X.~Wang, K.~Webster, S.~Yadlowsky,
  T.~Yun, X.~Zhai, and D.~Sculley.
\newblock Underspecification presents challenges for credibility in modern
  machine learning.
\newblock {\em arXiv preprint arXiv:2011.03395}, 2020.

\bibitem{Darmois}
G.~Darmois.
\newblock Analyse des liaisons de probabilit´e.
\newblock {\em In Proceedings of Intern. Statistics Conferences}, IIIA, 1951.

\bibitem{Durrett}
R.~Durrett.
\newblock {\em Probability: Theory and Examples}.
\newblock Thomson, 2019.

\bibitem{ossan2019}
H.~Ferenc.
\newblock Invariant risk minimization: An information theoretic view.
\newblock 2019.

\bibitem{finn17a}
C.~Finn, P.~Abbeel, and S.~Levine.
\newblock Model-agnostic meta-learning for fast adaptation of deep networks.
\newblock In D.~Precup and Y.~W. Teh, editors, {\em Proceedings of the 34th
  International Conference on Machine Learning}, volume~70 of {\em Proceedings
  of Machine Learning Research}, pages 1126--1135, International Convention
  Centre, Sydney, Australia, 06--11 Aug 2017. PMLR.

\bibitem{folland2013real}
G.~B. Folland.
\newblock {\em Real analysis: modern techniques and their applications}.
\newblock John Wiley \& Sons, 2013.

\bibitem{gamalNetworkinfo}
A.~E. Gamal and Y.H.Kim.
\newblock {\em Network Information Theory}.
\newblock Cambridge University Press, 2011.

\bibitem{Vselection}
S.~Gao, G.~V. Steeg, and A.~Galstyan.
\newblock Variational information maximization for feature selection.
\newblock {\em NeurlIPS}, 2016.

\bibitem{geirhos2020shortcut}
R.~Geirhos, J.-H. Jacobsen, C.~Michaelis, R.~Zemel, W.~Brendel, M.~Bethge, and
  F.~A. Wichmann.
\newblock Shortcut learning in deep neural networks.
\newblock {\em arXiv preprint arXiv:2004.07780}, 2020.

\bibitem{ilyas2019adversarial}
A.~Ilyas, S.~Santurkar, D.~Tsipras, L.~Engstrom, B.~Tran, and A.~Madry.
\newblock Adversarial examples are not bugs, they are features.
\newblock In H.~Wallach, H.~Larochelle, A.~Beygelzimer, F.~d~Alch\'{e}-Buc,
  E.~Fox, and R.~Garnett, editors, {\em Advances in Neural Information
  Processing Systems 32}, pages 125--136. Curran Associates, Inc., 2019.

\bibitem{ioffe2015batch}
S.~Ioffe and C.~Szegedy.
\newblock Batch normalization: Accelerating deep network training by reducing
  internal covariate shift.
\newblock {\em arXiv preprint arXiv:1502.03167}, 2015.

\bibitem{kamath2021does}
P.~Kamath, A.~Tangella, D.~J. Sutherland, and N.~Srebro.
\newblock Does invariant risk minimization capture invariance?
\newblock {\em arXiv preprint arXiv:2101.01134}, 2021.

\bibitem{kingma2014adam}
D.~P. Kingma and J.~Ba.
\newblock Adam: A method for stochastic optimization.
\newblock {\em arXiv preprint arXiv:1412.6980}, 2014.

\bibitem{li2018advdomain}
H.~Li, S.~Jialin~Pan, S.~Wang, and A.~C. Kot.
\newblock Domain generalization with adversarial feature learning.
\newblock In {\em Proceedings of the IEEE Conference on Computer Vision and
  Pattern Recognition}, pages 5400--5409, 2018.

\bibitem{linsker1988self}
R.~Linsker.
\newblock Self-organization in a perceptual network.
\newblock {\em Computer}, 21(3):105--117, 1988.

\bibitem{mehrabi2019survey}
N.~Mehrabi, F.~Morstatter, N.~Saxena, K.~Lerman, and A.~Galstyan.
\newblock A survey on bias and fairness in machine learning.
\newblock {\em arXiv preprint arXiv:1908.09635}, 2019.

\bibitem{Constrained_MI}
T.~Nguyen and T.~Nguyen.
\newblock Entropy-constrained maximizing mutual information quantization.
\newblock {\em arXiv preprint arXiv:2001.01830}, 2021.

\bibitem{parascandolo2020learning}
G.~Parascandolo, A.~Neitz, A.~Orvieto, L.~Gresele, and B.~Sch{\"o}lkopf.
\newblock Learning explanations that are hard to vary.
\newblock {\em arXiv preprint arXiv:2009.00329}, 2020.

\bibitem{paszke2019pytorch}
A.~Paszke, S.~Gross, F.~Massa, A.~Lerer, J.~Bradbury, G.~Chanan, T.~Killeen,
  Z.~Lin, N.~Gimelshein, L.~Antiga, et~al.
\newblock Pytorch: An imperative style, high-performance deep learning library.
\newblock {\em arXiv preprint arXiv:1912.01703}, 2019.

\bibitem{Peters2016jrssb}
J.~Peters, P.~B{\"u}hlmann, and N.~Meinshausen.
\newblock Causal inference using invariant prediction: identification and
  confidence intervals.
\newblock {\em Journal of the Royal Statistical Society, Series B (with
  discussion)}, 78(5):947--1012, 2016.

\bibitem{peters2017elements}
J.~Peters, D.~Janzing, and B.~Sch{\"o}lkopf.
\newblock {\em Elements of causal inference: foundations and learning
  algorithms}.
\newblock The MIT Press, 2017.

\bibitem{peters2012identifiability}
J.~Peters, J.~Mooij, D.~Janzing, and B.~Sch{\"o}lkopf.
\newblock Identifiability of causal graphs using functional models.
\newblock {\em arXiv preprint arXiv:1202.3757}, 2012.

\bibitem{rojas2018invariant}
M.~Rojas-Carulla, B.~Sch{\"o}lkopf, R.~Turner, and J.~Peters.
\newblock Invariant models for causal transfer learning.
\newblock {\em The Journal of Machine Learning Research}, 19(1):1309--1342,
  2018.

\bibitem{rosenfeld2021the}
E.~Rosenfeld, P.~K. Ravikumar, and A.~Risteski.
\newblock The risks of invariant risk minimization.
\newblock In {\em International Conference on Learning Representations}, 2021.

\bibitem{shen2018causally}
Z.~Shen, P.~Cui, K.~Kuang, B.~Li, and P.~Chen.
\newblock Causally regularized learning with agnostic data selection bias.
\newblock In {\em Proceedings of the 26th ACM international conference on
  Multimedia}, pages 411--419, 2018.

\bibitem{storkey2009training}
A.~Storkey.
\newblock When training and test sets are different: characterizing learning
  transfer.
\newblock {\em Dataset shift in machine learning}, pages 3--28, 2009.

\bibitem{prevent19a}
A.~Subbaswamy, P.~Schulam, and S.~Saria.
\newblock Preventing failures due to dataset shift: Learning predictive models
  that transport.
\newblock In K.~Chaudhuri and M.~Sugiyama, editors, {\em Proceedings of Machine
  Learning Research}, volume~89 of {\em Proceedings of Machine Learning
  Research}, pages 3118--3127, 16--18 Apr 2019.

\bibitem{IB_ganso}
N.~Tishb, F.~Pereira, and W.~Biale.
\newblock The information bottleneck method.
\newblock In {\em Annual Allerton Conf. on Communication, Control, and
  Computing}, volume~37, pages 368--377, 1999.

\bibitem{tokui2019chainer}
S.~Tokui, R.~Okuta, T.~Akiba, Y.~Niitani, T.~Ogawa, S.~Saito, S.~Suzuki,
  K.~Uenishi, B.~Vogel, and H.~Yamazaki~Vincent.
\newblock Chainer: A deep learning framework for accelerating the research
  cycle.
\newblock In {\em Proceedings of the 25th ACM SIGKDD International Conference
  on Knowledge Discovery \& Data Mining}, pages 2002--2011, 2019.

\end{thebibliography}
\bibliographystyle{abbrv}

\newpage
\appendix
\section{Appendix}

This appendix section is structured as follows.
In Appendix \ref{appsec:proof}, we provide the proofs for the formal versions of our theoretical results. 
In Appendix \ref{appsec:iga_linear}, we show the details of our computation in Section \ref{subsec:IGA_linear}.
In Appendix \ref{appsec:derivation}, we present the details of our derivation of IGA (Section \ref{subsec:IGA}). 
In Appendix \ref{appsec:implement}, we present the details of our experiment on Invariance Unit Tests \cite{aubinlinear} and MNIST derived datasets.
We also provide additional results in Appendix \ref{appsec:add_result} as well.

\section{Formal versions of Theorem \ref{thm:controllability_informal}, Theorem \ref{thm:MIP_objective_informal}, and Corollary \ref{thm:vs_Rojas} }
\label{appsec:proof}

\subsection{Important Remarks and notations}
\label{appsubsec:remark}
In this section, we provide the proofs for the formal versions of our theoretical statements (Theorem \ref{thm:controllability_informal}, Theorem \ref{thm:MIP_objective_informal}, Corollary \ref{thm:vs_Rojas}).  
The formal versions of our statements are based on measure theoretic probability, and hence are our proofs.
Throughout, we will use the notations in the standard probability texts like \cite{Durrett}.
We use upper case letters to represent random variables, and lower case letters to represent the realizations of the random variables. 
For example, $X$ would represent a random variable, and $x$ would be its realization. 
We also treat the environment $\mathcal{E}$ as a random variable, and use $\epsilon$ to represent a realization of $\mathcal{E}$.
We use $Z_1 \perp Z_2$ as a shorthand notation for "$Z_1$ is independent from $Z_2$." 

This section is structured as follows. 
First, in Appendix \ref{appsubsec:proof_setup} we will present the setup of our analysis along with the basic definitions that will be used throughout. 
Next, in Appendix \ref{appsubsec:control} we will present our proof of the formal version of our result about the controllability condition ( theorem \ref{thm:controllability_informal}),
In Appendix \ref{appsubsec:MIP_proof}, we prove the formal version of theorem \ref{thm:MIP_objective_informal}.
In Appendix \ref{appsubsec:vs_Rojas}, we will prove the formal version of corollary \ref{thm:vs_Rojas} that claims that our controllability condition \ref{thm:controllability_informal} generalizes the condition used in \cite{rojas2018invariant}. 
Finally, in Appendix \ref{appsec:other_proof}, we present our proofs for two general lemmas used in Appendix \ref{appsubsec:control} and Appendix \ref{appsubsec:MIP_proof} .


\subsection{Setup} 
\label{appsubsec:proof_setup}
Let $Y$ be the output random variable, $X$ be the input random variable and $\mathcal{E}$ be the environmental random variable.
We suppose that $Y, X, \mathcal{E}$ are all euclidean-space-valued random variables measurable with respect to the probability triple $(\Omega, \mathcal{F}, P)$. 
We follow the notation of the standard text like \cite{Durrett} and use $\sigma(X)$ to represent the sigma algebra of $X$.
Also following the probability convention, we say  $Z \in \sigma(X)$ whenever a random variable $Z$ is measurable with respect to $\sigma(X)$. 
For simplicity, we do not bother writing a feature of $X$ as $h(X)$ with some function $h$. 
Instead, we use the conventional notation $\Phi \in \sigma(X)$ to represent a feature of $X$, because the measurability of $\Phi$ with respect to $\sigma(X)$ is equivalent to the statement that there exists some measurable function $h$ such that $h(X) = \Phi$.

That being said, if $\Phi$ is a random variable that is measurable with respect to $\sigma(X)$, 
let us define $\mathcal{E}_\phi \in \sigma(\mathcal{E})$ to be a minimal random variable (in the sense of sigma algebra) such that $\Phi \perp \mathcal{E} | \mathcal{E}_\phi$.  
The variable $\mathcal{E}_\phi$ thus satisfies $P(\Phi | \mathcal{E}) = P(\Phi | \mathcal{E}_\phi)$. 
Also, if the conditional distribution of $\mathcal{E} | \mathcal{E}_\phi$ is smooth enough, the functional representation lemma \ref{thm:fxn_repn} states that there exists some $\mathcal{E}_\psi$ such that $\mathcal{E}_\psi \perp \mathcal{E}_\phi$ and $\sigma(\mathcal{E}) = \sigma  (\mathcal{E}_\psi, \mathcal{E}_\phi)$.
Because this implies the existence of an invertible map between $\mathcal{E}$ and  $(\mathcal{E}_\psi, \mathcal{E}_\phi)$, 
WLOG we write  $\epsilon = (\epsilon_\phi, \epsilon_\psi)$ for every $\epsilon \in \textrm{supp}(\mathcal{E})$.
Finally and most importantly, let us define the set $\mathcal{I}$ of invariant features to be
$$\mathcal{I} = \{\Phi \in \sigma(X) ; p(Y| \Phi, \mathcal{\epsilon}) =  p(Y | \Phi) ~\forall \epsilon \in \textrm{supp}(\mathcal{E})  \}$$ 
and that this set is non-empty.

\subsection{The proof of the formal version of theorem \ref{thm:controllability_informal}} \label{appsubsec:control}

We begin with the formal statement of our controllability condition \ref{def:controllability}.
\begin{defn}[Controllability Condition]
We say that a feature $\Phi$ satisfies a controllability condition if for all $\epsilon = (\epsilon_\phi, \epsilon_\psi) \in \textrm{supp}(\mathcal{E})$, there exists $\tilde \epsilon_\psi \in \textrm{supp}(\mathcal{E}_\psi)$ such that 
$$ Y \perp X  | \Phi,  \epsilon_\phi, \tilde \epsilon_\psi$$
\label{def:controllability}
\end{defn}
We emphasize that this is a condition about the feature $\Phi \in \mathcal{I}$.
To prove theorem \ref{thm:controllability_informal},
We will first prove the following small lemma about the property of $(Y, \Phi)$. 
\begin{lemma}
If $\Phi \in \mathcal{I}$, then $\mathcal{E}_\psi \perp (\Phi, Y) | \mathcal{E}_\phi$ \label{lemma:e_psi_ignore}
\end{lemma}
\begin{proof}
\begin{align}
    P(Y, \Phi | \epsilon) &=  P(Y, \Phi | \epsilon_\phi, \epsilon_\psi)  \\
    &= P(Y |  \Phi,  \epsilon_\phi, \epsilon_\psi) P ( \Phi |  \epsilon_\phi, \epsilon_\psi) \\
    &= P(Y | \Phi) P(\Phi | \epsilon_\phi)\\
    &= P(Y | \Phi, \epsilon_\phi ) P(\Phi | \epsilon_\phi)  \\
    &= P(Y , \Phi| \epsilon_\phi) 
\end{align} 
The second equality follows from the definition of conditional probability.
The third and fourth equality follows from the fact that $g(\mathcal{E}) \perp Y  | \Phi$ for any measurable $g$. The last equality follows from the definition of $\mathcal{E}_\phi$.
\end{proof}
This result implies that $(Y, \Phi) | \epsilon_\phi, \epsilon_\psi  =  (Y, \Phi) | \epsilon_\phi, \epsilon_{\psi'}$ for any arbitrary pair $\epsilon_\psi$ and $\epsilon_{\psi'}$.  
In other words, given $\Phi$, the distribution of $Y$ changes only with respect to $\epsilon_\phi$.  We can use this result to prove the following most important  result of our work. 

\begin{thm}
Let $g$ be a strictly convex, differentiable function and $D$ be the corresponding  Bregman Loss function that is convex in both input variables. Also, let $\Phi \in \sigma(\mathcal{I})$ and write $\mathbb{E}[Y |\Phi] = f^*(X)$.  
If $\Phi$ satisfies the controllability condition \ref{def:controllability}, then
$$f^* = \argmin_f \sup_{\epsilon \in \textrm{supp}(\mathcal{E})} \mathbb{E} [D(Y, f(X)| \epsilon)]$$  
\label{thm:controllability}
\end{thm}
\begin{proof}
We are going to leverage the fact that, if $\mathcal{G}$ is a sub sigma algebra of $\mathcal{F}$ to which $Y$ is measurable, then \cite{Banerjee03onthe}
\begin{align}
    \argmin_{Z \in \mathcal{G}} \mathbb{E}[D(Y, Z)] = \mathbb{E}[Y| \mathcal{G}].  
\end{align}
The situation considered here includes the case in which $D$ is a Kullback-Leibler Divergence loss or equivalently the case in which the loss $E[D(Y, Z)]$ is a cross entropy loss $E[- Y^T \log Z]$ that is convex in both $Y$ and $Z$. The case of $L_2$ loss is trivially included as well, because the $L_2$ metric is symmetric. See Appendix \ref{appsec:other_proof} and \cite{Banerjee03onthe} for more detail. 

To show the claim of this theorem, it suffices to show that, for any $\epsilon \in \textrm{supp}(\mathcal{E})$ and any measurable $f$, there exists $\epsilon' \in \textrm{supp}(\mathcal{E})$ such that 
$\mathbb{E}[D(Y, f(X))|\epsilon'] > \mathbb{E}[D(Y, f^*(X))|\epsilon]$
so that 
$\sup_\epsilon \mathbb{E}[D(Y, f(X))|\epsilon] \geq \sup_\epsilon \mathbb{E}[D(Y, f^*(X))|\epsilon]$ for all measurable $f$. 
Now, if $\epsilon = (\epsilon_\phi, \epsilon_\psi)$,  let us choose $\tilde \epsilon_\psi$ to be such that
$X \perp Y | \Phi, \epsilon_\phi, \tilde \epsilon_\psi$ and write 
$\epsilon' = (\epsilon_\phi, \tilde \epsilon_\psi)$
Then it follows that 
\begin{align}
    \mathbb{E}_{X,Y} [D(Y, f(X)) | \epsilon'] &=  \mathbb{E}_{\Phi,Y} [ \mathbb{E}_X[ D(Y, f(X)) | \Phi, Y, \epsilon']  \epsilon'] \\  
    &\geq \mathbb{E}_{\Phi,Y} [  D(Y, \mathbb{E}_X[f(X)|\Phi, \epsilon', Y]) | \epsilon'] \\ 
    &=\mathbb{E}_{\Phi,Y} [  D(Y, \mathbb{E}_X[f(X)|\Phi, \epsilon']) | \epsilon'] \\ 
    &:=\mathbb{E}_{\Phi,Y} [  D(Y, h(\Phi, \epsilon')) | \epsilon'] 
\end{align}
where the first equality follows from the tower rule, the second inequality follows from Jensen's inequality and the fact that our $D(y,x)$ is convex with respect to $x$ , and the third equality follows from our choice of $\epsilon'$. 
In the fourth equality, we defined $h(\Phi, \epsilon') := \mathbb{E}[f(X)|\Phi, \epsilon']$. With $\epsilon' = (\epsilon_\phi, \tilde \epsilon_\psi)$ fixed to be constant, $D(h(\Phi, \epsilon') , Y) $ is a random variable measurable with respect to $(\Phi, Y)$. 
Therefore, 
\begin{align}
   \mathbb{E}_{\Phi,Y}[ D(Y, h(\Phi, \epsilon')) | \epsilon'] &=  \mathbb{E}_{\Phi,Y}[ D(Y, h(\Phi, \epsilon_\phi, \tilde \epsilon_\psi)) | \epsilon_\phi, \epsilon_\psi] \\
    &\geq \mathbb{E}_{\Phi,Y}[ D(Y, \mathbb{E}_Y[Y | \Phi, \epsilon_\phi, \epsilon_\psi]) | \epsilon_\phi, \epsilon_\psi] \\
    &= \mathbb{E}_{\Phi,Y}[ D(Y, \mathbb{E}_Y[Y | \Phi]) | \epsilon] \\
    &:= \mathbb{E}_{\Phi,Y}[D(Y, f^*(X)) | \epsilon] 
\end{align}
where the first equality follows from the lemma \ref{lemma:e_psi_ignore}, the second equality follows from the minimality of conditional expectation, and the third equality follows from the fact that $Y \perp \mathcal{E} | \Phi$. 
All together we have proven that, for any $\epsilon$ there is $\epsilon'$ with
$$ \mathbb{E}_{\Phi,Y}[ D(Y, f(X)) | \epsilon']  \geq  \mathbb{E}_{\Phi,Y}[ D(Y, f^*(X)) | \epsilon]$$
as desired.
\end{proof}

\subsection{The proof of the formal version of theorem \ref{thm:MIP_objective_informal}} \label{appsubsec:MIP_proof}

To prove theorem \ref{thm:MIP_objective_informal}, we first need several small lemmas.
The first lemma states that any $\Phi$ that satisfies the controllability condition \ref{def:controllability} achieves the maximal mutual information with $Y$ in at least one environment.
\begin{lemma}
Suppose $\Phi \in \mathcal{I}$ and suppose that $\epsilon$ is such that $X \perp Y | \Phi, \epsilon$. Then for this particular $\epsilon$, 
$$\Phi = \arg\max_{Z \in \sigma(X)} I(Y ;  Z | \epsilon)$$
\end{lemma} 
\begin{proof}
Suppose that there exists $B \in \sigma(X)$ with $I(Y; B, \Phi | \epsilon) > I(Y; \Phi | \epsilon)$. Then by the property of 
conditional mutual information,
\begin{align}
    I(Y; B, \Phi | \epsilon) = I(Y; \Phi| \epsilon) +   I(Y; B|\Phi,  \epsilon)
\end{align}
This equation implies that $ I(Y; B|\Phi, \epsilon) > 0$.  However, $B \perp Y | \Phi, \epsilon$ by the choice of $\epsilon$ and the property of $\Phi$, so this must be zero, and it is a contradiction.  
\end{proof}

The next lemma states that any $\Phi$ that satisfies the controllability condition \ref{def:controllability} is maximal in $\mathcal{I}$.
\begin{lemma}
Suppose $\Phi \in \mathcal{I}$ and suppose that for some $\epsilon^*$. $Y \perp X | \Phi, \epsilon^*$.  Then $\Phi$ is maximal in the sense that, there is no $\tilde \Phi \in \mathcal{I}$ with $\Phi \in \sigma(\tilde \Phi)$. 
\label{thm:maximal}
\end{lemma}
\begin{proof}
Let $\tilde \Phi \in \mathcal{I}$ as stated in the assumption. Then 
\begin{align}
    P(Y | \tilde \Phi) &=   P(Y | \tilde \Phi, \epsilon^*) \\
    &= P(Y | \tilde \Phi, \Phi, \epsilon^*)  \\
    &= P(Y | \Phi , \epsilon^*) \\
    &= P(Y | \Phi) 
\end{align}
The first line follows from the property of $\mathcal{I}$; 
$Y \perp \mathcal{E} | \Phi.$  The second line follows from the fact that $r(\tilde \Phi) = \Phi$ for some $r$. The third line follows from $X \perp Y | \Phi, \epsilon^*$ for this specific $\epsilon^*$. Lastly,
the final line follows again from the property of $\mathcal{I}$. 
\end{proof}

The following is the formal version of theorem \ref{thm:MIP_objective_informal}. 
\begin{thm}
\label{thm:MIP_formal}
Suppose that there exists at least one $\Phi$ for which there is a corresponding $\tilde \epsilon_\psi$ for every $\epsilon_\phi$ such that $X \perp Y | \Phi, \epsilon_\phi, \tilde \epsilon_\psi$. If $\mathcal{I}$ is generated by one $\Phi_0$, then  $\mathbb{E}[Y | \Phi^*]$ is OOD optimal if 
\begin{align}
    \Phi^*  = \argmax_{\Phi \in \mathcal{I}} I(Y ; \Phi) 
\end{align}
\end{thm} 
\begin{proof}
If for all $\epsilon_\phi$ there exists $\tilde  \epsilon_\psi$ with $X \perp Y | \Phi, \epsilon_\phi, \tilde \epsilon_\psi$, then $\mathbb{E}[Y|\Phi]$ is optimal by the proposition \ref{thm:controllability}.  Now, if $\mathcal{I}$ is generated by one $\Phi_0$, $\Phi \in \sigma(\Phi_0)$ so that by the proposition \ref{thm:maximal}, $\sigma(\Phi) = \sigma(\Phi_0)$ and  $I(Y ; \Phi_0) = I(Y ; \Phi)$ necessarily. 
Because $I(Y ; \Phi_0)$ is maximal as well, the claim follows.
\end{proof}

\subsection{The proof of the formal version of theorem \ref{thm:vs_Rojas}}
\label{appsubsec:vs_Rojas}

In \cite{rojas2018invariant},  the authors use $x_A$ to denote 
a set $\{x_i; i \in A\}$
and assume that there is a 
specific subset of coordinates $S \subset \{1,2, ...p\}$ such that 
$p(y | x_S, \epsilon) = p(y|x_S)$.
In showing that their predictor $\mathbb{E}[Y|X_S]$ achieves OOD optimality,
they also assume the following condition in their proof:

\begin{assumption}
Let $N$ be the complement of $S$. Then, 
for any environment $\epsilon$,  there exists $\epsilon'$ such that 
\begin{align}
    p(x, y | \epsilon')  = p(y, x_S| \epsilon) p(x_N | \epsilon) 
\end{align}
\label{def:carulla_assumption}
\end{assumption}

\begin{thm}
The assumption \ref{def:carulla_assumption} in \cite{rojas2018invariant} is strictly stronger than the controllability condition \ref{def:controllability}.
\label{thm:Rojas_formal} 
\end{thm}

\begin{proof}
Suppose that there is such $\epsilon'$, and let $\epsilon = (\epsilon_\phi, \epsilon_\psi)$ be the decomposition derived from the Function Representation lemma with $\Phi = X_S$.  We will show that this necessitates \ref{def:controllability} for $X_S$. 

First, notice
\begin{align}
\begin{split}
    p(y, x_S| \epsilon) p(x_N | \epsilon) &= p(x, y | \epsilon') \\&= p(x_S, x_N, y | \epsilon') \\
    &= p(x_S,y | x_N, \epsilon') p( x_N | \epsilon')   
\end{split} \label{eq:carulla1} 
\end{align}
Integrating both sides with respect to $y$ and $x_S$ we obtain 
(i) $p(x_N | \epsilon) = p( x_N | \epsilon')$, and this allows us to say 
(ii) $p(y, x_S| \epsilon) =  p(y, x_S | x_N, \epsilon')$
as well.  
Because the left hand side of (2) has no dependence on $x_N$, 
we also obtain
\begin{align}
    p(y, x_S| \epsilon) &=  p(y, x_S | x_N, \epsilon') \\
    &= p(y, x_S | \epsilon') 
\end{align}
Now, we would use these relations to describe a relation between $\epsilon$ and $\epsilon'$. Continuing from the equality above,
\begin{align}
 p(y, x_S| \epsilon)  &= p(y, x_S | \epsilon')\\
 &= p(y | x_S, \epsilon') p(x_S| \epsilon')\\
        p(y | x_S, \epsilon) p(x_S| \epsilon) &=   p(y | x_S) p(x_S| \epsilon') \\
    p(y | x_S) p(x_S| \epsilon) &=   p(y | x_S) p(x_S| \epsilon') \\
    p(x_S| \epsilon) &= p(x_S| \epsilon')  
\end{align}
where we applied invariance property of $X_S$ in the third line.
By the definition of $\epsilon_\phi$, this allows us to say that 
$\epsilon$ and $\epsilon'$ agrees on $\epsilon_\phi$ (up to equivalence class). In other words, we are justified to write 
$\epsilon' = (\epsilon_\phi, \epsilon_\psi')$. 
Moreover, writing the same equality in a different way, 
\begin{align}
   p(y, x_S| \epsilon) &=  p(y, x_S | x_N, \epsilon') \\
   p(y| x_S, \epsilon) p(x_S| \epsilon) &=   p(y | x_S,  x_N, \epsilon') p(x_S| x_N, \epsilon') 
\end{align}
Again with the same trick of integrating both sides with respect to $y$, we obtain 
\begin{align}
   p(x_S| \epsilon) &= p(x_S| x_N, \epsilon') \textrm{~~and}  \label{eq:carulla_extra2} \\ 
      p(y| x_S, \epsilon) &=  p(y | x_S,  x_N, \epsilon')   \label{eq:pre_controllability}
\end{align}
Writing the second equality (eq \eqref{eq:pre_controllability}) with the decomposition, 
\begin{align}
    p(y| x_S, \epsilon_\phi, \epsilon_\psi) &=  p(y | x_S,  x_N, \epsilon_\phi, \epsilon_\psi')  \label{eq:xn_indpt}
\end{align}
Because the LHS of eq \eqref{eq:xn_indpt} does not depend on $x_N$, the RHS does not depend on $x_N$ as well.
Thus we can drop the $X_N$ from the RHS and this allows us to say
\begin{align}
    p(y | x_S,  x_N, \epsilon_\phi, \epsilon_\psi')  = p(y | x_S, \epsilon_\phi, \epsilon_\psi')
\end{align}
But this is the very controllability condition \eqref{def:controllability} for $x_S$.
At the same time, the controllability condition alone does not guarantee $p(x_N | \epsilon') = p(x_N | \epsilon)$ (condition (i) required from \eqref{eq:carulla1}) nor $p(x_S| \epsilon) = p(x_S| x_N, \epsilon')$ in eq \eqref{eq:carulla_extra2}. 
\end{proof}

\subsection{Other Lemmas used in the proofs}
\label{appsec:other_proof}
\subsubsection{The optimality of $P(Y|X)$ for Kullback-Leibler divergence} 
\label{appsubsec:cond_exp_opt}
\begin{lemma}
Suppose $Y$ is a categorical random variable expressed as a one-hot vector and 
\begin{align}
    \mathbb{E}[ - Y^T \log Z]   
\end{align}
is the cross entropy loss for $Z \in \sigma(X)$. Then 
\begin{align}
    \argmin_{Z \in \sigma(X)} \mathbb{E}[ - Y^T \log Z] := P(Y | X) 
\end{align}
\label{thm:kl_optimality} 
\end{lemma}
\begin{proof}
First note that, because $Y$ is a one hot vector, only one term of $Y^T \log Z =  \sum_{i} Y_i \log Z_i$ is non zero, and the Monte Carlo estimate of  
$\mathbb{E}[-Y^T \log Z]$ is the very empirical evaluation of cross entropy that is ubiquitously used in ML literatures. 
Also, for each categorical label $i$, $\mathbb{E}[Y = 1_{i} | X]$ is by definition a one-hot vector whose $i$th coordinate is $P(Y= 1_{y_i} | X)$.  
Extending this fact, $\mathbb{E}[Y | X] = \mathbb{E}[Y * \sum_i 1_{i} | X] $ it self is a vector $P(Y |X)$ whose $i$th coordinate is $P(Y= 1_{i} | X)$. 
For more detail of this fact, consult the standard text like \cite{Durrett}.

As we will describe momentarily, the cross entropy is a loss derived from Kullback-Leibler divergence, or 
\begin{align}
    D_\phi(x, x') = [\phi(x) - \phi(x')] - \langle x- x' ,  \nabla \phi(x') \rangle 
\end{align}
where $\phi$ is a convex function $x^T \log x$.  
We will first show that $\argmin_{Z \in \sigma(X)}  \mathbb{E}[D_\phi(Y, Z)] = P(Y |X)$.
Because $D_\phi(x, x')$ is literally the divergence of $[\phi(x) - \phi(x')]$ from its Taylor expansion centered about $\phi(x')$, this is positive for all $x \neq x'$ by the convexity of  $x^T \log x$. 

WLOG let us extend each $\phi_i(x) = x_i \log x_i$ to $[0,1]$ with the limit $ \lim_{x_i \to 0} x_i log x_i = 0$. 
Then $Y^T \log Y = 0$ almost surely and, for all $Z \in \sigma(X)$, 
\begin{align} 
\begin{split}
   \mathbb{E}[Y^T \nabla (Z^T\log Z)  ]  &= \mathbb{E}[ \mathbb{E}[Y^T \nabla (Z^T \log Z) | X]] \\
   &= \mathbb{E}[ \mathbb{E}[Y| X]^T \nabla (Z^T \log Z)] \\
   &= \mathbb{E}[ P(Y| X)^T \nabla (Z^T \log Z)]  
\end{split} \label{eq:trick}
\end{align}
because $\nabla(Z \log Z) =  1 + \log Z \in \sigma(X)$.
Let us write $P(Y | X) = Z^*$.  We will show  $D(Y, Z) - D(Y, Z^*) > 0$ when $Z \neq Z^*$. 
The first term $Y \log Y = 0$ cancels out on both terms, and we get
\begin{align}
\begin{split}
    D(Y, Z) - D(Y, Z^*) &= \mathbb{E}[ [ Z^T \log Z  -  (Y - Z)^T \nabla (Z \log Z) ] \\
    & \ -  [ Z^{*T} \log Z^* -  (Y - Z^*)^T \nabla (Z^*\log Z^*) ] ]  \\
    &= \mathbb{E}[ Z^T \log Z - Z^*\log Z^* - (Z^* - Z)^T \log Z] \\
    &= \mathbb{E}[D(Z, Z^*)]
\end{split}
\end{align}
where we used the fact that 
\begin{align}
    \mathbb{E}[Y^T \nabla (Z \log Z)] &= \mathbb{E}[\mathbb{E}[Y^T \nabla (Z \log Z) | X]] \\
&=  \mathbb{E}[\mathbb{E}[E [Y|X]^T \nabla (Z \log Z)]  \\
&=  \mathbb{E}[Z^* \nabla (Z \log Z)]
\end{align}
so 
$\mathbb{E}[(Y - Z)^T \nabla (Z \log Z)] = \mathbb{E}[(Z^* - Z)^T \nabla (Z\log Z)]$. Likewise, $\mathbb{E}[(Y - Z^*)^T \nabla (Z^*\log Z^*) ] = 0$ with the same logic since $Z^* \in \sigma(X)$ by its definition.

Finally, by substitution we have 
\begin{align}
D(Y, Z) &= Y^T \log Y - Z^T \log Z -  (Y - Z)^T  (\nabla (Z^T \log Z)) \\
&= - Z^T \log Z - (Y - Z)^T (1 + \log Z) \\
&= - Z^T \log Z - (1 - Z^T 1) - Y^T\log Z + Z^T \log Z \\
&= Y^T (\log Y - \log Z) -  (1 - Z^T 1) \\
&= - Y^T \log Z -  (1 - Z^T 1)
\end{align}
Because $\argmin_{Z \in \sigma(X)} \mathbb{E}[- Y^T \log Z -  (1 - Z^T 1)] = P(Y|X)$ and $P(Y|X)^T 1 = 1$,
\begin{align}
& \argmin_{Z \in \sigma(X), Z^T1 = 1} \mathbb{E}[- Y^T \log Z -  (1 - Z^T 1)] \\
&= \argmin_{Z \in \sigma(X), Z^T1 = 1} \mathbb{E}[- Y^T \log Z] \\
&= P(Y | X) 
\end{align}
as well because the latter has a smaller search space that contains the global optimal $P(Y|X)$. The claim follows. 
\end{proof}
The $L_2$ loss case $D(Y, Z) = \|Y - Z\|^2 $ is realized with $\phi(x) = x^T x$. 
For $D_\phi(Y, Z)$ with general $\phi$ that is convex with respect to both $Y$ and $Z$, see the proof in \cite{Banerjee03onthe}.

\subsubsection{A variant of functional representation lemma}
\label{appsubsec:functional_rep}
The result in this section is known as functional representation lemma \cite{gamalNetworkinfo, achille2017emergence, peters2012identifiability, Darmois}, and it roughly states that,
for any random variables $X$ and $Y$ , it is possible to represent $Y$ as a function of $(X, Z)$ such that $Z$ is independent of $X$.  
In this section, we reprove this lemma and show that, as a byproduct of the proof, we can also find a probability space in which $(X, Z)$ can also be represented as a function of $(Y, X)$ if the conditional cumulative distribution of $Y$ is smooth enough and if the cardinality of $Y|X=x$ does not differ by the choice of $x$. 
Applying this to our $\mathcal{E}$ and $\mathcal{E}_\phi$, we can therefore construct $\mathcal{E}_\psi$ that is independent of $\mathcal{E}_\phi$ such that there is an invertible map between $\mathcal{E}$ and $(\mathcal{E}_\phi, \mathcal{E}_\psi)$.

\begin{lemma}
Suppose $X$ and $Y$ are $\mathbb{R}$-valued random variables with probability space $(\Omega, P, \mathcal{F})$, and let $Y|x$ be the conditional random variable with law $p(\cdot | x)$. Assume then the $x$-parametrized cumulative distribution $F(\cdot, x)$ for $Y|x$, and suppose that inverse $F^{-1}(\cdot, x)$ of $F(\cdot, x)$ exists for every $x$. Also assume that both $F(\cdot, x)$ and its inverse are measurable with respect to its input. 
Then there exists a probability space $(\tilde \Omega, \tilde P, \tilde{\mathcal{F}})$ for which there exists $\tilde X$ and $\tilde Y$ with the same joint law as $X$, $Y$ as well as a uniform distribution $N_X$ independent from $\tilde X$ such that $\sigma(\tilde X, \tilde Y)= \sigma(N_X, \tilde Y)$ \label{thm:fxn_repn}
\end{lemma}
\begin{proof} (sketch)
Consider the product probability measure space on $\tilde \Omega= \Omega \times \Omega$, $\tilde{\mathcal{F}}  := \mathcal{F} \times \mathcal{F}$ with measure $\tilde P := P \times P$.
Let us write $\Omega_1 \times \Omega_2$ and $P_1 \times P_2$ to make distinction.
Given $X : \Omega_1 \to \mathbf{R}$,  let us also consider the random variable $U : \Omega_2 \to [0,1]$.
It it clear that $\tilde P(x \in A) =  P(x \in A)$ for all $A$.
Then letting the lower case denote the density and using the notation $p(U > a , x) := \int_{u>a}  p(u, x) du $, 
 \begin{align}
     \tilde p(F^{-1}(U, x) \geq a, x) &=  \tilde p_2(U \geq F(a, x), x) \\
     &= p_2(U \geq F(a, x) ) p_1(x) \\
     &= p_2(Y \geq a | x) p_1(x) \\
     &= p_1(Y \geq a | x) p_1(x) \\
     &= p_1(Y \geq a, x)
 \end{align}
This sequence of equality has multiple implications. 
First, $\tilde Y(w_1 \times w_2) = Y(w_2)|X(w_1)$ generated by first sampling $y$ from $P(y | X(w_1))$ as  $F^{-1}(U(w_2), x)$ has the same law as $Y : \Omega_1 \to \mathbb{R}$, and that 
$p(F^{-1}(U, X)\in A, X \in B ) = p(\tilde Y \in A, X \in B) = p_1(Y \in A, X \in B)$. 
By construction, $F^{-1}(U, X) = \tilde Y$. 
  This implies $Y \in \sigma(U ,X)$ in the product probability space.
Next, again by the construction, $F^{-1}(U, x) = Y | x$ for all $x$, so that in particular $F([Y|x], x) = U$ for every choice of $x$.
Thus, for all $w \in \Omega$, 
$F([Y(w_2)|x], x) = U(w_2)$ irrespective of the choice of $w_2 \in \Omega_2$ and $x \in \mathbb{R}$, and hence 
$F(\tilde Y(w_1, w_2), X(w_1)) =  F([Y(w_1)|X(w_2)], X(w_2)) = U(w_2)$ as well. 
This in particular implies $U \in \sigma(Y, X)$ in the product space. 
All together, we have $\sigma(U, X) = \sigma(\tilde Y, X)$ and the claim follows.
\end{proof}
This result shall be extendable to multi-dimensional euclidean case by using the same logic to the multi-variate version of cumulative distribution function.

\section{Analysis of IGA} 
\subsection{Parametrization of $P(Y | h(X), \epsilon)$ and $P(Y | h(X))$ }
\label{apppsec:param_app}
As we discuss in Section \ref{subsec:IGA}, we parametrize $P(Y | h(X), \epsilon)$ for each instaqnce of $\epsilon$ as 
\begin{align}
    P(Y | h(X), \epsilon) &:= Q(Y| X; \theta - \alpha \nabla_\theta \mathcal{L}_\epsilon(\theta)) \label{eq:app_parametrization} 
\end{align}
where $Q(Y | X ; \theta)$ is some base distribution model with sufficient representation power.
To obtain the parametrization of $P(Y | h(X))$, we use the fact that, for any measurable $A$ in the range of $Y$, $P(y \in A | h(X)) := \mathbb{E}[1_A(Y) | h(X)]$.
Now, if $\mathcal{E}'$ is the environmental variable that is correlated with $(X' , Y')$, the other realization of $(X,Y)$ representing the distribution of training set that is integrated away in $\mathcal{L}_{\mathcal{E}'}$, then $\mathcal{E}'$ is independent from $(X, Y)$ (See Section \ref{subsec:IGA}). 
Then, by the the tower-rule of conditional expectation \cite{Durrett}, 
\begin{align}
    P(y \in A | h(X) ) &:= \mathbb{E}_Y[1_A(Y) | h(X)]  \\
    &=\mathbb{E}_\mathcal{E}'[\mathbb{E}_Y[1_A(Y) | h(X), \mathcal{E}']] \\
    &\cong \mathbb{E}_{\mathcal{E}'} [P(y \in A | h(X), \mathcal{E}')] 
\end{align}
Since this holds for all $A$,  $ P(y | h(X) ) \cong \mathbb{E}_{\mathcal{E}'} [P(y | h(X), \mathcal{E}')]$. Assume that $Q(Y| X, \theta)$ is Lipschitz with respect to $\theta$ uniformly about $Y$ and $X$.
Then, by substituting  \eqref{eq:MAML_like_param_e} and using the fact that the variable $\mathcal{E}'$ used in the model parameter is independent from $(X,Y)$,  
\begin{align}
    \mathbb{E}_{\mathcal{E}'} [P(y | h(X), \mathcal{E}')] &\cong \mathbb{E}_{\mathcal{E}'}[Q(y| X; \theta - \alpha \nabla_\theta \mathcal{L}_{\mathcal{E}'}(\theta))| X ] \\ 
    &\cong Q(y| X; \theta) - \alpha \nabla_\theta Q(y|X; \theta)^T  \nabla_\theta \mathbb{E}_{\mathcal{E'}}[\mathcal{L}_{\mathcal{E}'}(\theta) |X]  + O(\alpha^2) \\ 
    &= Q(y| X; \theta) - \alpha \nabla_\theta Q(y|X; \theta)^T  \nabla_\theta \mathbb{E}_{\mathcal{E'}}[\mathcal{L}_{\mathcal{E}'}(\theta)]  + O(\alpha^2) \\ 
    &\cong  Q(y| X; \theta - \alpha \nabla_\theta \mathbb{E}_{\mathcal{E'}}[\mathcal{L}_{\mathcal{E}'}(\theta)])   + O(\alpha^2)
    \label{eq:MAML_like_param_app} 
\end{align}
Thus, when we use the parametrization \eqref{eq:app_parametrization},  we may approximate $P(Y | h(X))$ as $Q(Y| X; \theta - \mathbb{E}_\mathcal{E}[\nabla_\theta \mathcal{L}_\epsilon(\theta)])$ with an error on the scale of $O(\alpha^2)$. 



\subsection{Derivation of the IGA Penalty}
\label{appsec:derivation}
We show that $\mathbb{E}_{\mathcal{E}'}[d_{KL}(P(Y| h(X), \mathcal{E}') \| P(Y | h(X)))]$ may be approximated as $\alpha \text{trace}(\text{Var}(\nabla_{\theta} \mathcal{L}_{\mathcal{E}'}(\theta))$ with an error on the scale of $O(\alpha^2)$. 
We recall that $\mathcal{E}'$ in the expression above is a factor that is used in the determination of the model parameter and is hence independent from the inference variables $(X,Y)$ (see \ref{sec:method}).  
The derivation follows simply from substituting \eqref{eq:MAML_like_param_e} into the KL divergence:
\begin{align}
    &\mathbb{E}_{\mathcal{E}'}[d_{KL}(P(Y| h(X), \mathcal{E}') \| P(Y | h(X)))] 
    = \mathbb{E}_{\mathcal{E}'}[\log P(Y| h(X), \mathcal{E}') - \log P(Y| h(X)) ] \\
    &= \mathbb{E}_{\mathcal{E}'}[\log Q(Y| X; \theta - \alpha \nabla_{\theta} \mathcal{L}_{\mathcal{E}'}(\theta)) - \log Q(Y| X; \theta - \alpha \nabla_{\theta}  \mathbb{E}_{\mathcal{E}'}[\mathcal{L}_{\mathcal{E}'}(\theta)])) ] + O(\alpha^2) \\
    &= \mathbb{E}_{\mathcal{E}'}[\mathcal{L}_{\mathcal{E}'}(\theta - \alpha  \nabla_{\theta}  \mathcal{L}_{\mathcal{E}'}(\theta)) - \mathcal{L}_{\mathcal{E}'}(\theta - \alpha \nabla_{\theta} \mathbb{E}_{\mathcal{E}'}[\mathcal{L}_{\mathcal{E}'}(\theta)]) ] + O(\alpha^2) \\
    &= \alpha ( \mathbb{E}_{\mathcal{E}'}[  \nabla_{\theta} \mathcal{L}_{\mathcal{E}'}(\theta)^T \nabla_{\theta} \mathcal{L}_{\mathcal{E}'}(\theta)] - \mathbb{E}_{\mathcal{E}'}[\nabla_{\theta} \mathcal{L}_{\mathcal{E}'}(\theta)]^T \mathbb{E}[\nabla_{\theta} \mathcal{L}_{\mathcal{E}'}(\theta)] ) + O(\alpha^2) \\
    &= \alpha \ \text{trace}(\text{Var}(\nabla_{\theta} \mathcal{L}_{\mathcal{E}'}(\theta)) + O(\alpha^2)
\end{align}
On the third equality, we used the result in Appendix \ref{apppsec:param_app} and the fact that 
\begin{align}
    \log (x + c\alpha^2) &= \log (x) + \log (1 + c\alpha^2/x) \\
    &\cong \log (x) + O(\alpha^2)
\end{align}
for $\alpha^2$ small enough.  On the fourth equality, we used the Taylor approximation of $\mathcal{L}_{\mathcal{E}'}(\theta - \alpha \nabla_{\theta} \mathbb{E}[\mathcal{L}_{\mathcal{E}'}(\theta)])$ and $\mathcal{L}_{\mathcal{E}'}(\theta - \alpha  \nabla_{\theta}  \mathcal{L}_{\mathcal{E}'}(\theta))$ around $\theta$.
 
Thus, under sufficient regularity conditions, $\mathbb{E}_{\mathcal{E}'}[d_{KL}(P(Y| h(X), \mathcal{E}') \| P(Y | h(X)))]$ can be approximated by $\alpha \ \text{trace}(\text{Var}(\nabla_{\theta} \mathcal{L}_{\mathcal{E}'}(\theta))$ upto $O(\alpha^2)$ error.

\subsection{IGA solves the linear problem}
\label{appsec:iga_linear}
In this section, we provide the details of our claim in Section \ref{subsec:IGA_linear} and show that our IGA finds the OOD-optimal solution to the linear problem in Section \ref{subsec:IGA_linear}:
\begin{align}
X_1 &= N_0 \\
Y &= X_1 + N_1 \\
X_2 &= \mathcal{E}Y + N_2
\end{align}
We will seek the OOD-optimal solution of this problem from the parametric family of random variables of the form 
\begin{align}
\hat{Y} = w_1 X_1 + w_2 X_2.
\end{align}
For each $e$, the conditional loss of this predictor is given by 
\begin{align}
L(w_1, w_2|e) = \mathbb{E}[(Y - \hat{Y})^2|\epsilon] & = \mathbb{E}[(Y - (w_1 X_1 + w_2 X_2))^2|\epsilon] \\
& = \mathbb{E}[(N_0 + N_1 - (w_1 N_0 + w_2 (\epsilon(N_0 + N_1) + N_2)))^2|\epsilon] \\
& = \mathbb{E}[((1 - w_1 -w_2 \epsilon)\epsilon_0 + (1 - w_2 \epsilon)\epsilon_1 - w_2 \epsilon_2)^2|\epsilon] \\
& = (1 - w_1 - w_2 \epsilon)^2 + (1 - w_2 \epsilon)^2 + w_2^2 
\end{align}
Given this $L$, we consider the predictor with the following parametrization (see Section \ref{subsec:parametrization}):
\begin{align}
\hat{Y} &= f_{iga}(X_1, X_2; w_1, w_2, \alpha) \\
&= \left(w_1 - \alpha \hat{\mathbb{E}}_\epsilon\left[\frac{\partial \mathcal{L}_\epsilon}{\partial w_1}\right]\right) X_1 + \left(w_2 - \alpha \hat{\mathbb{E}}_\epsilon\left[\frac{\partial \mathcal{L}_\epsilon}{\partial w_2}\right]\right) X_2 \label{eq:IGA_param}
\end{align}
IGA then seeks the solution $f_{iga}(X_1, X_2; w_1^*, w_2^*, \alpha))^2$ that satisfies 
\begin{align}
&[w_1^*, w_2^*] = \argmin_{w_1, w_2, \alpha} \mathbb{E}[(Y - f_{iga}(X_1, X_2; w_1, w_2, \alpha))^2] \\
&\text{s.t. ~} \text{Trace}\left(\widehat{\text{Var}}\left(\frac{\partial \mathcal{L}_\epsilon}{\partial w_1}\right)\right) + \text{Trace}\left(\widehat{\text{Var}}\left(\frac{\partial \mathcal{L}_\epsilon}{\partial w_2}\right)\right) = 0 .
\end{align}
We will show that this solution agrees with the OOD-optimal invariant predictor $X_1$ if $\widehat{\text{Var}}(\mathcal{E}) > 0.$ That is,
$f_{iga}(X_1, X_2; w_1^*, w_2^*, \alpha))^2 = X_1$. 

\begin{proof}
First, each coordinate of the gradient $\nabla_w \mathcal{L}_\epsilon(w)$ is given by 
\begin{align}
\frac{\partial \mathcal{L}_\epsilon}{\partial w_1} &= 2 (w_1 + w_2 \epsilon - 1) \\   
\frac{\partial \mathcal{L}_\epsilon}{\partial w_2} &= 2 (- \epsilon + \epsilon^2 w_2 + w_2 + \epsilon^2 w_2 + \epsilon w_1 - \epsilon) \\
&= 2 (\epsilon w_1 + (2\epsilon^2 + 1)w_2 - 2\epsilon)
\end{align}

Next, our constraint requires that
$\text{Trace}\left(\widehat{\text{Var}}\left(\frac{\partial \mathcal{L}_\epsilon}{\partial w_1}\right)\right)$ and $ \text{Trace}\left(\widehat{\text{Var}}\left(\frac{\partial \mathcal{L}_\epsilon}{\partial w_2}\right)\right)$ are both zero.

Computing both of these terms, we get 
\begin{align}
&\text{Trace}\left(\widehat{\text{Var}}\left(\frac{\partial \mathcal{L}_\epsilon}{\partial w_1}\right)\right) \\
&=  \hat{\mathbb{E}}_\epsilon\left[ \left( \frac{\partial \mathcal{L}_\epsilon}{\partial w_1} - \hat{\mathbb{E}}_e\left[ \frac{\partial \mathcal{L}_\epsilon}{\partial w_1} \right] \right)^2\right] \\
&= \hat{\mathbb{E}}_\epsilon\left[ \left( 2 (w_1 + w_2\epsilon - 1) - \hat{\mathbb{E}}_\epsilon\left[ 2 (w_1 + w_2\epsilon - 1) \right] \right)^2\right] \\
&= 4\hat{\mathbb{E}}_\epsilon\left[ \left( w_1 + w_2\epsilon - 1 - (w_1 + w_2\hat{\mathbb{E}}_\epsilon[\epsilon] - 1) \right)^2 \right] \\
&= 4\hat{\mathbb{E}}_\epsilon\left[ \left( w_1 + w_2\epsilon - 1 - (w_1 + w_2\hat{\mathbb{E}}_\epsilon[\epsilon] - 1) \right)^2 \right] \\
&= 4\hat{\mathbb{E}}_\epsilon\left[ \left((\epsilon - \hat{\mathbb{E}}_\epsilon[\epsilon] )w_2 \right)^2 \right] \\
&= 4\widehat{\text{Var}}(\epsilon)w_2^2 \label{eq:var_w_1} 
\end{align}

and 
\begin{align}
\text{Trace}\left(\widehat{\text{Var}}\left(\frac{\partial \mathcal{L}_\epsilon}{\partial w_2}\right)\right) &=  \hat{\mathbb{E}}_\epsilon\left[ \left( \frac{\partial \mathcal{L}_\epsilon}{\partial w_2} - \hat{\mathbb{E}}_\epsilon\left[ \frac{\partial \mathcal{L}_\epsilon}{\partial w_2} \right] \right)^2\right] \\
&= \hat{\mathbb{E}}_\epsilon\Big[ \big( 2 (\epsilon w_1 + (2\epsilon^2 + 1)w_2 - 2\epsilon) \nonumber \\
& ~~~~~ - \hat{\mathbb{E}}_\epsilon\left[ 2 (\epsilon w_1 + (2\epsilon^2 + 1)w_2 - 2\epsilon) \right] \big)^2 \Big] \\ 
&= 4\hat{\mathbb{E}}_\epsilon \Big[ \big( (\epsilon w_1 + (2\epsilon^2 + 1)w_2 - 2\epsilon) \nonumber \\
& ~~~~~ -  (\hat{\mathbb{E}}_\epsilon[\epsilon] w_1 + (2\hat{\mathbb{E}}_\epsilon[\epsilon^2] + 1)w_2 - 2\hat{\mathbb{E}}_\epsilon[\epsilon]) \big)^2 \Big] \\
&= 4\hat{\mathbb{E}}_\epsilon\left[ \left( (\epsilon - \hat{\mathbb{E}}_\epsilon[\epsilon])(w_1 -2) + 2(\epsilon^2 - \hat{\mathbb{E}}_\epsilon[\epsilon^2])w_2 \right)^2\right] \label{eq:var_w_2} 
\end{align}

Because $\hat{\text{Var}}(\epsilon) > 0$ by assumption, \eqref{eq:var_w_1} requires that $w_2 = 0$.
Substituting $w_2 = 0$ into \eqref{eq:var_w_1}, we obtain \begin{align}
0 &= \text{Trace}\left(\text{Var}\left(\frac{\partial \mathcal{L}_\epsilon}{\partial w_1}\right)\right) \\
&= 4\hat{\mathbb{E}}_\epsilon\left[ \left( (\epsilon - \hat{\mathbb{E}}_\epsilon[\epsilon])(w_1 -2) \right)^2\right] \\
&= 4\hat{\widehat{\text{Var}}}(\epsilon)(w_1 - 2)^2
\end{align}
and this forces $w_1 = 2$.
Putting this back into the gradient $\nabla_w \mathcal{L}_\epsilon(w)$, we also get 
\begin{align}
\frac{\partial \mathcal{L}_\epsilon}{\partial w_1} &= 2\\   
\frac{\partial \mathcal{L}_\epsilon}{\partial w_2} &= 2 (2\epsilon - 2\epsilon) = 0
\end{align}
At this point, our hands are tied when it comes to the optimization of $w$;  
$w_2^* =2$ and $w_1^* = 0$ necessarily.
Now, how about the optimization of $\alpha$?
All together, our \eqref{eq:IGA_param} becomes
\begin{align}
\hat{Y} &= f_{iga}(X_1, X_2; w_1^*, w_2^*, \alpha) \\
&= \left(2 - 2\alpha\right) X_1
\end{align}
Optimizing the $\alpha$ in the expression above about the loss 
\begin{align}
&\mathbb{E}[(Y - f_{iga}(X_1, X_2; w_1, w_2, \alpha))^2]\\
&= (1 - \left(2 - 2\alpha\right))^2 + 1
\end{align}
we get $\alpha^* = 0.5$. 
Thus, the optimal $\hat Y = (2- 2*0.5) X_1 = X_1$, and 
we obtain the desired OOD optimal solution for this problem.
\end{proof}

\subsection{To what range of distributions does IGA extrapolate?} \label{appsec:app_bound}

Our IGA provides an approximation scheme for MIP \eqref{eq:MIP_objective_informal}.
However, when used in practice with a finite number of training environments, what is the range of environments on which the solution of IGA can generalize?
\cite{alexis_confounding} provides another derivation of IGA that gives us some clue for this question. 
If our predictor model as a function of its parameters belongs to a Sobolev space, we can use the same argument as in the theorem 1 of \cite{alexis_confounding} to bound the loss of the model on the affine combination of the training environment from above by the IGA loss; 
\begin{prop}
Let $\mathcal{E}_{train}$ be the set of training environments, and
let $\Delta_\eta = \{ \{\alpha_\epsilon\} ; \alpha_\epsilon > - \eta, ~  \sum_\epsilon \alpha_\epsilon =1 \}$. 
Then $\mathcal{E}_\eta  = \{ \sum_{\epsilon \in \mathcal{E}_{train}} \alpha_\epsilon P_\epsilon ; \alpha \in \Delta \eta\}$
defines a linear space of probability distributions. 
Now, let 
$\tilde \theta = \theta - \alpha \nabla_\theta E[L_\epsilon(\theta)]$
as defined in \ref{subsec:IGA}, 
and assume that $L_\epsilon$ lives in a sobolev space $W^{1,2}$ with respect to the parameter $\theta$ and that its evaluation is a bounded linear operator.
Then 
\begin{align}
\sup_{\epsilon \in \mathcal{E}_\eta} L_\epsilon (\tilde \theta) \leq 
\hat E_{\epsilon \in \mathcal{E}_{train}} [L_\epsilon(\tilde \theta)] +  M_\eta  \| \text{trace} (\text{Var}_{\mathcal{E}}(\nabla_{\theta} \Delta_\mathcal{E}(\tilde \theta)))\|_{L_{2}(\theta)}
\end{align}
for $ M_\eta > 0$ that is monotonic in $\eta$. 
\end{prop}
Note that the right-hand side of the bound above is almost identical to the IGA loss. 
However, if we approximate the $L_{2}(\theta)$ norm by an evaluation at specific $\theta$, we retrieve the IGA loss itself \ref{subsec:IGA}. The evaluation functional $\phi : L_\epsilon \to L_\epsilon(\theta)$ can be bounded if it also lives in RKHS subspace.
\begin{proof}
The proof borrows almost exactly from the one provided in \cite{alexis_confounding}.
Let us write $n = |\mathcal{E}|$.
\begin{align}
    &\sup_{\epsilon \in \mathcal{E}_\eta} L_\epsilon (\tilde \theta) = \sup_{\alpha \in \Delta \eta} \sum_{\epsilon \in \mathcal{E}_{train}} \alpha_\epsilon L_\epsilon(\tilde \theta) \\
    &= (1 + n \eta) \sup_{\epsilon \in \mathcal{E}_{train}} L_\epsilon (\tilde \theta) - \eta \sum_{\epsilon \in \mathcal{E}_{train}} L_\epsilon(\tilde \theta)\\ 
    &= \hat E_{\epsilon \in \mathcal{E}_{train}} [L_\epsilon(\tilde \theta)]  + (1 + n\eta ) \sup_{\epsilon \in \mathcal{E}_{train}} L_\epsilon(\tilde \theta) - (1 + n\eta ) \frac{1}{n} \sum_{\epsilon \in \mathcal{E}_{train}}  L_\epsilon(\tilde \theta)] \\
    & =  \hat E_{\epsilon \in \mathcal{E}_{train}} [L_\epsilon(\tilde \theta)] 
    + (1 + n\eta) \left\{  \sup_{\epsilon \in \mathcal{E}_{train}} L_\epsilon(\tilde \theta) -   \hat E_{\epsilon \in \mathcal{E}_{train}} [L_\epsilon(\tilde \theta)] \right\} \\
    & = \hat E_{\epsilon \in \mathcal{E}_{train}} [L_\epsilon(\tilde \theta)]  
    + (1 + n\eta) \left\{  \sup_{\epsilon \in \mathcal{E}_{train}} \left( L_\epsilon(\tilde \theta) -   \hat E_{\epsilon \in \mathcal{E}_{train}} [L_\epsilon(\tilde \theta)] \right)\right\} \label{eq:ineq_Poincare} 
\end{align}
For simplicity, let us put 
$$\Delta_\epsilon(\theta) = L_\epsilon(\tilde \theta) -   \hat E_{\epsilon \in \mathcal{E}_{train}} [L_\epsilon(\tilde \theta)].$$ 
Then continuing from \ref{eq:ineq_Poincare}, we can use the assumption about $L_\epsilon$ to bound   $\sup_{\epsilon \in \mathcal{E}_\eta} L_\epsilon (\tilde \theta)$ from above by 
\begin{align}
    \hat E_{\epsilon \in \mathcal{E}_{train}} [L_\epsilon(\tilde \theta)]   + (1 + n\eta) M  \| \sup_{\epsilon \in \mathcal{E}_{train}} \Delta_\epsilon(\theta) \|_{L_2}  \\
\leq \hat E_{\epsilon \in \mathcal{E}_{train}} [L_\epsilon(\tilde \theta)]  + (1 + n\eta) M  \sup_{\epsilon \in \mathcal{E}_{train}} \| \Delta_\epsilon(\theta) \|_{L_2} 
\end{align}
for some $M$.  We can bound  $\sup_{\epsilon \in \mathcal{E}_{train}} \| \Delta_\epsilon(\theta) \|_{L_2} $ from above by 
$$\|\sup_{\epsilon} \Delta_\epsilon(\theta) \|_{L_{2}} \! \leq \sup_{\epsilon} \|\Delta_\epsilon(\theta) \|_{L_{2}} \! \leq \sup_{\epsilon} \|\nabla_{\theta} \Delta_\epsilon (\theta)\|_{L_{2}},$$ 
where we used Poincare inequality in the last inequality. 
If $\Delta_\epsilon$ is also Lipschitz about $\epsilon$ this can be bounded from above by  $\tilde M \| \textrm{trace} \textrm{Var}_{\mathcal{E}}(\nabla_{\theta} \Delta_\mathcal{E}(\theta))\|_{L_{2}}$ for some $\tilde M \! > \! 0$, and our claim follows.
\end{proof}


This suggests that, while IGA is derived from MIP in our work, it can also be interpreted from the perspective of [1], and that the OOD loss of its solution can be approximately bounded (by the similar bound we derived above) on a set of affine combinations of the training environments.

\section{Implementation Detail}
\label{appsec:implement}
In this section we describe the details of the experiment design along with the architectures of the models we used. 
In order to present a self-contained material, we first restate the experimental setting we already described in the main manuscript.

\subsection{More details on Invariant Unite test}
\label{appsubsec:exp_unit_test}

For the evaluation of all models in Invariance unit test, we used the code published in \url{https://github.com/facebookresearch/InvarianceUnitTests} (MIT license). 
However, we trained each model with 5 times longer iterations than the experiments in the original article (50000),
and used different set of \textit{Dataseed}(5) and \textit{Modelseed}(15). 
We also conducted experiments for $n_{env} = 3, 5, 7, 10$ only.
We have set the search range of $\lambda$ in IGA to $100 - 1000$.
The results of IGA differs from the values reported in \cite{aubinlinear} most likely because we trained each model with longer iteration and because we also trained the parameter $\alpha$ in \eqref{eq:MIP_objective_informal} that was not trained in the preprint version of this paper.
To obtain all the results in the table, we used 256CPU over 7days.
This experiment is implemented by \cite{paszke2019pytorch}.


\subsection{Colored MNIST} 
\label{appsubsec:cmnist}
\textit{Colored MNIST} is an experiment proposed in \cite{arjovsky2019invariant} (code published in \url{https://github.com/facebookresearch/InvariantRiskMinimization} (The license is provided in \url{https://github.com/facebookresearch/InvariantRiskMinimization/blob/master/LICENSE}). ). 
The goal of the task in \textit{Colored MNIST} is to predict the label of a given digit in the presence of varying exterior factor, $\mathcal{E}$. 
The left panel of Figure \ref{appfig:mnist} is a Bayesian Network representation of this experiment.
Each member of the \textit{Colored MNIST} dataset is constructed from an image-label pair $(x, y)$ in  MNIST, as follows.
\begin{enumerate}
\setlength{\itemsep}{0.1cm} 
    \item Assign a binary label $\hat y_{obs}$ from $y$ with the following rule: $\hat y_{obs}=0$ if $y \in \{0 \sim 4\}$ and  $\hat y_{obs}=1$ otherwise.  
    \item Flip $\hat y_{obs}$ with a fixed probability $p$ to produce $y_{obs}$.
    \item Let $x_{fig}$ be the binary image corresponding to $y$.
    \item Put $y_{obs}= \hat x_{ch1}$, and construct $x_{ch1}$ from  $\hat x_{ch}$ by flipping $\hat x_{ch1}$ with probability $e$.
    \item Construct $x_{ods} = x_{fig} \times  [x_{ch0}, (1-x_{ch0}), 0]$.(that is, red if $x_{ch1}=1$ and green if   $x_{ch1}=0$.) Indeed, $x_{obs}$ has exactly same information as the pair $(x_{fig},  x_{ch1})$. 
\end{enumerate}
In this experiment, only $(Y_{obs}, X_{obs})$ are assumed observable.
At training times, the machine learner will be given a set of datasets $\mathcal{D}_{train} = \{D_e ; e \in R_{train}\}$ in which $D_e$ is a set of observations gathered when $\mathcal{E} = e$. 
We set $|R_{train}| =2$, and choose $|D_e| = 25000$.
More particularly, for the $e_1$ we chose the flip-rate($p$) to be $0.1$,
and chose $p = 0.2$ for the $e_2$. 
Each image was resized to $14 \times 14$ resolution.

For the test evaluation, we randomly sampled 10 instances of $p$ uniformly from the range $[0, 1]$ to construct $R_{test}$, and approximated the OOD accuracy by computing the worst performance over all $R_{test}$ .
We used $5$ seeds to produce each numerical result.
For the model, we used 4 Layers MLP with 2500 units per each layer and elu activation\cite{clevert2015fast}, and did not use bias term in the last sigmoid activation. 
We used batch normalization (BN)\cite{ioffe2015batch} for each layer, and optimized the model using  
Adam\cite{kingma2014adam} with alpha = 0.0015, beta1=0.0, beta2=0.9 over 500 iterations.
In general, less number of iterations yielded better results when $|\mathcal{R}_{train}|$ was small (less overfitting). 
On C-MNIST, we trained our IGA-models while fixing $\alpha=0$ (See Appendix \ref{appsubsec:method_lim} for the reason of choosing this setting ). 
For this set of experiment, we used 8 NVIDIA Tesla P100 GPUs.
This experiment is implemented by \cite{tokui2019chainer}.

\subsection{Extended Colored MNIST} 
\label{appsubsec:ecmnist}
As described in the main manuscript, Extended Colored MNIST is a modified version of colored MNIST, in which the dataset was constructed using the following procedure.
The right panel of Figure \ref{appfig:mnist} is a Bayesian Network representation of this experiment.
\begin{enumerate}
\setlength{\itemsep}{0.1cm}
\item Set $x_{ch2}$ to $1$ with probability $e_{ch2}$. Set it to $0$ with probability $1- e_{ch2}$. 
\item Construct $\hat y_{obs}$ in the same way as in \textit{Colored MNIST}.  If $e_{ch2}=k$, construct $y_{obs}$ by flipping $\hat y_{obs}$ with probability $p_k$($k \in \{0, 1\}$.)
\item Put $y_{obs}= \hat x_{ch0}$, and construct $x_{ch0}$ from  $\hat x_{ch}$ by flipping $\hat x_{ch1}$ with probability $e_{ch0}$.
\item Construct $x_{obs}$ as $x_{fig} \times  [x_{ch0}, (1-x_{ch0}), x_{ch2}]$.  As an RGB image, this will come out as an image in which the red scale is \textit{turned on} and the green scale is \textit{turned off} if $x_{ch0}=1$, and otherway around if $x_{ch0} =0$. 
Blue scale is turned-on only if $x_{ch2} = 1$. 
\end{enumerate}
In this experiment, we set $|R_{train}| =5$, and choose $|D_e| = 10000$,
and resized each image in the dataset to $14 \times 14$ resolution.
To produce $e \in R_{train}$, we selected $e_{ch0}$ randomly from the range $[0.1, 0.2]$, and selected $e_{ch2}$ randomly from the range $[0.3, 0.4]$.

Mean while, we set $|R_{test}| = 9$.
To produce $n$-th member of $R_{test}$, we set $e_{ch0} = 0.1$ and  
we selected $e_{ch2}$ randomly from the range $[0.0, 1.0]$.
We chose $p_0=0.25$, $p_1=0.75$ for both $R_{test}$ and $R_{train}$.

We used $5$ seeds to produce each numerical result.
For the model, we used 4 Layers MLP with 2500 units per each layer, and did not use bias term in the last sigmoid activation. 
We used batch normalization for each layer, and optimized the model using  
Adam with alpha = 0.0005, beta1=0.0, beta2=0.9 over 2000 iterations.
The performance-values of IRM in the Table 3 of the main article are the results produced by the model that achieved the best average \textbf{\textit{train}} accuracy among all models trained with $\lambda > 10^4$.
The averages were computed over 5 seeds.
On EC-MNIST, we trained our IGA-models while fixing $\alpha=0$ (See Appendix \ref{appsubsec:method_lim} for the reason of choosing this setting). 
For this set of experiment, we used 8 NVIDIA Tesla P100 GPUs.
This experiment is implemented by \cite{tokui2019chainer}.



\subsection{Other technical limitations of IGA}
\label{appsubsec:method_lim}
In IGA, the invariant feature $h(X)$ is hidden within the base distribution $Q(Y|X; \theta)$ it does not appear explicitly in the algorithm.
Thus, even after training the predictor, we have no way of extracting the invariant feature itself. 
This is not the case in IRM \cite{arjovsky2019invariant} and causality-inspired methods \cite{rojas2018invariant, prevent19a},  because they train $h(X)$ and $P(Y|h(X))$ separately.
Causality inspired methods makes this possible by restricting the search space of $h(X)$ to the family of masking functions, and IRM makes this possible by restricting the search space of $Y | h(X)$ to the family of predictors that is linear in $h(X)$ (i.e., $Y | h(X) = w^T h(X)$.)  
Meanwhile, we considered the possibility that, for an arbitrary nonlinear $h(X)$,  the complexity of $P(Y | h(X) , \epsilon)$ itself might differ across environments. 
Although this allows us to treat more general situations, this comes at the cost of not being able to identify the explicit form of $h(X)$.  
Also, we shall note that IGA is not a method to check whether a given feature satisfies a controllability condition; IGA is a method aimed at solving our MIP-\textit{invariance problem} under that assumption that there exists at least one feature that satisfies the controllability condition.


Finally, we shall report that, in our application of IGA to MNIST derived datasets, the training of the $\alpha$ parameter in our inference model (\ref{eq:MAML_like_param}) was often unstable.
It seems that, when we train our formulation with an overparametrized model like MLP, it is difficult to train $\alpha$ together with the model parameters. 
From this observation, we deduced that we shall use a fixed value of $\alpha$ to reduce the parameter redundancy.
Also, as we mention in the method section and \ref{apppsec:param_app}, a small $\alpha$ better justifies the Taylor approximation. 
We therefore decided to use a very small, fixed value of $\alpha$ during the training. 
When we experimented with such settings, we also observed that the forward output did not change much between setting $\alpha$ to $0$ and to a very small value.
Therefore, in the end, we decided to set $\alpha \! = \! 0$ at the forward time in the implementation. 
We shall emphasize, however, that 
since our variance approximation holds for arbitrarily small $\alpha$, our use of $\alpha \! = \! 0$ here is strictly implementational, and it does not contradict our formulation. 


\begin{figure}[h]
\centering
\hspace{-0.5cm}
\begin{minipage}[t]{0.35\hsize}
\begin{center}
\hspace{-0.3cm}
\includegraphics[width=\linewidth]{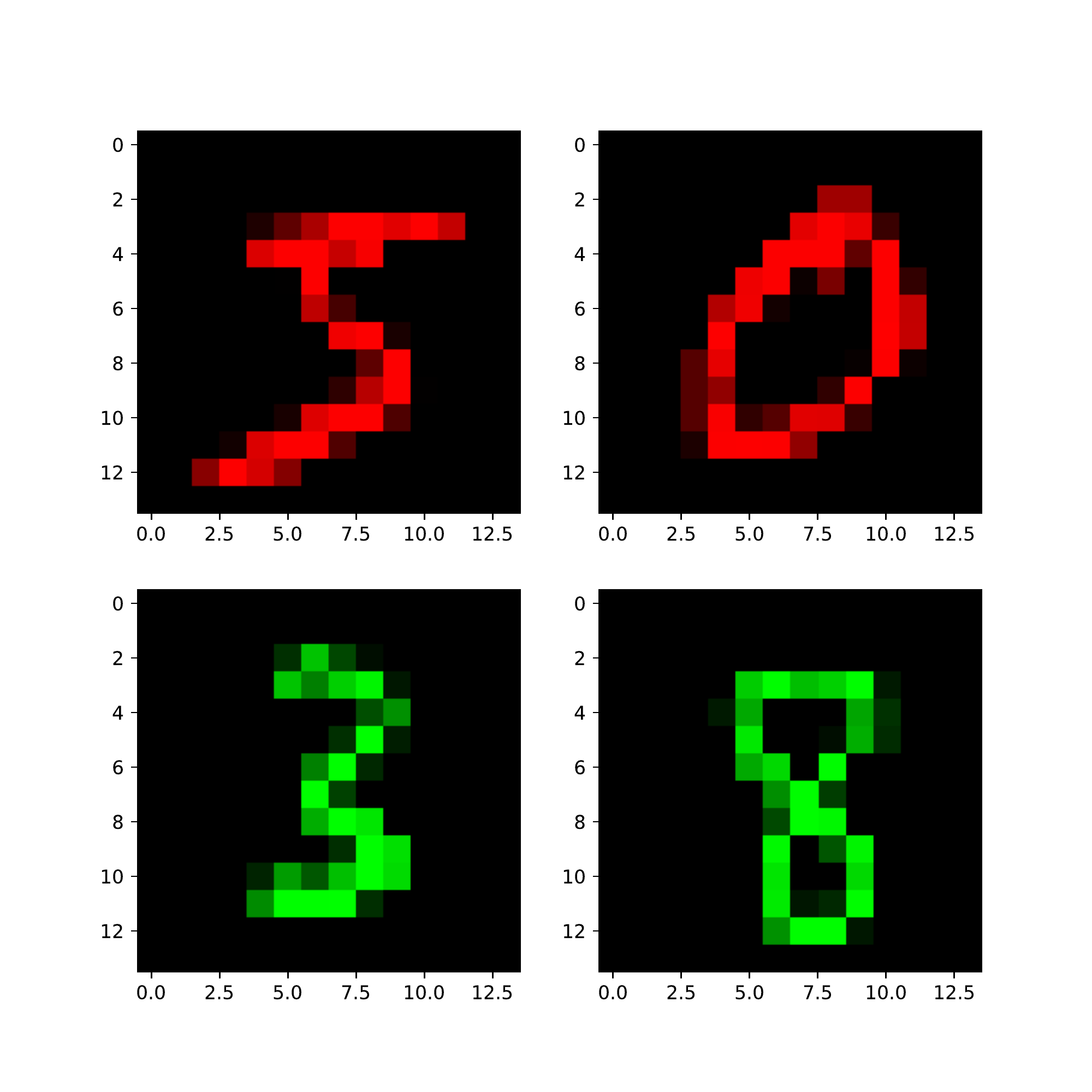}
\caption{Example images of \textit{Colored MNIST}. Only two channels are used for all images, and the colors are flipped randomly by the exterior factor.}
\end{center}
\end{minipage}
%
\hspace{0.5cm}
\begin{minipage}[t]{0.35\hsize}
\begin{center}
\hspace{-0.15cm}
\includegraphics[width=\linewidth]{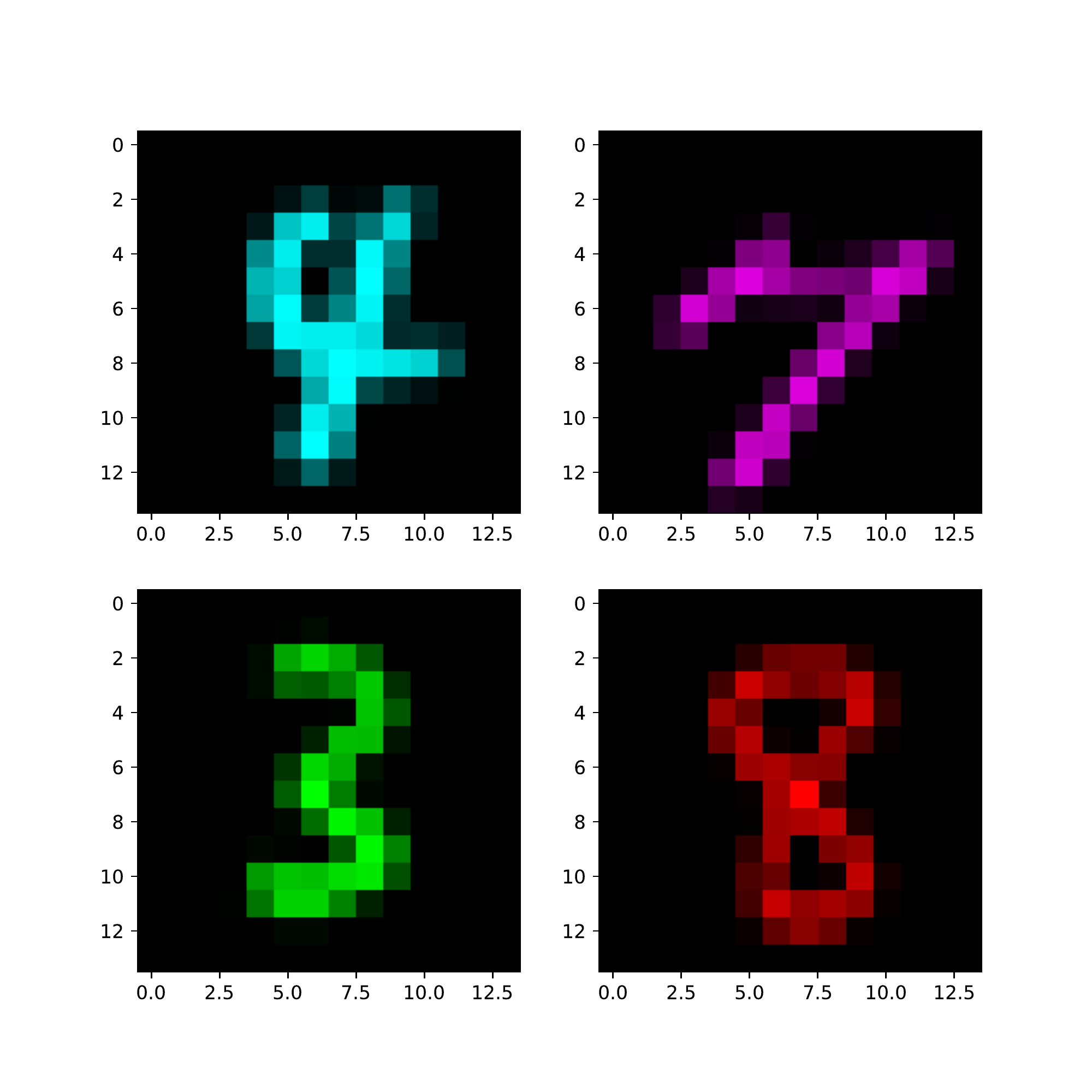}
\caption{Example images of \textit{Extended Colored MNIST}. The first  two channels and the third channel are perturbed by the different mechanism.  See the main manuscript for the way of the construction.}
\end{center}
\end{minipage}
\label{fig:cmnist_fig}
\end{figure}

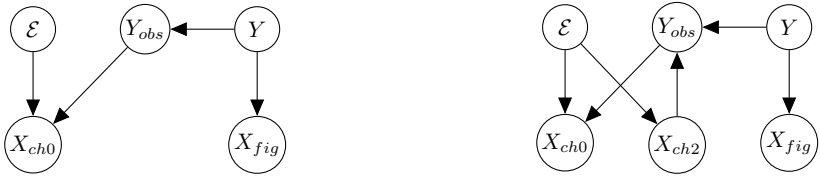
\begin{figure}[t]
\begin{tabular}{cc}
\centering
\hspace{-0.3cm}
\begin{minipage}[t]{0.5\hsize}
\begin{center}
\begin{tikzpicture}[scale=0.85,transform shape,wrap/.style={inner sep=0pt,
fit=#1,transform shape=false}]
 \node[latent](e){$\mathcal{E}$};
 \node[latent, right=of e](ylabel){$Y_{obs}$};
  \node[latent, right=of ylabel](y){$Y$};
  \node[latent, below=of e](x2){$X_{ch0}$};
  \node[latent, below=of y](x1){$X_{fig}$};
 \edge {e, ylabel}{x2};
 \edge {y}{ylabel};
 \edge {y}{x1};
%
\end{tikzpicture}
\end{center}
\end{minipage}
\centering
\begin{minipage}[t]{0.5\hsize}
\begin{center}
\begin{tikzpicture}[scale=0.85,transform shape,wrap/.style={inner sep=0pt,
fit=#1,transform shape=false}]
 \node[latent](e){$\mathcal{E}$};
 \node[latent, right=of e](ylabel){$Y_{obs}$};
  \node[latent, right=of ylabel](y){$Y$};
 \node[latent, below=of e](x2){$X_{ch0}$};
  \node[latent, below=of y](x1){$X_{fig}$};
  \node[latent, below=of ylabel](x3){$X_{ch2}$};
 \edge {e, ylabel}{x2};
 \edge {e}{x3};
 \edge {y, x3}{ylabel};
 \edge {y}{x1};
\end{tikzpicture}
\end{center}
\end{minipage}
\end{tabular}
\caption{The graphical model of \textit{Colored MNIST}(left) and \textit{Extended Colored MNIST}(right)}
\label{appfig:mnist}
\end{figure}

\section{Additional Result of Invariance Unit Tests}
\label{appsec:add_iut_result}

Table \ref{appfig:IUT_table} below is the full version of Table \ref{fig:IUT_table} in main script.

\begin{figure*}[h!]
  \begin{minipage}{1.0\textwidth}
  \centering
  \scalebox{0.9}{
  \begin{tabular}{ll|c|cccc}
    \toprule
     & & Oracle & ERM & ANDMask & IRM & IGA (Ours)\\
    \midrule
    \multirow{4}{*}{Example2} 
    & $n_{env}=3$ 
    & $0.0 {\scriptstyle \pm .00}$ 
    & $0.02 {\scriptstyle \pm .01}$ 
    & $\cellbest 0.00 {\scriptstyle \pm .00}$ 
    & $0.46 {\scriptstyle \pm .01}$ 
    & $\cellbest 0.00 {\scriptstyle \pm .00}$ \\
    & $n_{env}=5$ 
    & $0.00 {\scriptstyle \pm .00}$ 
    & $\cellbest 0.00 {\scriptstyle \pm .00}$ 
    & $\cellbest 0.00 {\scriptstyle \pm .00}$ 
    & $0.46 {\scriptstyle \pm .02}$ 
    & $\cellbest 0.00 {\scriptstyle \pm .00}$ \\
    & $n_{env}=7$ 
    & $0.00 {\scriptstyle \pm .00}$ 
    & $\cellbest 0.00 {\scriptstyle \pm .00}$ 
    & $\cellbest 0.00 {\scriptstyle \pm .00}$ 
    & $0.46 {\scriptstyle \pm .02}$ 
    & $\cellbest 0.00 {\scriptstyle \pm .00}$ \\
    & $n_{env}=10$ 
    & $0.00 {\scriptstyle \pm .00}$ 
    & $\cellbest 0.00 {\scriptstyle \pm .00}$ 
    & $\cellbest 0.00 {\scriptstyle \pm .00}$ 
    & $0.47 {\scriptstyle \pm .02}$ 
    & $\cellbest 0.00 {\scriptstyle \pm .00}$ \\
    \hline
    \multirow{4}{*}{Example2s} 
    & $n_{env}=3$ 
    & $0.0 \ {\scriptstyle \pm .000}$ 
    & $0.08 \ {\scriptstyle \pm .05}$ 
    & $0.45 \ {\scriptstyle \pm .01}$ 
    & $0.45 \ {\scriptstyle \pm .01}$ 
    & $\cellbest 0.00 {\scriptstyle \pm .00}$ \\
    & $n_{env}=5$ 
    & $0.00 \ {\scriptstyle \pm .00}$ 
    & $0.03 \ {\scriptstyle \pm .03}$ 
    & $0.46 \ {\scriptstyle \pm .02}$ 
    & $0.46 \ {\scriptstyle \pm .02}$ 
    & $\cellbest 0.00 {\scriptstyle \pm .00}$ \\
    & $n_{env}=7$ 
    & $0.00 \ {\scriptstyle \pm .00}$ 
    & $0.02 \ {\scriptstyle \pm .01}$ 
    & $0.46 \ {\scriptstyle \pm .02}$ 
    & $0.46 \ {\scriptstyle \pm .02}$ 
    & $\cellbest 0.00 {\scriptstyle \pm .00}$ \\
    & $n_{env}=10$ 
    & $0.00 \ {\scriptstyle \pm .00}$ 
    & $\cellbest 0.00 \ {\scriptstyle \pm .00}$ 
    & $0.47 \ {\scriptstyle \pm .03}$ 
    & $0.47 \ {\scriptstyle \pm .03}$ 
    & $\cellbest 0.00 {\scriptstyle \pm .00}$ \\
    \hline
    \multirow{4}{*}{Example3} 
    & $n_{env}=3$ 
    & $0.01 \ {\scriptstyle \pm .00}$ 
    & $0.50 \ {\scriptstyle \pm .01}$ 
    & $\cellbest 0.30 {\scriptstyle \pm .24}$ 
    & $0.50 \ {\scriptstyle \pm .01}$
    & $0.50 \ {\scriptstyle \pm .01}$ \\
    & $n_{env}=5$ 
    & $0.01 \ {\scriptstyle \pm .00}$ 
    & $0.37 \ {\scriptstyle \pm .19}$ 
    & $\cellbest 0.01 {\scriptstyle \pm .00}$ 
    & $0.23 \ {\scriptstyle \pm .22}$ 
    & $0.36 \ {\scriptstyle \pm .19}$ \\
    & $n_{env}=7$ 
    & $0.01 \ {\scriptstyle \pm .00}$ 
    & $0.26 \ {\scriptstyle \pm .23}$ 
    & $\cellbest 0.01 {\scriptstyle \pm .00}$ 
    & $0.23 \ {\scriptstyle \pm .21}$ 
    & $0.21 \ {\scriptstyle \pm .23}$ \\
    & $n_{env}=10$ 
    & $0.01 \ {\scriptstyle \pm .00}$ 
    & $0.06 \ {\scriptstyle \pm .07}$ 
    & $\cellbest 0.01 {\scriptstyle \pm .00}$ 
    & $0.08 \ {\scriptstyle \pm .08}$ 
    & $0.03 \ {\scriptstyle \pm .02}$ \\
    \hline
    \multirow{4}{*}{Example3s} 
    & $n_{env}=3$ 
    & $0.01 \ {\scriptstyle \pm .00}$ 
    & $0.50 \ {\scriptstyle \pm .01}$ 
    & $0.50 \ {\scriptstyle \pm .00}$ 
    & $0.50 \ {\scriptstyle \pm .01}$ 
    & $0.50 \ {\scriptstyle \pm .01}$ \\
    & $n_{env}=5$ 
    & $0.01 \ {\scriptstyle \pm .00}$ 
    & $0.36 \ {\scriptstyle \pm .20}$ 
    & $0.50 \ {\scriptstyle \pm .00}$ 
    & $0.38 \ {\scriptstyle \pm .17}$ 
    & $\cellbest 0.35 {\scriptstyle \pm .21}$ \\
    & $n_{env}=7$ 
    & $0.01 \ {\scriptstyle \pm .00}$ 
    & $0.26 \ {\scriptstyle \pm .23}$ 
    & $0.48 \ {\scriptstyle \pm .08}$ 
    & $0.24 \ {\scriptstyle \pm .22}$ 
    & $\cellbest 0.21 {\scriptstyle \pm .23}$ \\
    & $n_{env}=10$ 
    & $0.01 \ {\scriptstyle \pm .00}$ 
    & $0.03 \ {\scriptstyle \pm .02}$ 
    & $0.51 \ {\scriptstyle \pm .10}$ 
    & $\cellbest 0.01 {\scriptstyle \pm .00}$ 
    & $0.02 \ {\scriptstyle \pm .01}$ \\
    \bottomrule
    \end{tabular}
    }
    \end{minipage} \tblcaption{Test errors for all algorithms and datasets for the invariant unit tests in \cite{aubinlinear}  $(d_{inv}, d_{spu}, n_{env}) = (5, 5, n_{env})$ ($n_{env}$ is the number of training environments).} 
    \label{appfig:IUT_table}
\end{figure*}

\section{Additional Result of \textit{Colored MNIST} and \textit{Extended Colored MNIST}}

\label{appsec:add_result}
The result of \textit{Extended Colored MNIST} with $p_0=0.25$, $p_1=0.65$ for both $R_{test}$ and $R_{train}$.  
Our algorithm outperforms the Invariant Risk Minimization(IRM)\cite{arjovsky2019invariant} in this case as well in Figure \ref{fig:excmnist_065}.

\begin{figure}[h!]
\centering
\begin{minipage}[t]{0.75\hsize}
\begin{center}
\includegraphics[width=\linewidth]{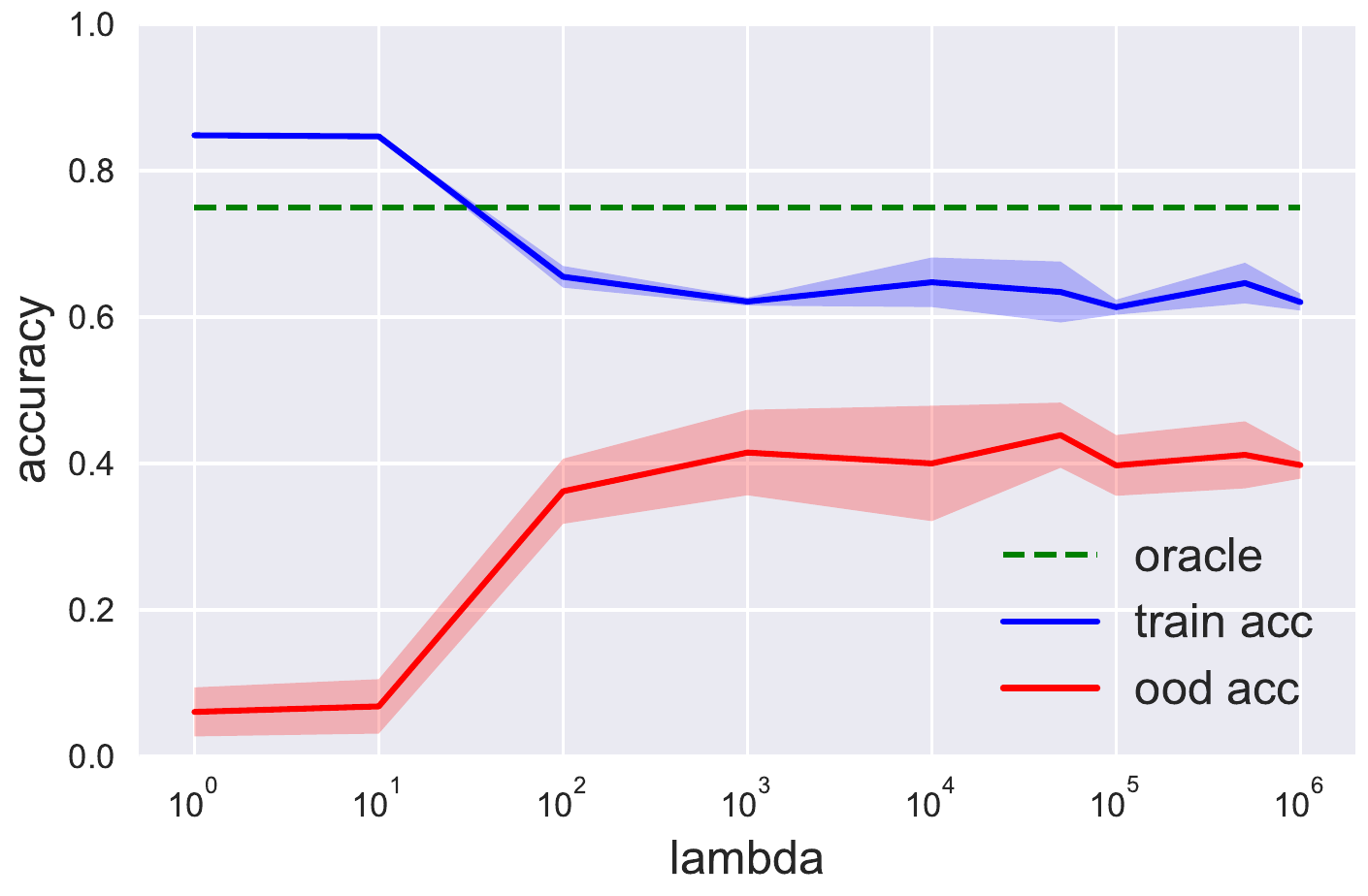}
\end{center}
\end{minipage}
\caption{Result on \textit{Colored MNIST} with MLP without BN.}
\label{fig:cmnist_mlp_result}
\end{figure}

\begin{figure*}[h]
\centering
\begin{minipage}[t]{0.35\hsize}
\begin{center}
\hspace{-0.3cm}
\stackunder[5pt]{\includegraphics[width=\linewidth]{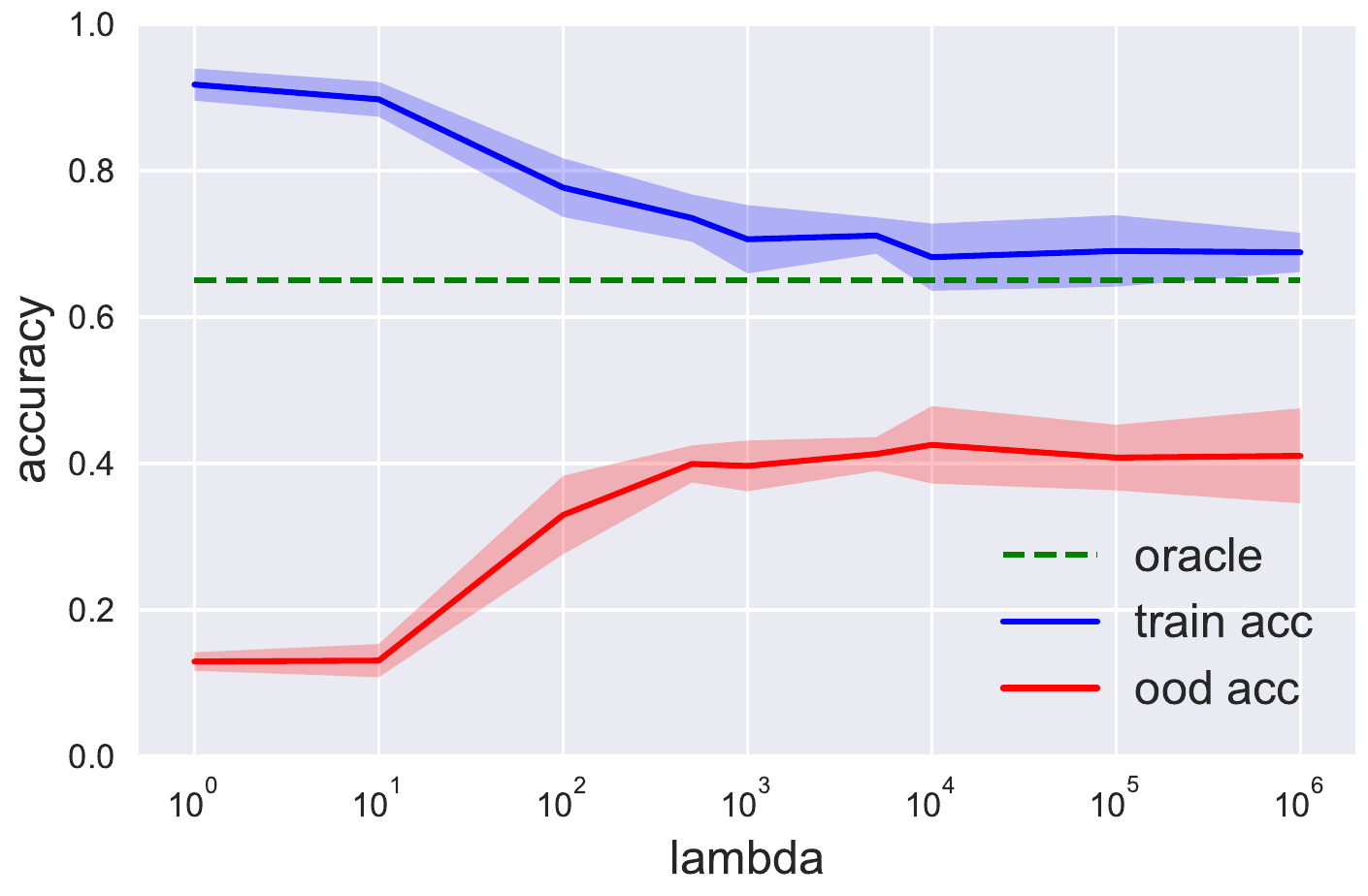}}{IRM}
\end{center}
\end{minipage}
\begin{minipage}[t]{0.35\hsize}
\begin{center}
\hspace{-0.15cm}
\stackunder[5pt]{\includegraphics[width=\linewidth]{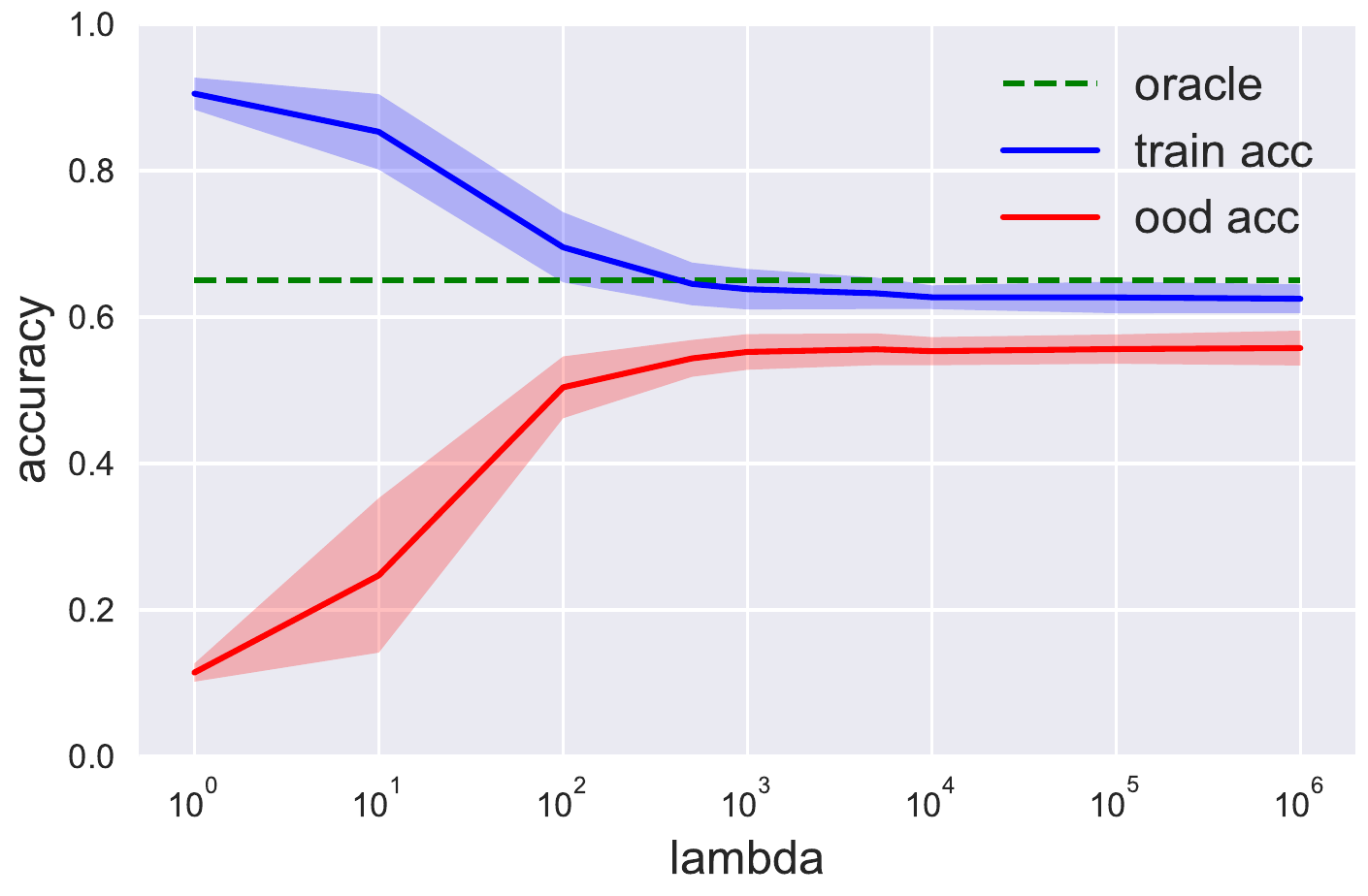}}{Ours}
\end{center}
\end{minipage}
\caption{Result on \textit{Extended Colored MNIST} ($p_0 = 0.25, p_1 = 0.65$)}
\label{fig:excmnist_065}
\end{figure*}

In general, IRM does not work well with standard gradient descent when we implement MLP without Batch Normalization (Figure \ref{fig:cmnist_mlp_result}).

\newpage
We shall  note that the original implementation of the IRM published in Github (\url{https://github.com/facebookresearch/InvariantRiskMinimization}) uses a very specific schedule for the regularization parameter $\lambda$, and it makes $\lambda$ to jump to a very large value at a very specific timing. 
The following figures are the result of their original algorithm on MNIST and \textit{Extended Colored MNIST} implemented with various jump-timings of $\lambda$.  
For \textit{Colored MNIST}, the original IRM works for specific choices of the jump timing($200\sim 300$). 
For \textit{Extended Colored MNIST}, the original algorithm does not work too well for any choice of the jump timings. 
Meanwhile, IRM works relatively well on \textit{Colored MNIST} consistently if we apply batch normalization, and it works well even without "jumping" the $\lambda$. 
For the tables we present in the main manuscript, we reported the result of IRM implemented \textit{with} batch normalization, which consistently yielded better results than the original implementation. 
\begin{figure*}[h]
\centering
\begin{minipage}[t]{0.35\hsize}
\begin{center}
\hspace{-0.3cm}
\stackunder[5pt]{\includegraphics[width=\linewidth]{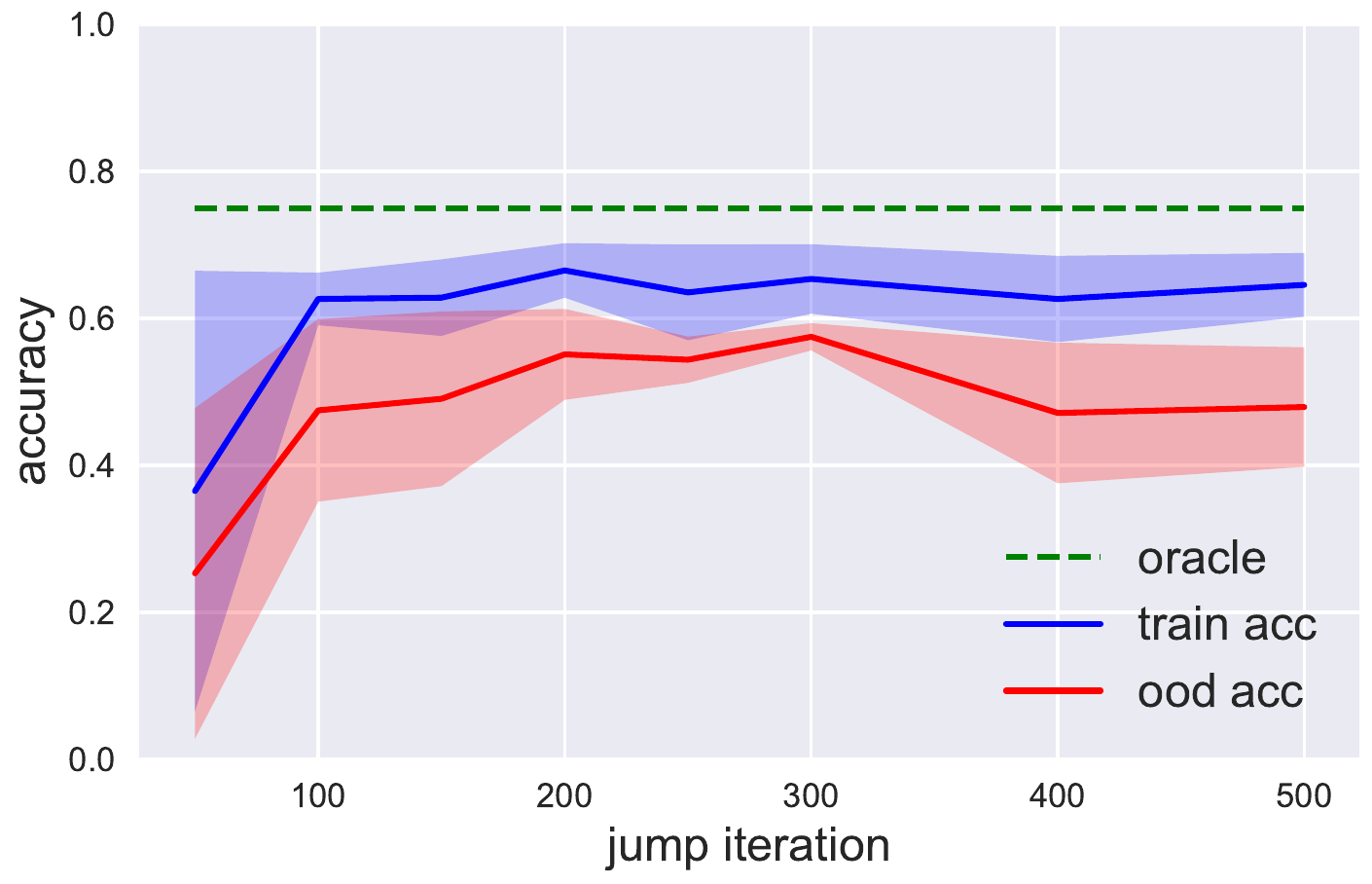}}{MLP}
\end{center}
\end{minipage}
\begin{minipage}[t]{0.35\hsize}
\begin{center}
\hspace{-0.15cm}
\stackunder[5pt]{\includegraphics[width=\linewidth]{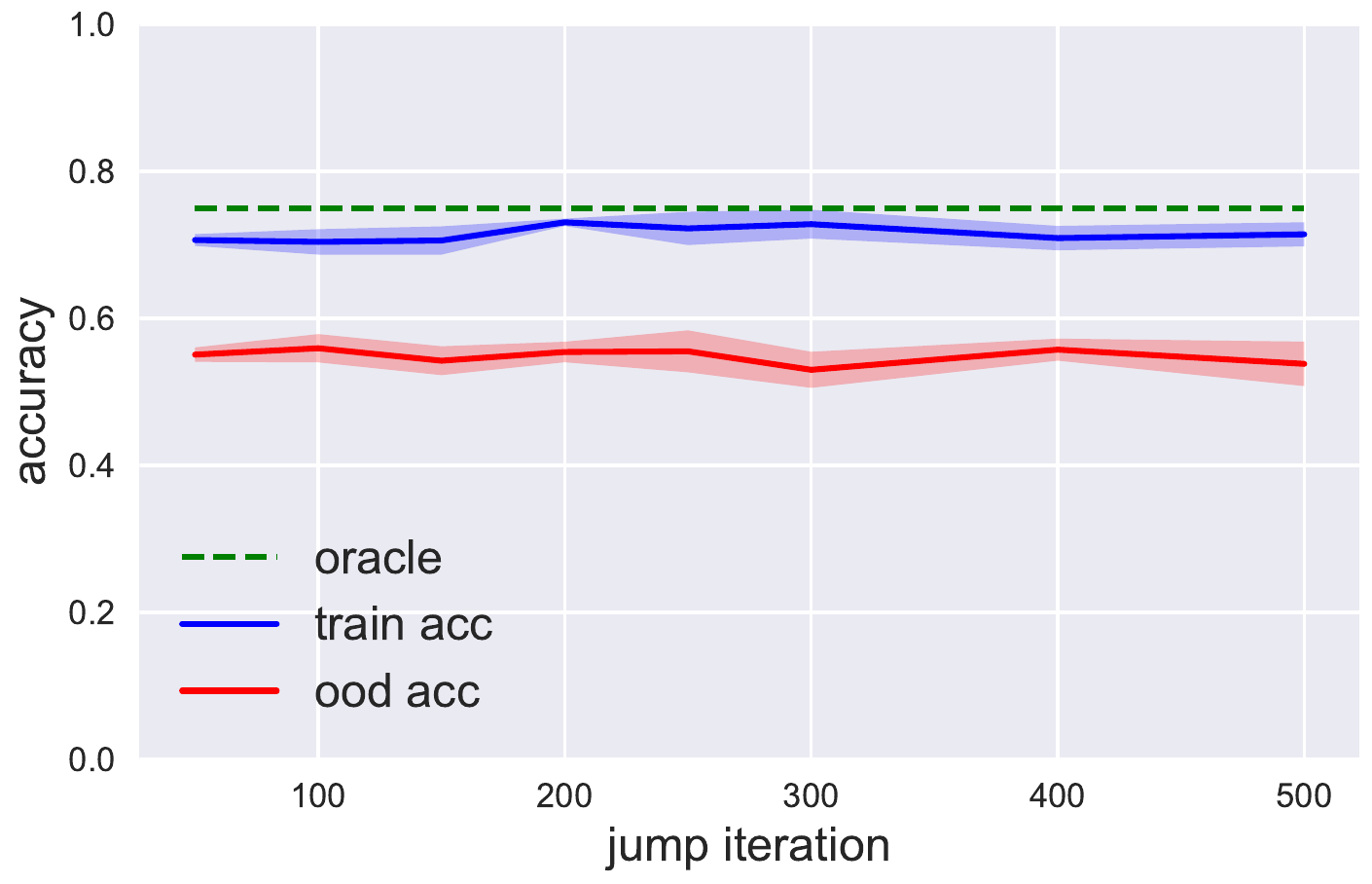}}{MLP with BN}
\end{center}
\end{minipage}
\caption{The plot of jump timing against the accuracies on \textit{Colored MNIST}}
\label{fig:cmnist_result_app}
\vspace{1.5cm}
\begin{minipage}[t]{0.35\hsize}
\begin{center}
\hspace{-0.3cm}
\stackunder[5pt]{\includegraphics[width=\linewidth]{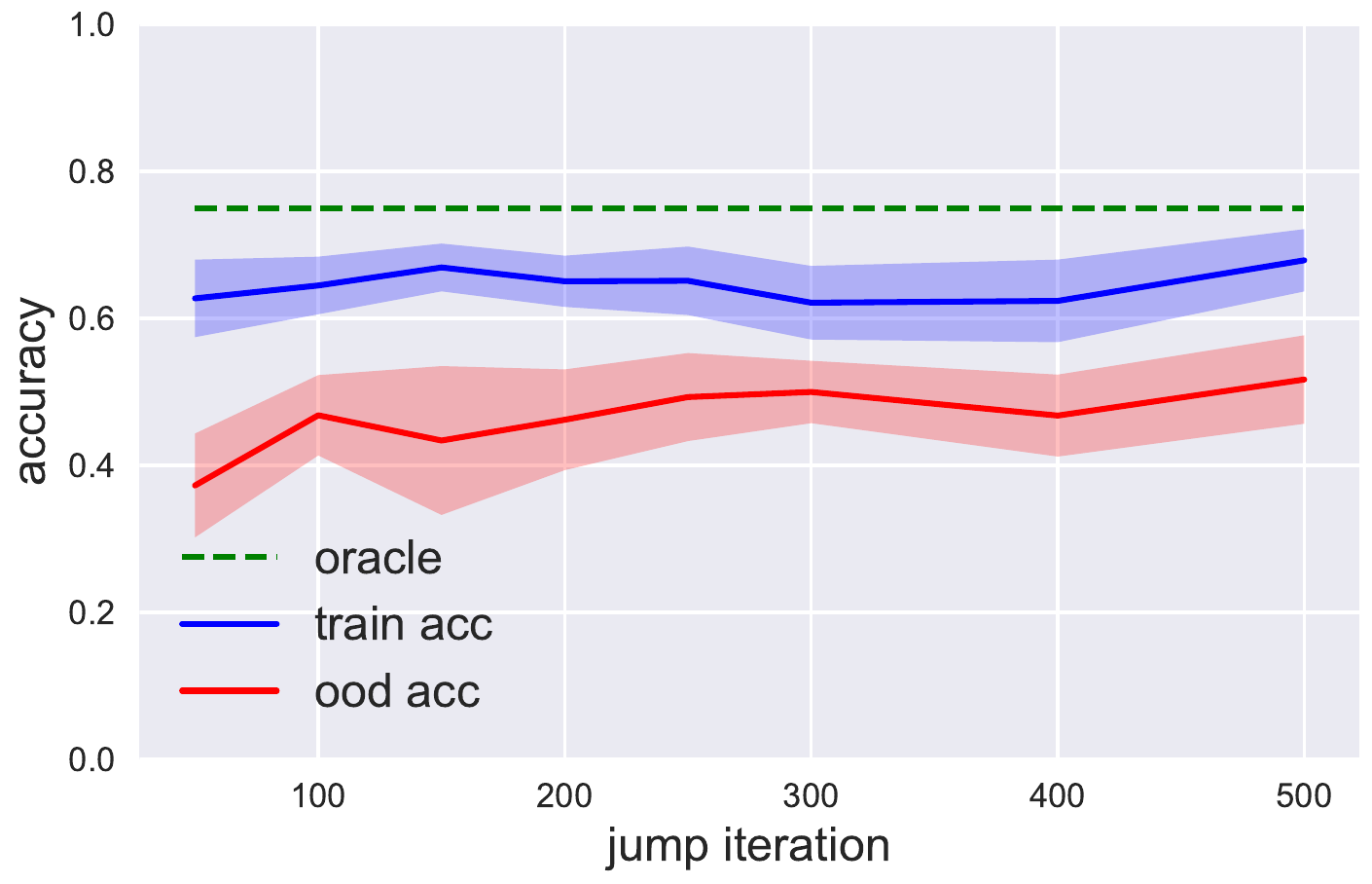}}{MLP}
\end{center}
\end{minipage}
\begin{minipage}[t]{0.35\hsize}
\begin{center}
\hspace{-0.15cm}
\stackunder[5pt]{\includegraphics[width=\linewidth]{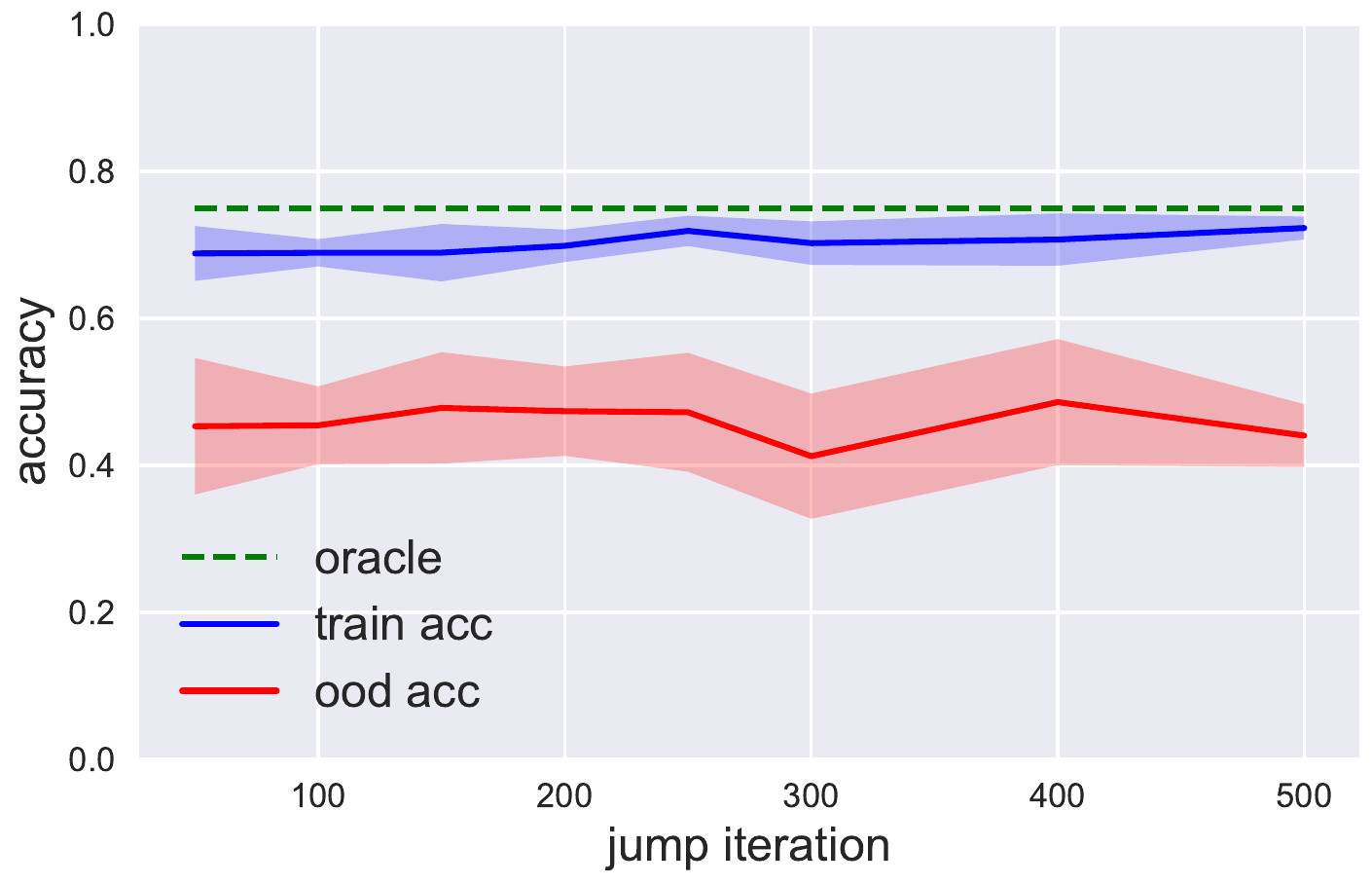}}{MLP with BN}
\end{center}
\end{minipage}
\caption{Results on \textit{Extended Colored MNIST}.}
\label{fig:ecmnist_result}

\begin{minipage}[t]{0.35\hsize}
\begin{center}
\hspace{-0.3cm}
\stackunder[5pt]{\includegraphics[width=\linewidth]{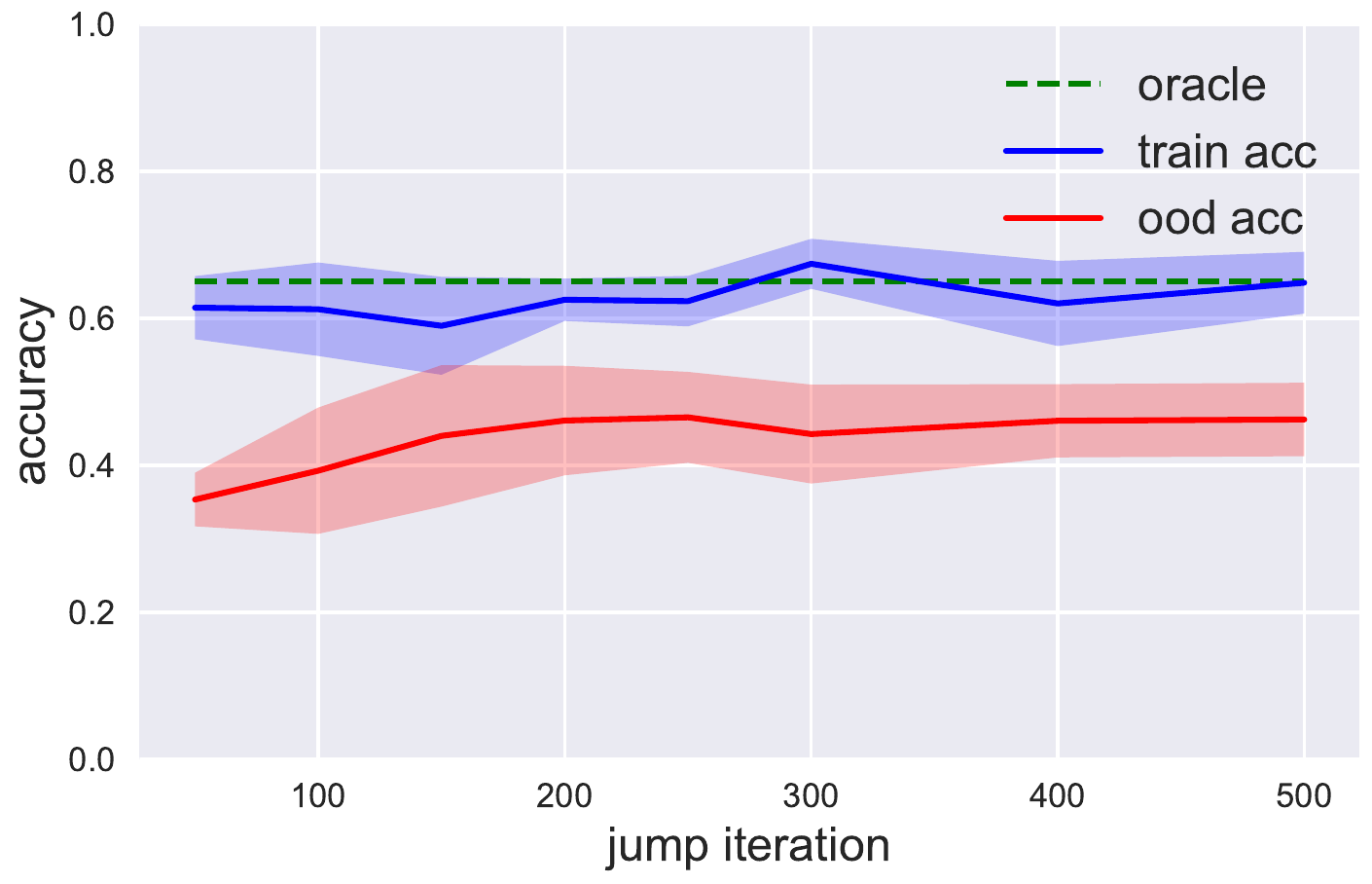}}{MLP}
\end{center}
\end{minipage}
\begin{minipage}[t]{0.35\hsize}
\begin{center}
\hspace{-0.15cm}
\stackunder[5pt]{\includegraphics[width=\linewidth]{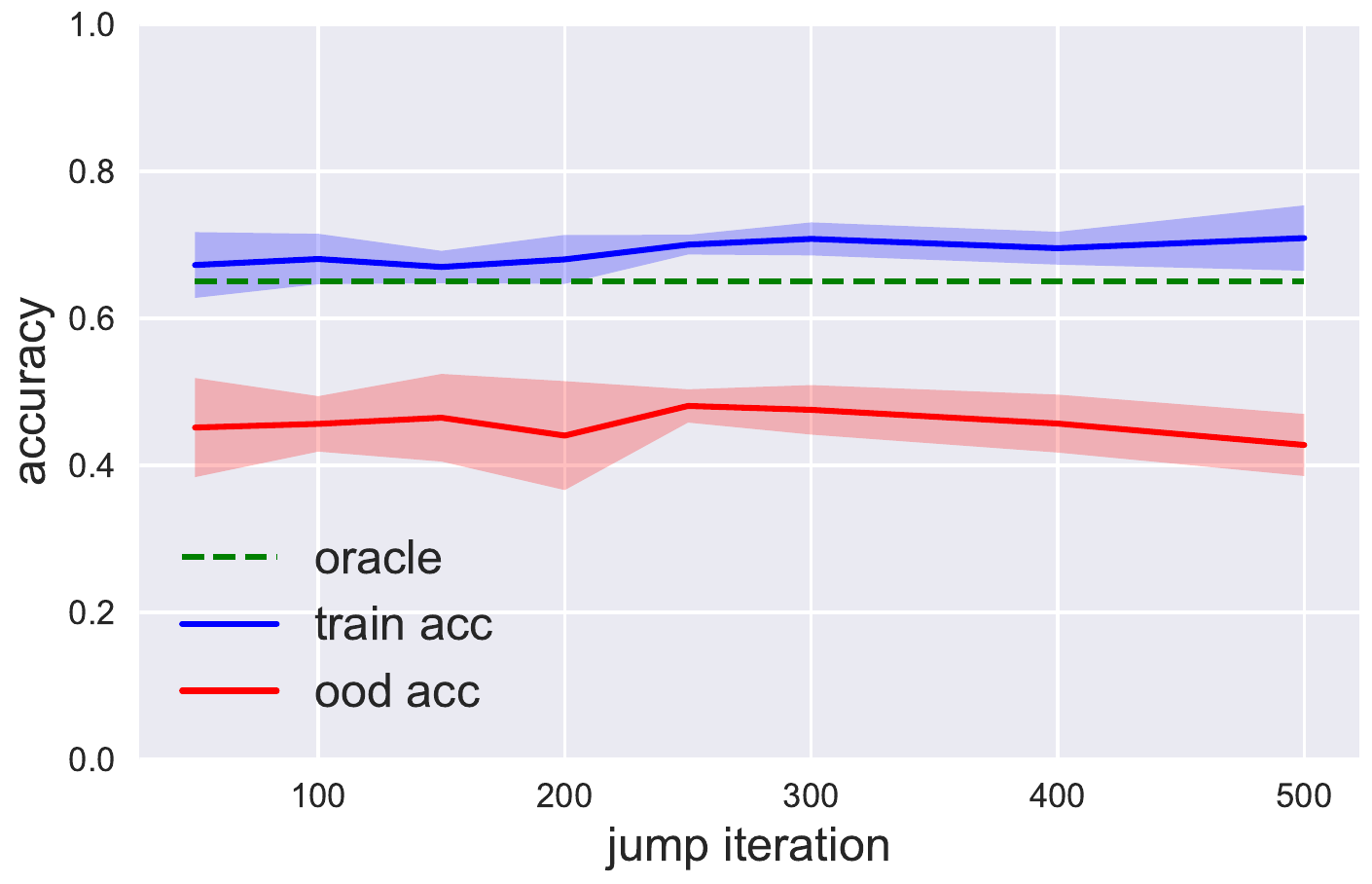}}{MLP with BN}
\end{center}
\end{minipage}
\caption{The plot of jump timing against the accuracies on \textit{Extended Colored MNIST}.}
\end{figure*}

\subsection{Ablation study for two phase training with nonlinear predictor}

\subsubsection{Two-phase training} 
Previous deep learning OOD algorithms like \cite{arjovsky2019invariant} and \cite{chang2020invariant} aim to learn an OOD optimal predictor in two phases: (i) the phase of learning a \textit{invariant feature} $h(X)$ and (ii) the phase of learning a predictor $\hat Y = g(h(X)$. 
When the loss is of Bregman divergence type, the optimal $g^*$ for any given $h$ takes of the form $\mathbb{E}[Y|\Phi] = g^*(h(X))$, and the optimal solution $g^*$ itself depends on the choice of $h$; we shall therefore write $g^*_h$ for $g^*$.
Because $g^*_h$ and $h$ are dependent on one another, allowing a large model space for both $g^*_h$ or $h$ would make the training difficult.
The algorithm of \cite{arjovsky2019invariant} took the approach of using a small model space for $g^*_h$ and a large black-box model space for $h$.
In other words, they trained the predictor of the form of $w^T h(X)$.
However, the complexity of $g^*_h$ may vary with $h$.
That is, even if $g^*_{h_0}$ is linear so that $E[Y|h_0(X)] = w^T h_0(X)$ for some $w$, it is possible that $g^*_{h_1} E[Y|h_1(X)]$ for a different $h_1$  might be more a more complex function of $h_1(X)$.
On the other hand, causality-inspired works took a reverse approach of assuming a possibly large model space for $g^*_h$, and sought $h$ from those that can be expressed as $M \odot X$ with a binary mask variable $M$.
That is, they constructed the invariant predictor of form $g^*_h(M \odot X)$ with nonlinear $g*_h$. 
However, if the model space of $g^*_h$ is large, optimizing $g^*_h$ with respect to function $h$ can be a daunting task.
In fact, \cite{chang2020invariant} is giving up the computation of the gradient of $g^*_h$ with respect to the parameter of $h$.
IGA is different from previous approaches in that it implicitly trains $g$ and $h$ in one phase. 

We conducted an ablation study to compare our one-phase training against the two-phase training.
To train a generic model in two phase training, we used the following modification of the objective function used in \cite{chang2020invariant}.
\begin{align}
     \argmin_{g, h} & \Big\{  \mathbb{E}[\mathcal{L}_\mathcal{E}(g \circ h)]  \nonumber \\
     &+ \lambda a( \mathbb{E}[\mathcal{L}_\mathcal{E}(g \circ h )] - \mathbb{E}[\mathcal{L}_\mathcal{E}(g_{\mathcal{E}, h}^* \circ h)]) \Big\}
\end{align}
We describe each component of this expression below. 
The function $a(t)$ is a  convex function that monotonically increases in $t$ when $t<0$, and strictly increases in $t$ when $t \geq 0$.  
The function $g_{\epsilon, h}^*$ for each $\epsilon$ and $h$ is the function that achieves  
$\min_g \mathcal{L}_\epsilon(g(h))$, and it is to be approximated with an internal round of gradient descent. 
The parameter $\lambda$ is the regularization parameter. 
We use a monotonic function like $a$ in this objective function because,  when we take the full expectation with respect to $X$ and $Y$, $\mathcal{L}_\epsilon(g \circ h) > \mathcal{L}_\epsilon(g^*_{\epsilon, h} \circ h)$  by the optimality of $g^*_{\epsilon, h}$.  
We trained $g$ and $h$ with this objective function on both C-MNIST and EC-MNIST, and studied the relation between $\lambda$ and the final accuracy as well as the value of the regularization term.

To model both $g$ and $h$, we used MLP with 4 layers containing 1500 nodes each and activation function elu. We did not use bias term in the last sigmoid activation. 
As is done in both \cite{arjovsky2019invariant} and \cite{chang2020invariant}, we optimized both models in parallel without propagating the loss of $g_h$ to $h$. 
For both C-MNIST and EC-MNIST, we evaluated the model performance in the same way as in the IRM experiments.

As we see in the plots below, even when the loss of the environment agnostic predictor $g \circ h$ is close to environment specific $g_{\epsilon, h}^* \circ h$,  the the performance on the training environments does not generalize to all environments. 
This tendency was observed irrespective of the presence of Batch normalization.
This is possibly true $g^*_\epsilon$ is not estimated well in the training process due to the inter-dependency between $g$ and $h$.

\begin{figure*}[h]
\centering
\begin{minipage}[t]{0.275\hsize}
\begin{center}
\hspace{-0.3cm}
\stackunder[5pt]{\includegraphics[width=\linewidth]{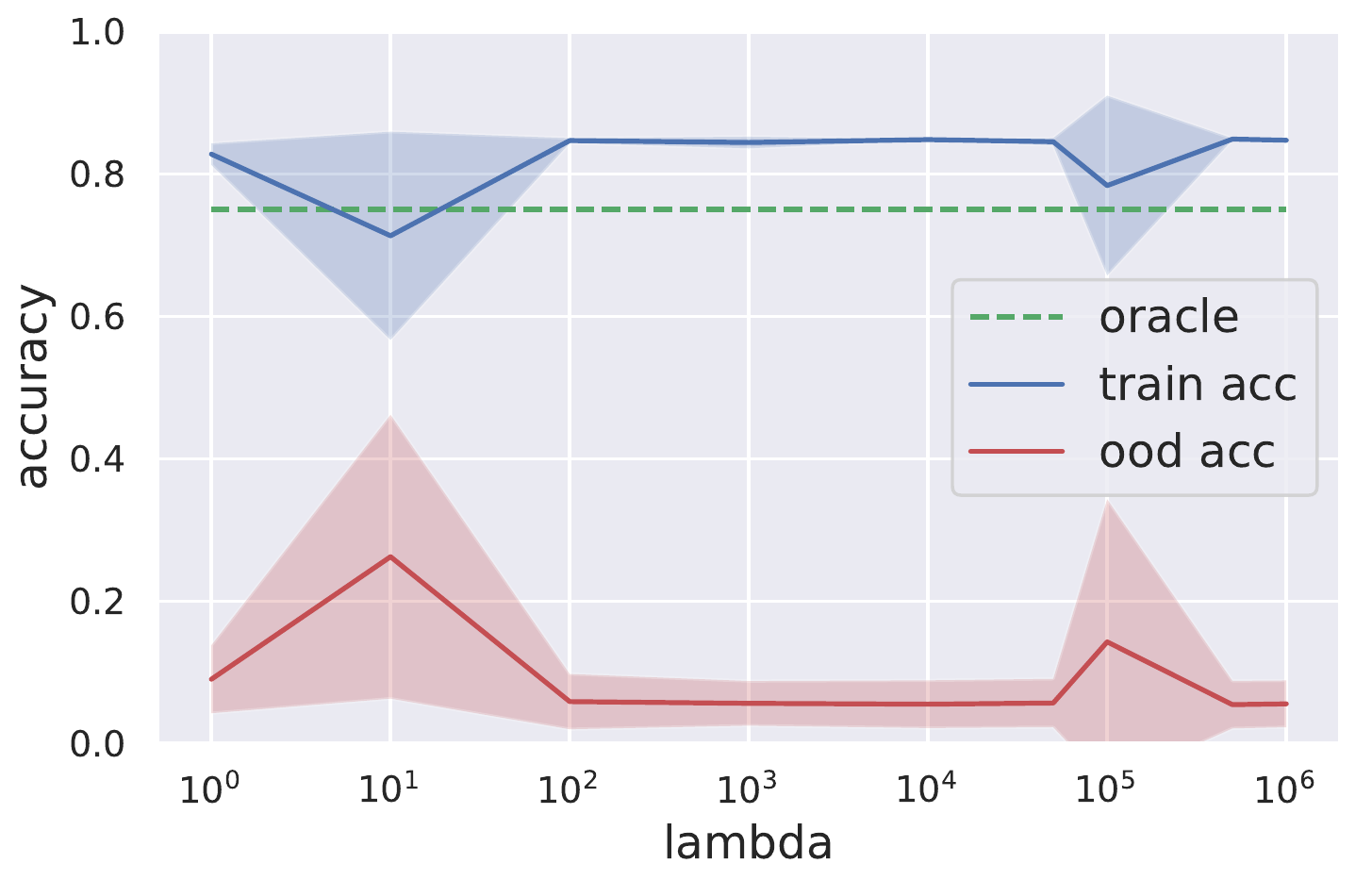}}{accuracy}
\end{center}
\end{minipage}
\begin{minipage}[t]{0.275\hsize}
\begin{center}
\hspace{-0.15cm}
\stackunder[5pt]{\includegraphics[width=\linewidth]{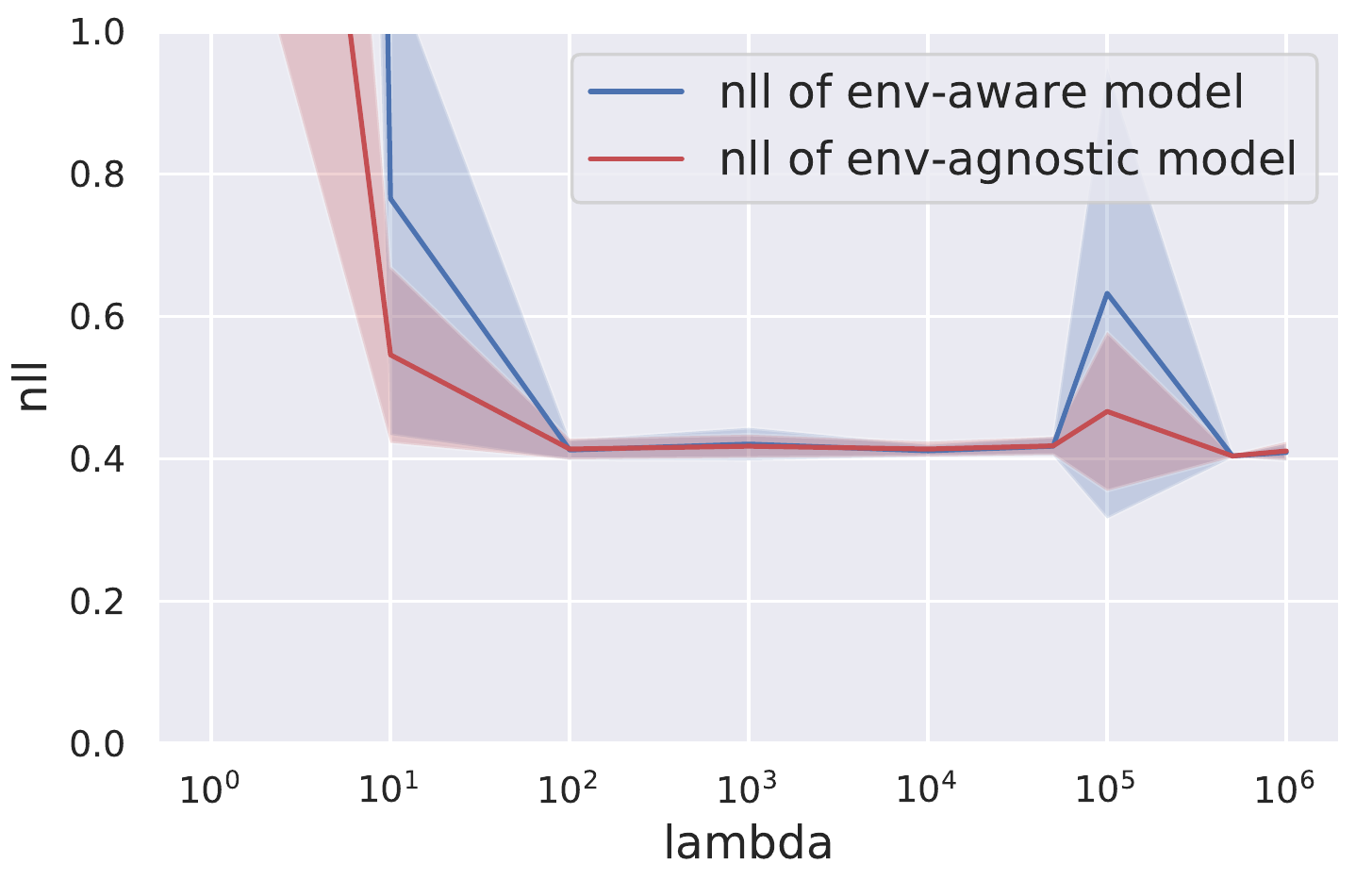}}{The final nll}
\end{center}
\end{minipage}
\caption{The two phase training results for \textit{Colored MNIST} with MLP encoder and MLP predictor}
\label{fig:ir_cmnist_result_app}
\end{figure*}

\begin{figure*}[h]
\centering
\begin{minipage}[t]{0.275\hsize}
\begin{center}
\hspace{-0.3cm}
\stackunder[5pt]{\includegraphics[width=\linewidth]{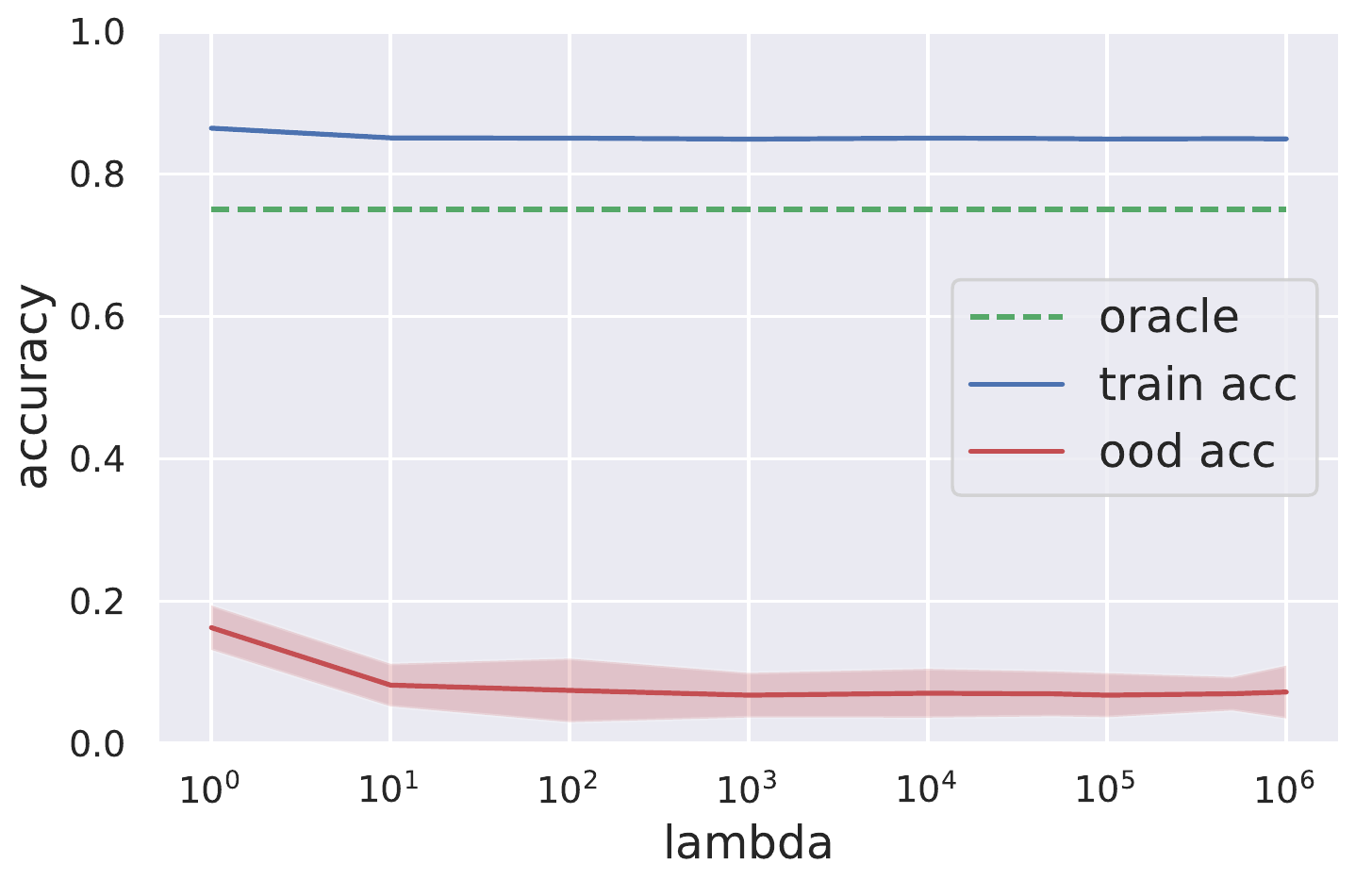}}{accuracy}
\end{center}
\end{minipage}
\begin{minipage}[t]{0.275\hsize}
\begin{center}
\hspace{-0.15cm}
\stackunder[5pt]{\includegraphics[width=\linewidth]{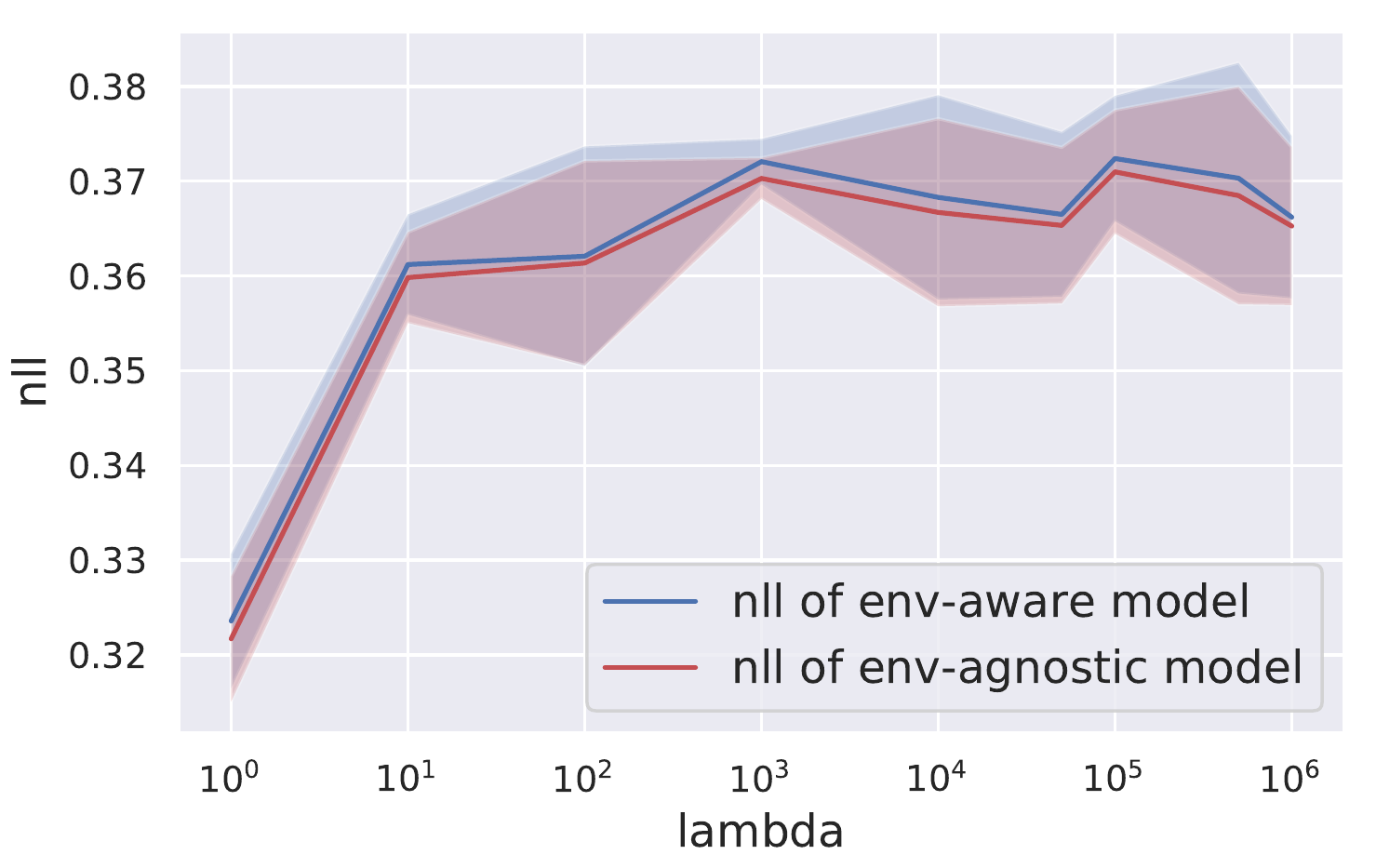}}{The final nll}
\end{center}
\end{minipage}
\caption{The two phase training results for \textit{Colored MNIST} with MLP + BN encoder and MLP + BN predictor }
\label{fig:ir_cmnist_result_app_bn}
\end{figure*}

\begin{figure*}[h]
\centering
\begin{minipage}[t]{0.275\hsize}
\begin{center}
\hspace{-0.3cm}
\stackunder[5pt]{\includegraphics[width=\linewidth]{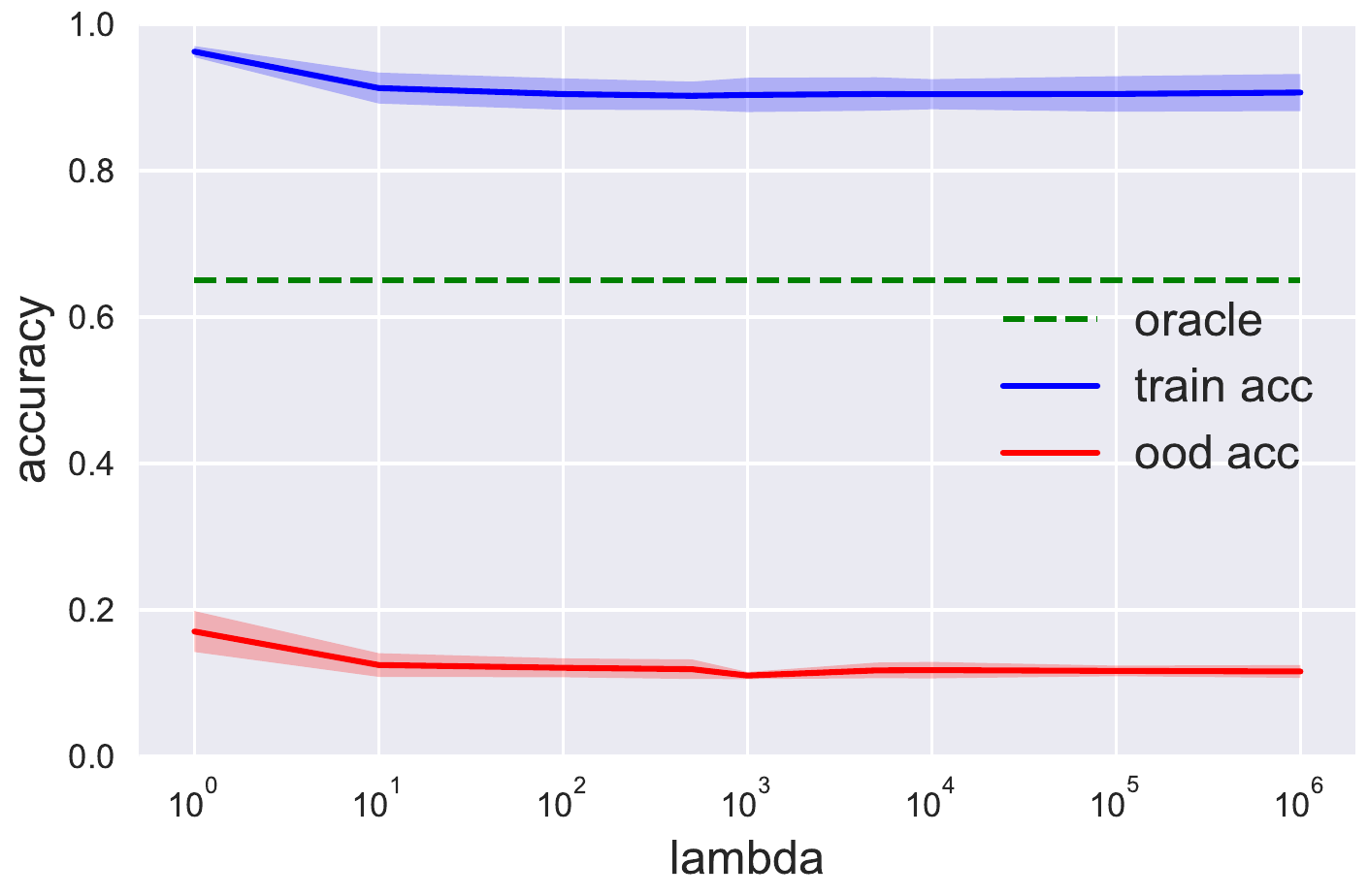}}{accuracy}
\end{center}
\end{minipage}
\begin{minipage}[t]{0.275\hsize}
\begin{center}
\hspace{-0.15cm}
\stackunder[5pt]{\includegraphics[width=\linewidth]{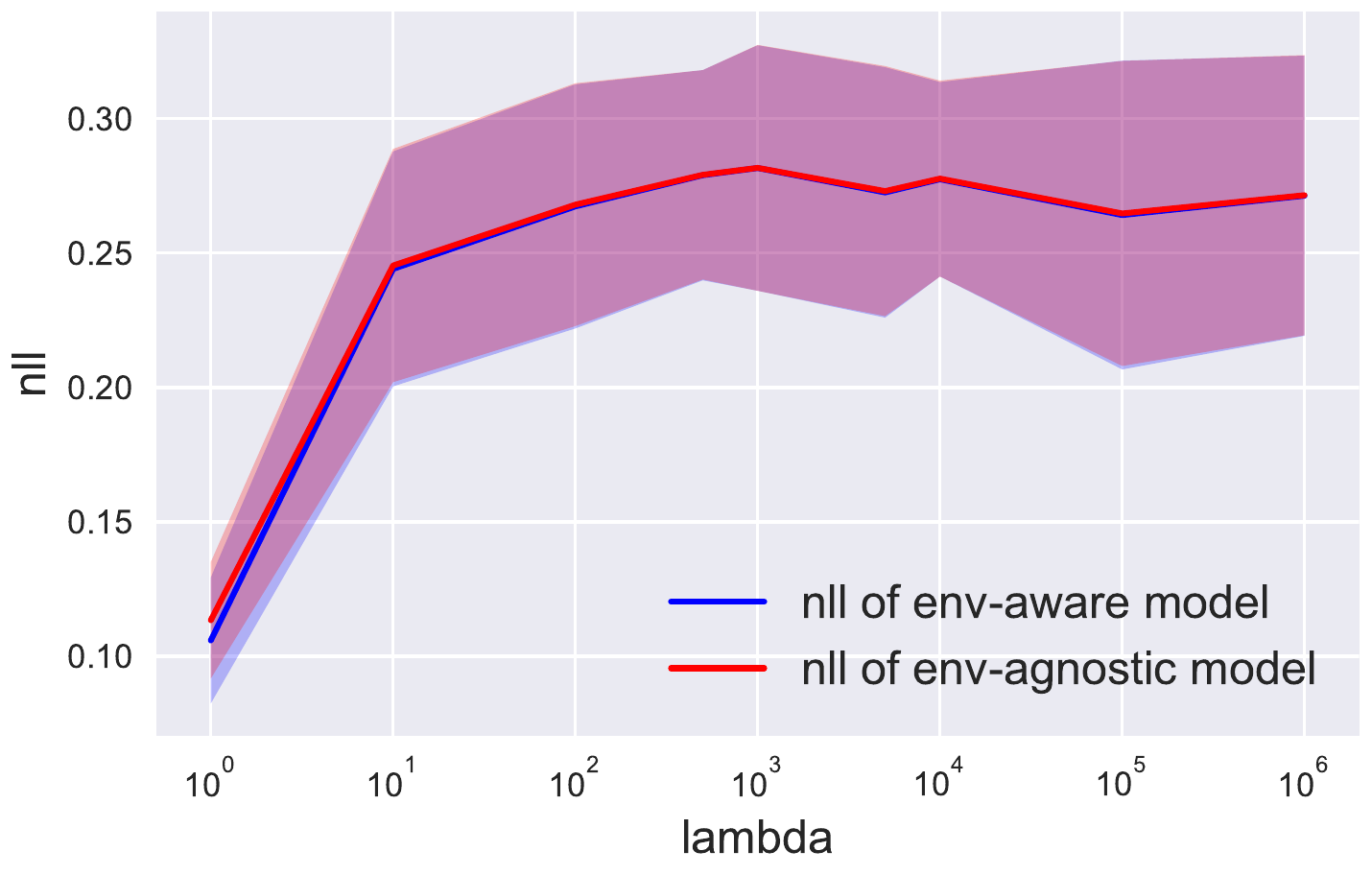}}{The final nll}
\end{center}
\end{minipage}
\caption{The two phase training results for \textit{Extended Colored MNIST}($p_0 = 0.25, p_1=0.65$) with MLP + BN encoder and MLP + BN predictor}
\label{fig:ir_ecmnist_result_app_065_bn}
\end{figure*}

\begin{figure*}[h]
\centering
\begin{minipage}[t]{0.275\hsize}
\begin{center}
\hspace{-0.3cm}
\stackunder[5pt]{\includegraphics[width=\linewidth]{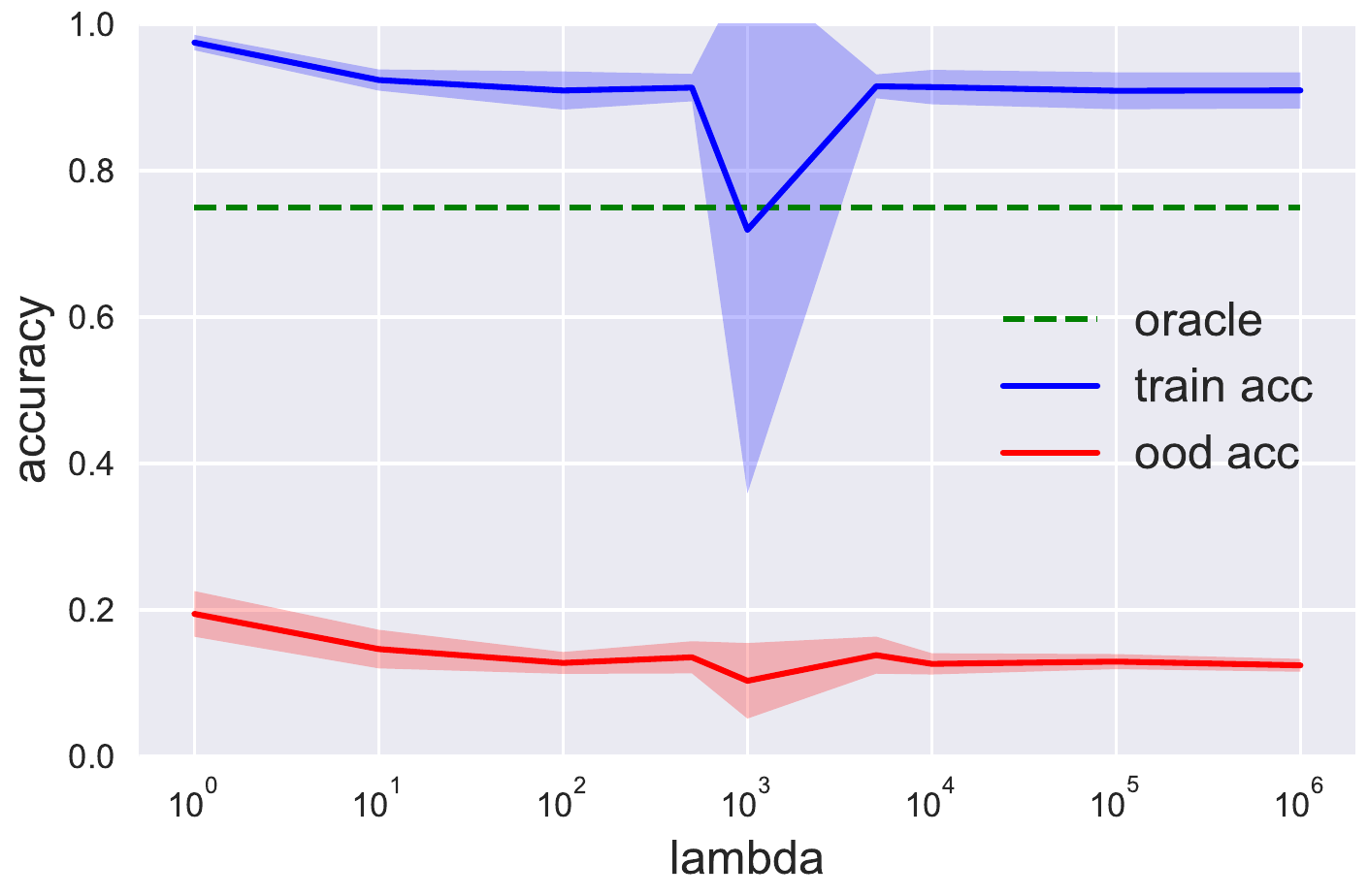}}{accuracy}
\end{center}
\end{minipage}
\begin{minipage}[t]{0.275\hsize}
\begin{center}
\hspace{-0.15cm}
\stackunder[5pt]{\includegraphics[width=\linewidth]{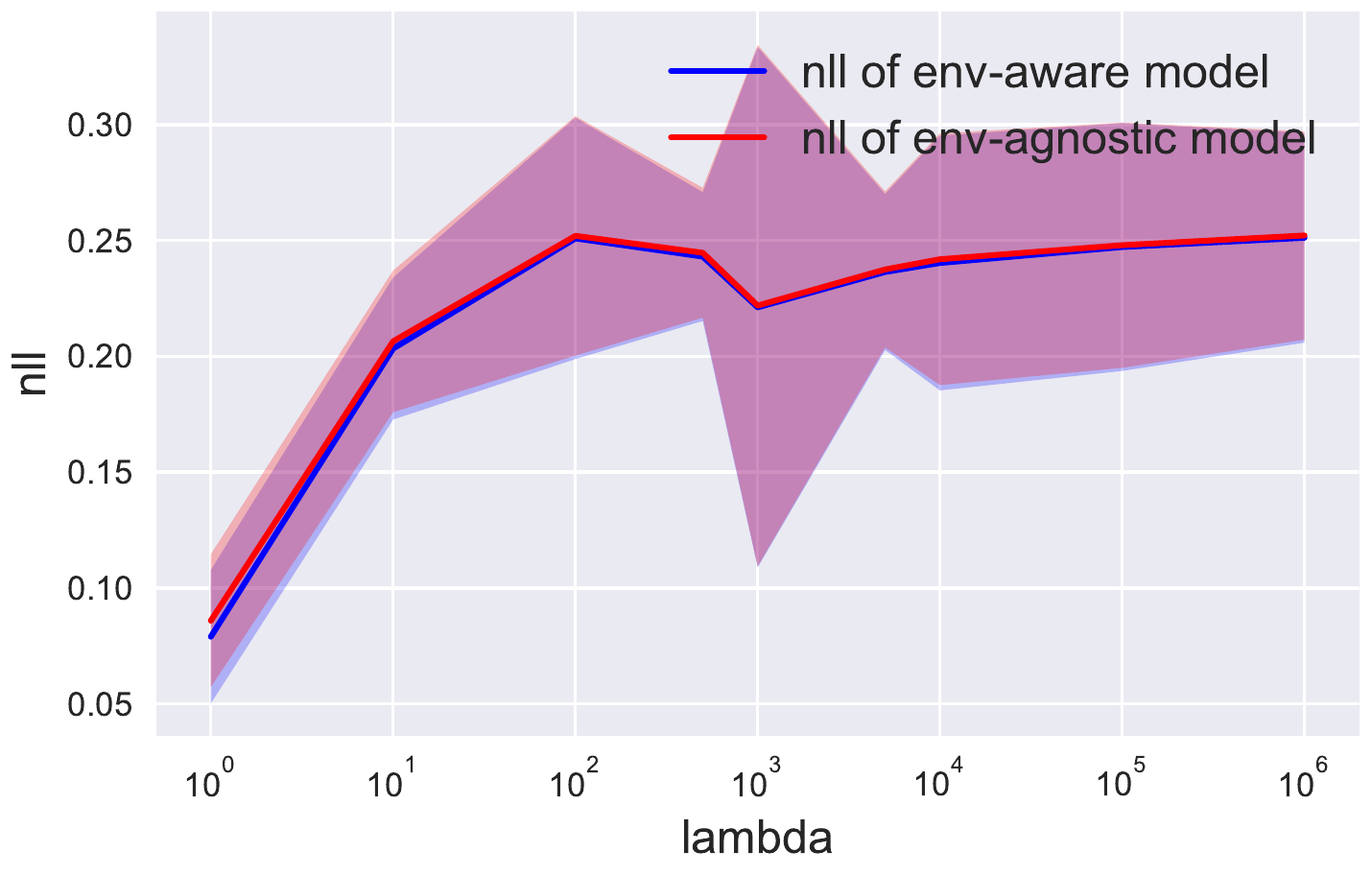}}{The final nll}
\end{center}
\end{minipage}
\caption{The two phase training results for \textit{Extended Colored MNIST}($p_0 = 0.25, p_1=0.75$) with MLP + BN encoder and MLP + BN predictor}
\label{fig:ir_ecmnist_result_app_075_bn}
\end{figure*}

\begin{figure*}[h]
\centering
\begin{minipage}[t]{0.275\hsize}
\begin{center}
\hspace{-0.3cm}
\stackunder[5pt]{\includegraphics[width=\linewidth]{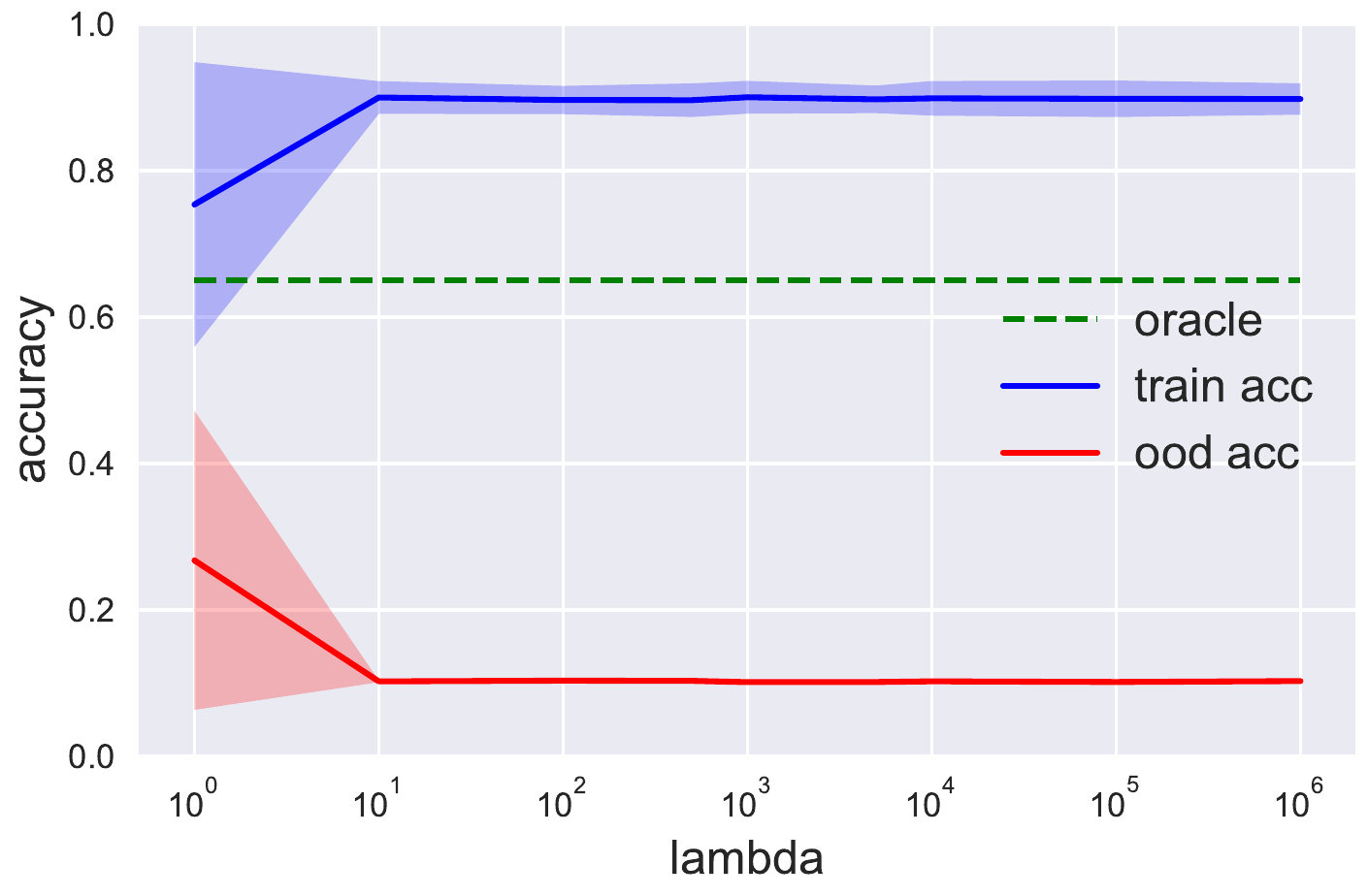}}{accuracy}
\end{center}
\end{minipage}
\begin{minipage}[t]{0.275\hsize}
\begin{center}
\hspace{-0.15cm}
\stackunder[5pt]{\includegraphics[width=\linewidth]{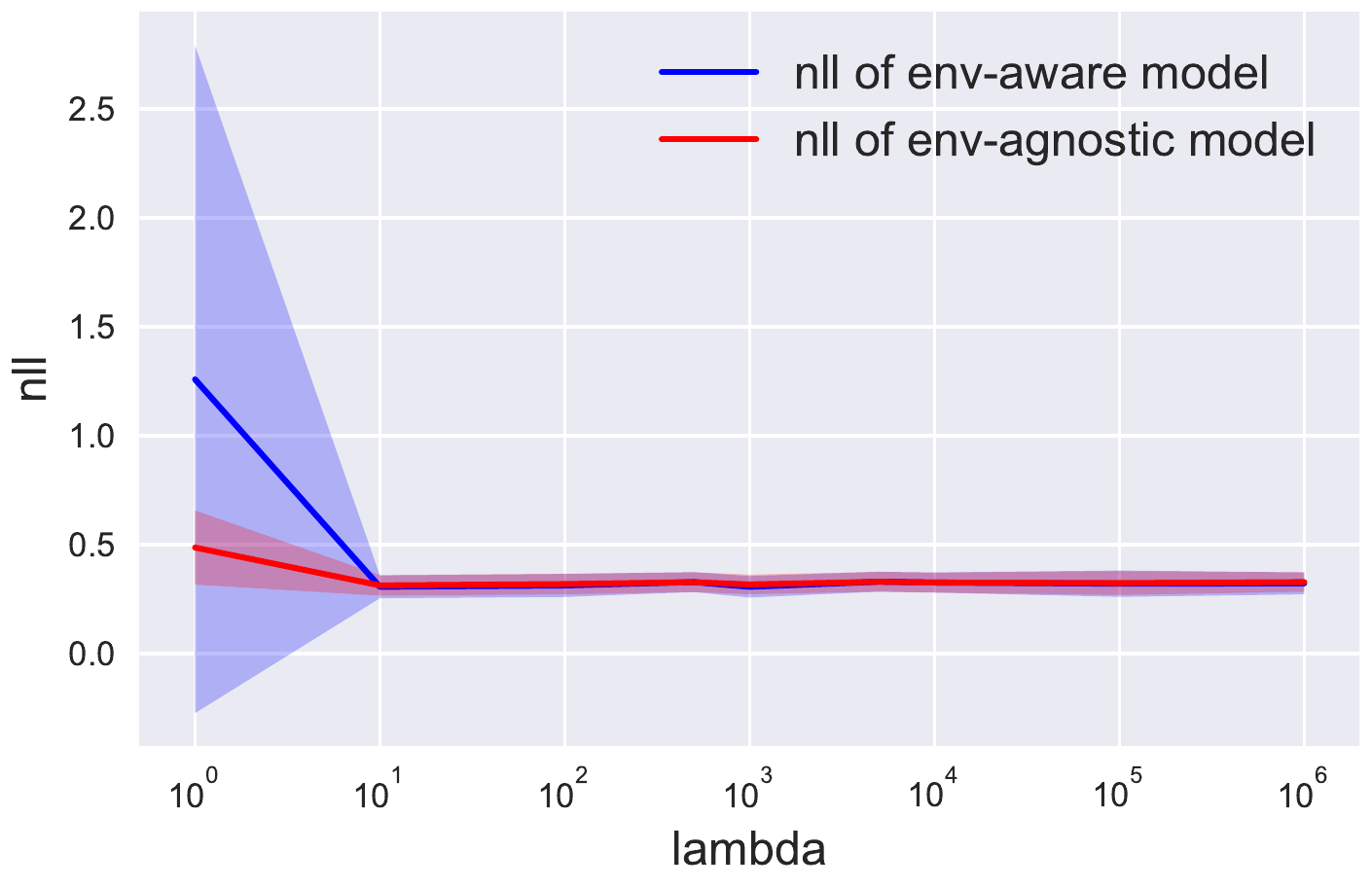}}{The final nll}
\end{center}
\end{minipage}
\caption{The two phase training results for \textit{Extended Colored MNIST}($p_0 = 0.25, p_1=0.65$) with MLP encoder and MLP predictor}
\label{fig:ir_ecmnist_result_app_065}
\end{figure*}

\begin{figure*}[h]
\centering
\begin{minipage}[t]{0.275\hsize}
\begin{center}
\hspace{-0.3cm}
\stackunder[5pt]{\includegraphics[width=\linewidth]{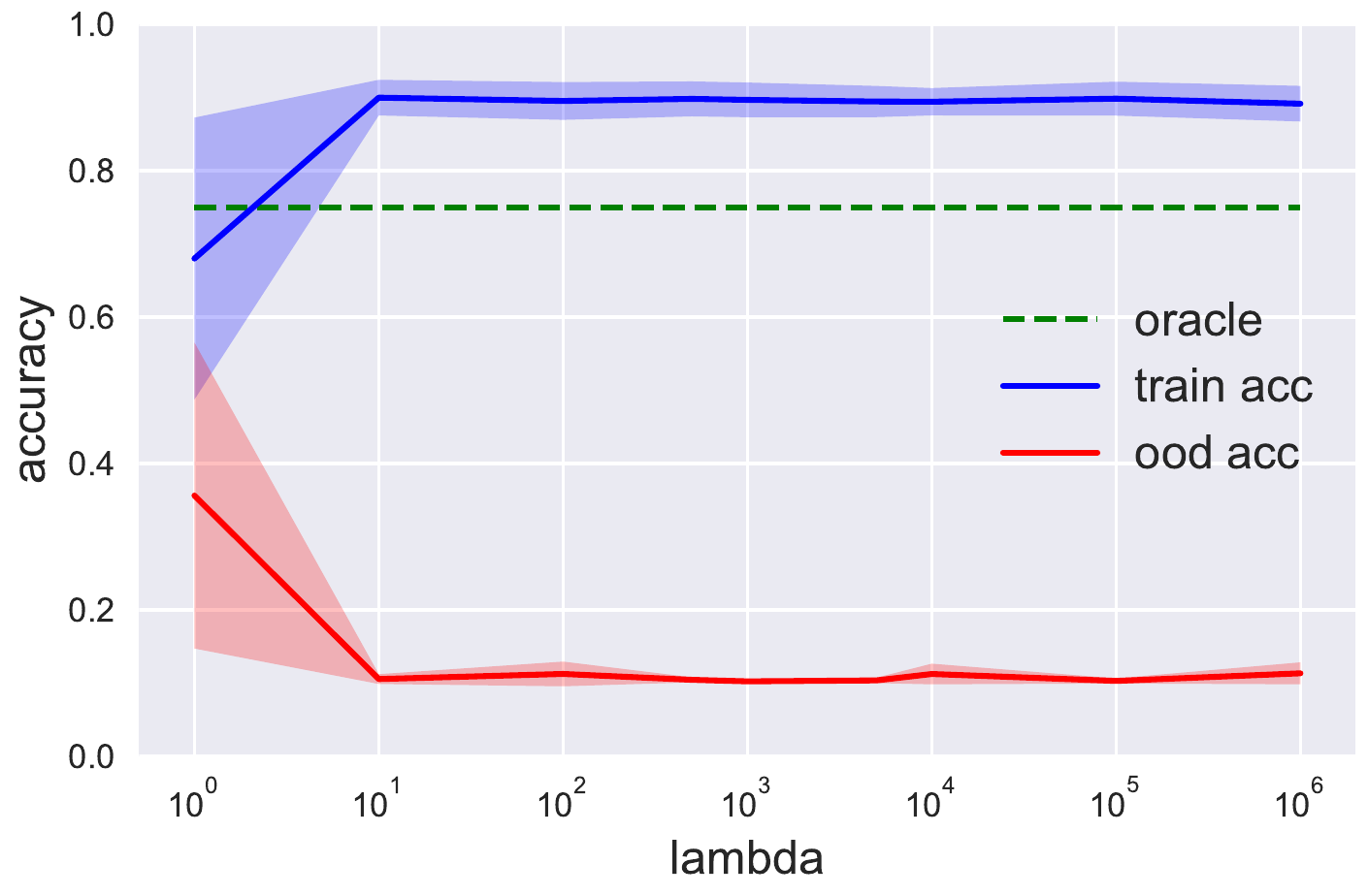}}{accuracy}
\end{center}
\end{minipage}
\begin{minipage}[t]{0.275\hsize}
\begin{center}
\hspace{-0.15cm}
\stackunder[5pt]{\includegraphics[width=\linewidth]{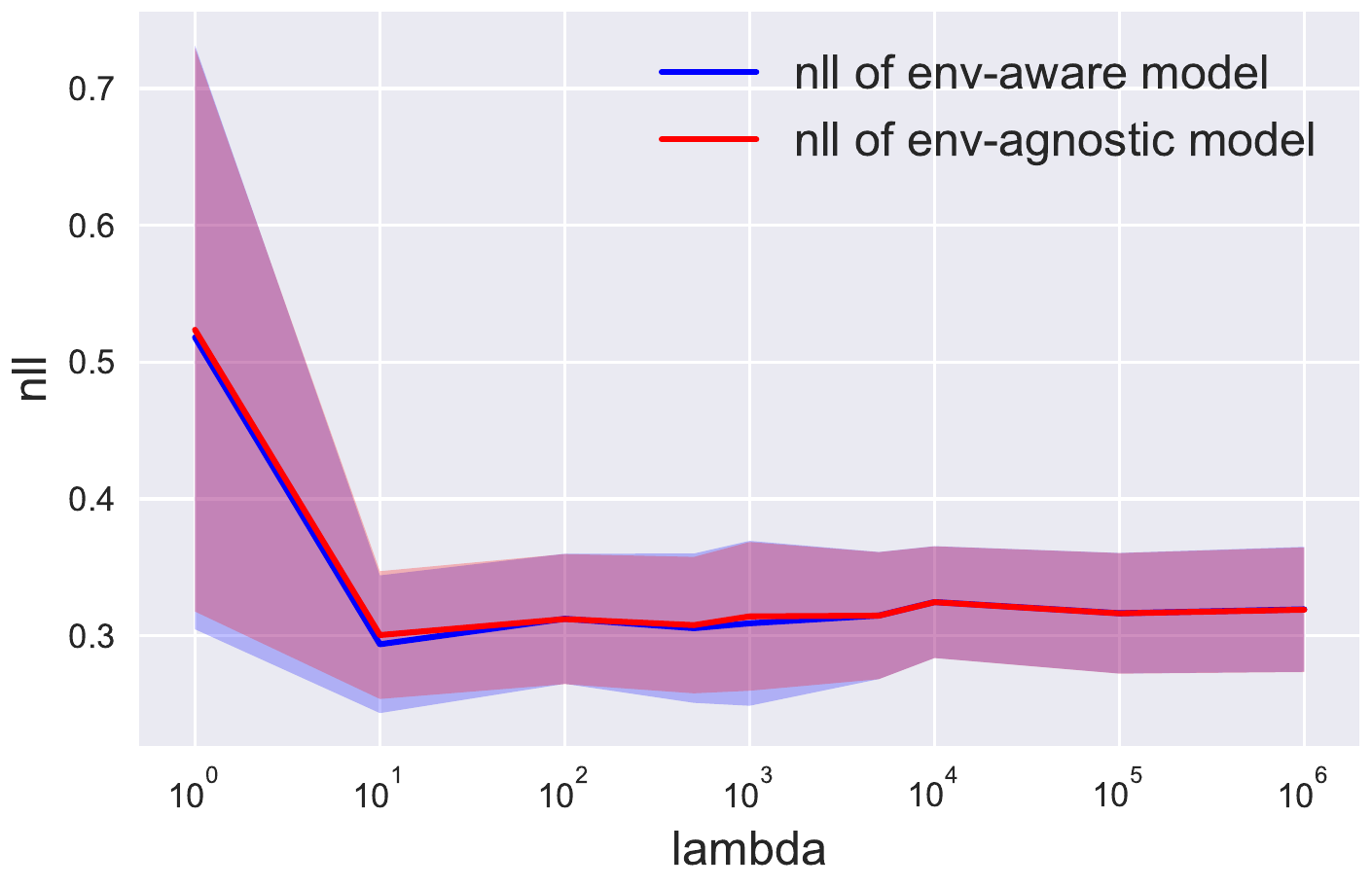}}{The final nll}
\end{center}
\end{minipage}
\caption{The two phase training results for \textit{Extended Colored MNIST}($p_0 = 0.25, p_1=0.75$) with MLP encoder and MLP predictor}
\label{fig:ir_ecmnist_result_app_075}
\end{figure*}

\end{document}